\documentclass{article}

\PassOptionsToPackage{numbers, compress}{natbib}
\usepackage[eandd, final]{neurips_2026}
\usepackage[utf8]{inputenc}
\usepackage[T1]{fontenc}
\usepackage{microtype}
\usepackage{graphicx}
\usepackage{subcaption}
\usepackage{booktabs} %
\usepackage{makecell}
\usepackage{tabularx}
\usepackage{caption}
\usepackage{sidecap}
\usepackage[breaklinks=true,letterpaper=true,colorlinks,bookmarks=false]{hyperref}
\usepackage{xcolor}
\usepackage[page]{appendix}
\usepackage[most]{tcolorbox}
\usepackage{enumitem} %
\makeatletter
\let\oldappendix\appendices

\renewcommand{\appendices}{%
  \clearpage
  \RestoreAddContentsLine %
  \renewcommand{\thesection}{\Roman{section}}
  \let\tf@toc\tf@app
  \addtocontents{app}{\protect\setcounter{tocdepth}{2}}
  \immediate\write\@auxout{%
    \string\let\string\tf@toc\string\tf@app^^J
  }
  \oldappendix
}%

\newcommand{\listofappendices}{%
  \begingroup
  \renewcommand{\contentsname}{Contents}
  \let\@oldstarttoc\@starttoc
  \def\@starttoc##1{\@oldstarttoc{app}}
  \tableofcontents
  \endgroup
}

\makeatother

\makeatletter
\newcommand{\RestoreAddContentsLine}{%
  \ifcsname hyper@anchor\endcsname
    \def\addcontentsline##1##2##3{%
      \addtocontents{##1}{%
        \protect\contentsline{##2}{##3}{\thepage}{\@currentHref}%
      }%
    }%
  \else
    \def\addcontentsline##1##2##3{%
      \addtocontents{##1}{%
        \protect\contentsline{##2}{##3}{\thepage}%
      }%
    }%
  \fi
}
\makeatother
\definecolor{dark_green}{rgb}{.2,.5,0}
\definecolor{dark_blue}{rgb}{0,.15,.5}

\hypersetup{pdftex,  %
  breaklinks=true,  %
  colorlinks=true,
  linkcolor=dark_blue,
  citecolor=dark_green,
}

\DeclareFontEncoding{LS1}{}{}
\DeclareFontSubstitution{LS1}{stix}{m}{n}
\DeclareSymbolFont{stixsymbols4}{LS1}{stixbb}{m}{it}
\DeclareMathSymbol{\hexagonblack}{\mathord}{stixsymbols4}{"DE}
\DeclareFontEncoding{LS2}{}{}
\DeclareFontSubstitution{LS2}{stix2}{m}{n}
\DeclareSymbolFont{STIXTwoArrowsIII}{LS2}{stix2tt}{m}{n}
\SetSymbolFont{STIXTwoArrowsIII}{bold}{LS2}{stix2tt}{b}{n}

\makeatletter
\def\STIX@undefine#1{%
  \if\relax\noexpand#1\let#1=\@undefined\fi}
\STIX@undefine\mdblkcircle
\makeatother

\DeclareMathSymbol{\mdblkcircle}{\mathord}{STIXTwoArrowsIII}{"7E}

\usepackage{float}
\usepackage{enumitem}

\PassOptionsToPackage{hyphens}{url}
\usepackage{hyperref}
\usepackage{url}

\usepackage{xcolor}

\usepackage{amsmath}
\usepackage{amssymb}
\usepackage{mathtools}
\usepackage{amsthm}
\usepackage{wrapfig}

\usepackage[capitalize,nameinlink,noabbrev]{cleveref}

\usepackage{multirow}

\usepackage{snaptodo}

\snaptodoset{block rise=2em}
\snaptodoset{margin block/.style={font=\footnotesize}} 
\AtBeginDocument{\setlength{\marginparsep}{-0.05cm}}

\usepackage{soul}

\newcommand{\bbE}{\mathbb{E}}

\theoremstyle{plain}
\newtheorem{theorem}{Theorem}[section]
\newtheorem{proposition}{Proposition}[section]
\newtheorem{lemma}{Lemma}[section]
\newtheorem{corollary}{Corollary}[section]
\theoremstyle{definition}
\newtheorem{definition}[theorem]{Definition}
\newtheorem{assumption}{Assumption}[section]
\theoremstyle{remark}
\newtheorem{remark}[theorem]{Remark}

\crefname{proposition}{Proposition}{Propositions}
\crefname{lemma}{Lemma}{Lemmas}
\crefname{corollary}{Corollary}{Corollaries}
\crefname{assumption}{Assumption}{Assumptions}

\newcommand{\defeq}{\overset{\text{\tiny def}}{=}}

\usepackage[textsize=tiny]{todonotes}

\title{STRABLE: Benchmarking Tabular Machine Learning with Strings}

\author{
  \textbf{Gioia Blayer}$^{1}$ \quad
  \textbf{Myung Jun Kim}$^{1}$ \quad
  \textbf{F\'{e}lix Lefebvre}$^{1}$ \quad
  \textbf{Lennart Purucker}$^{4,3}$ \\
  \textbf{Alan Arazi}$^{4,6}$ \quad
  \textbf{Eilam Shapira}$^{6}$ \quad
  \textbf{Roi Reichart}$^{6}$ \quad
  \textbf{Frank Hutter}$^{4,5,3}$ \\
  \textbf{Marine Le Morvan}$^{1}$ \quad
  \textbf{David Holzm\"{u}ller}$^{1}$ \quad
  \textbf{Ga\"{e}l Varoquaux}$^{1,2}$ \\
  \vspace{0.5em} \\
  \small $^{1}$SODA Team, INRIA Saclay, Palaiseau, France \quad
  \small $^{2}$Probabl, France \quad
  \small $^{3}$University of Freiburg \\
  \small $^{4}$Prior Labs \quad
  \small $^{5}$ELLIS Institute T\"{u}bingen \quad
  \small $^{6}$Technion -- Israel Institute of Technology \\
  \texttt{gioia.blayer@inria.fr}
}

\begin{document}

\maketitle

\begin{abstract}
Benchmarking tabular learning has revealed the benefit of dedicated architectures, pushing the state of the art. But real-world tables often contain string entries, beyond numbers, and these settings have been understudied due to a lack of a solid benchmarking suite. They lead to new research questions: Are dedicated learners needed, with end-to-end modeling of strings and numbers? Or does it suffice to encode strings as numbers, as with a categorical encoding? And if so, do the resulting tables resemble numerical tabular data, calling for the same learners? To enable these studies, we contribute STRABLE, a benchmarking corpus of 108 tables, all real-world learning problems with strings and numbers across diverse application fields. We run the first large-scale empirical study of tabular learning with strings, evaluating 445 pipelines. These pipelines span end-to-end architectures and modular pipelines, where strings are first encoded, then post-processed, and finally passed to a tabular learner. We find that, because most tables in the wild are categorical-dominant, advanced tabular learners paired with simple string embeddings achieve good predictions at low computational cost. On free-text-dominant tables, large LLM encoders become competitive. Their performance also appears sensitive to post-processing, with differences across LLM families. Finally, we show that STRABLE is a good set of tables to study ``string tabular'' learning as it leads to generalizable pipeline rankings that are close to the oracle rankings. We thus establish STRABLE as a foundation for research on tabular learning with strings, an important yet understudied area. \footnote{Web Page of the Strable benchmark: \url{https://soda-inria.github.io/strable/}} 
\end{abstract}

\section{Introduction: The importance of strings for tabular learning}
\label{sec:intro}

Benchmarking datasets have been central to progress in machine learning and artificial intelligence: for instance, MNIST \citep{726791} and ImageNet \citep{5206848} drove the deep-learning revolution. Standard benchmarking datasets for a given domain, such as vision or language, enable the emergence of new ideas, sorting out the good from the bad. The model rankings they produce are useful beyond the benchmark itself, as they transfer to new datasets from the same domain \citep{recht2019imagenet,roelofs2019meta,hardt2025emerging}.

The diversity of data tables, central to enterprise machine learning, may seem like a roadblock to this benchmarking methodology. Yet, tabular data has regularities of its own, such as a strong columnar structure (different quantities and distributions per column \citep{grinsztajn2022tree}). Benchmarks like %
TALENT \citep{ye2025closerlookdeeplearning} or TabArena \citep{erickson2025tabarenalivingbenchmarkmachine} have guided the development of new tabular-specific deep learning methods —such as foundation models \citep{hollmann2025accurate,qu2025tabicltabularfoundationmodel} and tailored architectures \cite{holzmuller2024better,gorishniy2025tabm} that challenge the dominance of industry standards such as XGBoost \citep{Chen_2016}. %

However, datasets in the wild rarely adhere to the strict numerical tables favored by current research; they are frequently populated with string entries, some short --- names or codes --- while other longer with richer semantic content. Despite this prevalence, the intersection of tabular learning and strings --- and the critical challenge of selecting their appropriate representation --- remains understudied due to the absence of a solid benchmarking suite. Existing benchmarks often prioritize readily-prepared numerical tables and effectively exclude the high-cardinality and semantic entries found in raw string-heavy tables. To bridge this gap, we need a robust set of datasets capable of supporting empirical work specifically on the domain of tables with strings.

Our contributions are: (i) \textbf{STRABLE}, a corpus of 108 real-world tables with raw strings and its string taxonomy; (ii) evidence that modular architectures currently dominate the Pareto frontier and that LLM embeddings need a critical post-processing step to perform well; (iii) an analysis establishing STRABLE's generalizability; (iv) evidence that lightweight encoders suffice on categorical-dominant tables, while large LLMs gain ground for free-text. 
The paper is organized as follows: Section~\ref{sec:tabular_learning_research} reviews the benchmarking landscape; Section~\ref{sec:building_the_benchmark} details STRABLE's construction and taxonomy; Section~\ref{sec:studying_tabular_learners} reports the empirical study; Section~\ref{sec:strable_generalizability} analyzes the stability of STRABLE's rankings; Section~\ref{sec:discussion_and_conclusion} concludes with the implications of our findings on future research.

\section{Context: tabular learning research}
\label{sec:tabular_learning_research}

\subsection{The tabular-learning benchmarking landscape}

\paragraph{A need for string tabular learning benchmarks}
The ``iron rule'' guiding machine-learning research is to compare pipelines on held-out data \citep{hardt2025emerging}. While %
model rankings remain surprisingly consistent across data splits \citep{roelofs2019meta,recht2019imagenet,hardt2025emerging}, %
no algorithm is optimal across all problem classes \citep{585893}. Rankings are domain-dependent, and models whose inductive biases match the data distribution perform best \citep{grinsztajn2022tree}. Introducing strings into a table fundamentally changes the data distribution. Can high-cardinality categorical encodings \citep{micci2001preprocessing,cerda2020encoding} suffice, or do we need models with different inductive biases that leverage string semantics \citep{kim2024cartepretrainingtransfertabular}? Can we identify the conditions that make each approach optimal? 
Answering these questions requires a benchmark that preserves raw, heterogeneous string entries. Existing benchmarks, while rigorous in their respective scopes, were not designed for this purpose. We organize them below by how they handle string features (see also \autoref{tab:benchmark_comparison}).

\paragraph{String-excluding benchmarks} Several widely-used benchmark suites focus on numerical or low-cardinality categorical features by design. PMLB \citep{romano2021pmlbv10opensource} explicitly replaces categorical and string-encoded features with numerical integer codes, and PMLBmini \citep{knauer2024pmlbminitabularclassificationbenchmark} inherits this property. OpenML-CC18 \citep{bischl2021openmlbenchmarkingsuites} filters out high-cardinality features through its post-one-hot feature-count cap. \citet{grinsztajn2022tree} removes categorical features with more than 20 categories and one-hot-encodes the rest, eliminating high-cardinality string content. TabReD \citep{rubachev2024tabredanalyzingpitfallsfilling} curates datasets that have already undergone feature-engineering, removing raw string content prior to evaluation. TabArena \citep{erickson2025tabarenalivingbenchmarkmachine} inherits similar curation choices from prior work.

\paragraph{String-flattening benchmarks} A larger group of benchmarks retains string columns but converts them to fixed numerical representations before evaluation, preventing the study of alternative string-handling strategies. \citet{mcelfresh2024neuralnetsoutperformboosted}%
, AMLB \citep{gijsbers2023amlbautomlbenchmark} and TabRepo \citep{salinas2024tabrepolargescalerepository} 
apply pipeline- or method-specific encodings; AMLB additionally excludes free-form text features at curation time. TALENT \citep{ye2025closerlookdeeplearning} and \citet{zabërgja2025tabulardatadeeplearning} use standard vectorization that treats strings as mathematical vectors, and BeyondArena \citep{purucker2026iidgeneraltabularfoundation} embeds all text with a single fixed language model. Concurrently, TEmBed \citep{vogel2026universaltabularembeddingsbenchmark} benchmarks tabular embeddings across cell, row, column, and table granularities on retrieval and similarity tasks, serializing tables to text for most encoders. %

\paragraph{Narrow string-aware benchmarks} Few benchmarks evaluate strings, though often with restrictive scopes. \citet{shi2021benchmarkingmultimodalautomltabular} pioneer the multimodal tabular setting with 18 text-rich datasets, but many of their tables rely almost entirely on text with only one or two tabular features. CARTE \citep{kim2024cartepretrainingtransfertabular} excludes predominantly missing columns and curates datasets with meaningful columns and discrete entries, it has a lower density of text columns that are moderately diverse (\autoref{fig:uniqueness_ratio_comparison_4_benchmarks}). TextTabBench \citep{mraz2025benchmarkingfoundationmodelstabular} identifies the mixed-modality gap in tables but focuses only on datasets with semantically rich ``free-text'' features, and MulTaBench \citep{arazi2026multabenchbenchmarkingmultimodaltabular} only on tasks where fine-tuned encoders help; both cover a small subset of the tables-with-string domain. 

\subsection{Progress in tabular learning}

\paragraph{Numerical tabular learners}
For years, Gradient-Boosted Decision Trees such as XGBoost and CatBoost \citep{prokhorenkova2019catboostunbiasedboostingcategorical} have dominated tabular data, outperforming deep learning on tabular benchmarks \citep{grinsztajn2022tree}. Recent tuned deep baselines like RealMLP \citep{holzmuller2024better} and parameter-efficient ensembles like TabM \citep{gorishniy2025tabm} have narrowed this gap, and the emergence of tabular foundation models -- %
TabPFN-2.5 \citep{grinsztajn2025tabpfn25advancingstateart} and TabICLv2 \citep{qu2026tabiclv2betterfasterscalable} -- has let in-context learning catch up with tree-based models.

\paragraph{End-to-end string table learners} Some end-to-end tabular learners accept tables with strings. Classical methods discard string surface form via static encodings: \texttt{CatBoost} treats string columns as categorical features and applies its native ordered target statistics scheme, while \texttt{Mambular} \citep{thielmann2025mambularsequentialmodeltabular} ordinally encodes strings and passes them through a learnable embedding lookup with a Mamba backbone. A more recent wave jointly models numbers, string entries, and column names to capture semantics: \texttt{CARTE} \citep{kim2024cartepretrainingtransfertabular} and \texttt{TARTE} \citep{kim2025table} pretrain transformers on string embeddings and numerical features; \texttt{TabSTAR} \citep{arazi_tabstar_2025} unfreezes a language-model encoder for target-aware semantics; and \texttt{ConTextTab} \citep{spinaci2025contexttabsemanticsawaretabularincontext} combines semantic encoders with cross-row and cross-column attention \citep{kossen2022selfattentiondatapointsgoingindividual}. 

\section{Building the STRABLE benchmark corpus}
\label{sec:building_the_benchmark}

\paragraph{Data Collection} To construct our benchmark corpus, we aggregated data from $33$ distinct sources spanning 8 application fields (\autoref{tab:dataset_distribution_table1.5}), ranging from large institutional repositories to community-driven platforms (see \autoref{app:datasets}). %
The raw data format varied significantly across these domains: while sources like HIFLD (Infrastructure) and ClinicalTrials.gov typically provided structured CSV, other fields like Commerce and Food heavily relied on web-scraped HTML tables or nested JSONs. From the corresponding available datasets, we manually selected tables as follows.
First, we only considered tables with at least two string columns and a minimum of $500$ samples (limit inspired by \cite{erickson2025tabarenalivingbenchmarkmachine}). 
Each table was paired with a target variable for supervised learning. The final corpus comprises 13 binary classification, 19 multi-class classification, and 76 regression tasks.

\paragraph{Minimal Preprocessing} We minimize preprocessing to reflect the reality that data preparation is a major bottleneck \citep{datanami2020dataprep,stonebraker2019machine}. Providing raw data enables studying models that automate this burden. Moreover,  preprocessing choices (e.g., specific categorical encoders) lack consensus and risk biasing the benchmark against future architectures that might process strings differently. We likewise perform no feature engineering so STRABLE reflects how learners behave on data as-found rather than as-curated. %
We flattened nested structures, removed duplicate rows, and dropped single-value columns, all-null columns, and rows with missing labels. To prevent leakage, we removed features $X_i$ where $X_i$ is a trivial function of the target. Missing entries were handled by the encoder-learner pipelines.
Since the benchmark focuses on I.I.D. evaluation, we adopted a snapshot strategy—extracting only the most recent available year of data. We sub-sampled large tables $(12/108)$ to a maximum of 75,000 rows for all pipelines, leaving rankings unchanged ($\tau = 0.96$, \Cref{app:downsampling}). %
For regression tasks, targets often exhibit heavy skew (e.g., wages). While individual applications may benefit from domain-specific loss functions, to ensure fair evaluation across diverse domains we applied a skewness-minimization protocol: from a set of candidate functions—including $\log(y)$, $\log(1+y)$, $\sqrt[3]{y}$, $\operatorname{arcsinh}(y)$, and $\operatorname{sgn}(y) \cdot \log(1+|y|)$ —we selected the transformation that minimized the skewness in the target distribution following \citep{Kuhn_13}.
STRABLE's results should therefore be read as a lower bound from which practitioners can measure the added value of their own domain-specific engineering.

\begin{figure}[t]
  \centering
  
  \begin{subfigure}[t]{0.62\textwidth}
    \centering
    \textbf{\sffamily\footnotesize (a) STRABLE strings are short and repetitive.}
    \vspace{0.3em}
    \includegraphics[width=\linewidth, trim={0.3cm 0.3cm 0.2cm 0.55cm}, clip]%
      {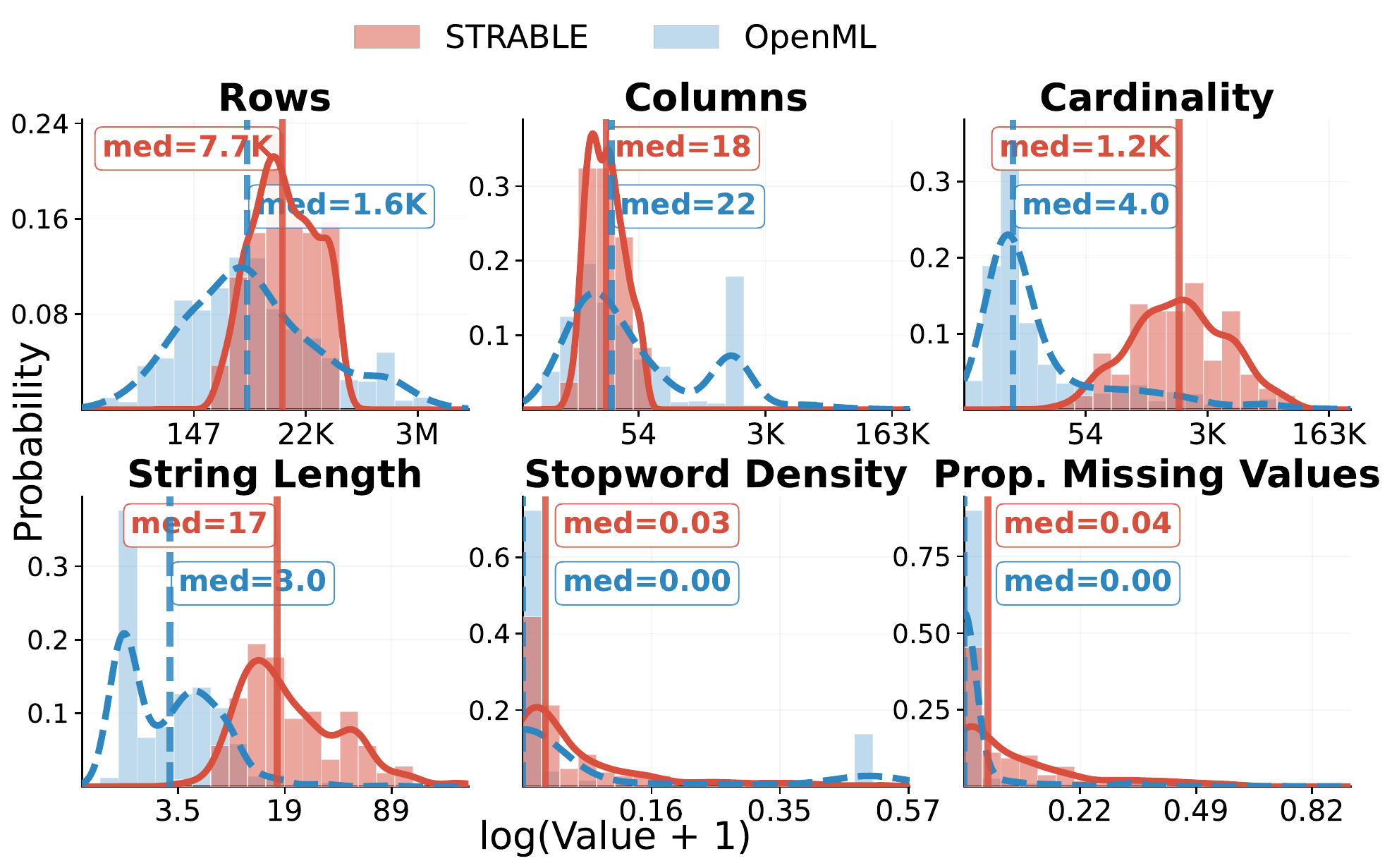}
    \phantomcaption
    \label{fig:strable_motivation_a}
  \end{subfigure}
  \hfill
  \begin{subfigure}[t]{0.37\textwidth}
    \centering
    \textbf{\sffamily\footnotesize (b) All models benefit from combining both modalities.}
    \vspace{0.3em}
    \includegraphics[width=\linewidth]%
      {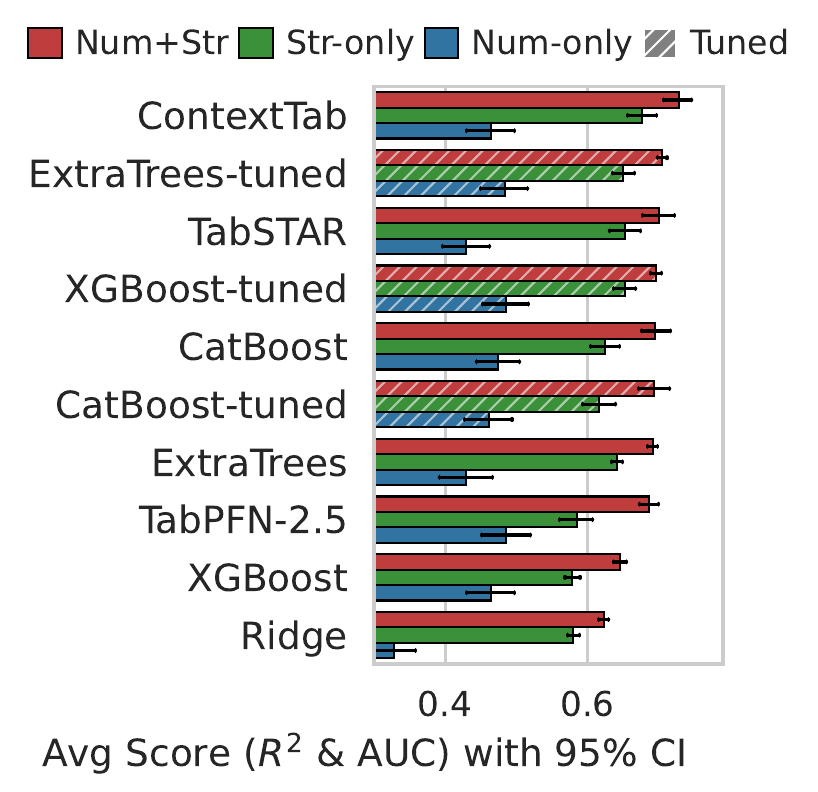}
    \phantomcaption
    \label{fig:strable_motivation_b}
  \end{subfigure}
  
  \caption{\textbf{(a)} STRABLE (solid) vs. OpenML (dashed); cardinality and string length aggregate over string columns. \textbf{(b)} Performance for Num-only, Str-only and full table (Num+Str) by learner.}
  \label{fig:strable_motivation}
\end{figure}

\paragraph{Insights into string tables in the wild captured by our corpus}
STRABLE's column metadata (\autoref{fig:strable_motivation_a}) follows heavy-tailed distributions: a median of 18 columns, 7.7K rows, 17-character strings per cell, and a median string-column cardinality of 1.2K. While the median number of columns is comparable to OpenML's -- a known tabular machine learning repository \citep{vanschoren2014openml,bischl2025openml} -- STRABLE tables contain roughly 5 times more rows. In addition, STRABLE's strings are shorter and more repetitive than in prior tabular-text studies (\autoref{fig:uniqueness_ratio_comparison_4_benchmarks}) because we include any string type rather than only long-form text. String columns can be broken into different categories (Appendix~\ref{subsec:string-taxonomy}): Names (22.78\%), Structured Codes (17.01\%), Free Text (8.23\%), Datetime (1.97\%) and Identifiers (0.5\%), with the remaining 49.45\% being plain Categoricals. Half (50.55\%) consist of modalities that string-excluding and string-flattening benchmarks typically ignore or destroy (e.g.: PMLB ignores Names, Structured Codes, Free Text, Datetimes, and Identifiers by collapsing them into integer labels; AMLB removes Free Text columns entirely, \citet{grinsztajn2022tree} drops categorical features with more than 20 items).

\section{Studying today's string tabular learners}
\label{sec:studying_tabular_learners}

Strings carry signal that tabular learning cannot afford to ignore. Indeed, \autoref{fig:strable_motivation_b} shows that integrating string columns yields a tangible performance improvement across the benchmark for every learner — averaged across encoders for modular pipelines, and across datasets for end-to-end ones — confirming that numeric and string features are complementary. This raises the question of which pipelines best exploit that signal. We run on the STRABLE corpus typical pipelines used today on tables with strings. We benchmark both modular and end-to-end (E2E) architectures.

\paragraph{Modular Pipelines}
Modular pipelines decompose the learning task into three distinct stages: a string encoder, a dimensionality reduction step for LLMs (post-processing), and a tabular learner.

\quad \textbf{Encoders.} We investigate a high-cardinality categorical encoder, \texttt{TargetEncoder} \citep{micci2001preprocessing}, as well as string encoders such as Tf-Idf on character-level n-grams \citep[StringEncoder in][]{skrub2026}, or Sentence Transformers \citep[using SBERT, ][]{reimers-2019-sentence-bert}. \Cref{app:encoder_learner_pipelines} drills down on encoders.  We use different LLMs (\autoref{tab:runtime_summary}) to embed strings, but focus on the following representative LLMs that were either used by an end-to-end model or appear prominently in the English MTEB benchmark \citep{muennighoff2023mtebmassivetextembedding}: 
All-MiniLM-L6-v2~\citep{reimers-2019-sentence-bert}, FastText~\citep{mikolov2017advancespretrainingdistributedword}, E5-small-v2~\citep{wang2022text}, LLaMA-3.1-8B~\citep{grattafiori2024llama3herdmodels}, Qwen3-Embedding-8B~\citep{qwen3embedding}, and Jasper-0.6B~ \citep{zhang2025jaspertokencompression600mtechnicalreport}.
Additionally, we include the TARTE encoder \citep{kim2025table}, which generates row-wise embeddings through pre-trained weights. 
We generate embeddings for each string cell independently. To obtain a tractable input for tabular learners while preserving relevant signal, these LLM representations are reduced to 30 dimensions. 

\quad \textbf{Post-processing.} We evaluate three distinct  strategies within this stage: (1) Principal Component Analysis\footnote{For Tf-Idf, we rather use an SVD, i.e. we do not center, to avoid densifying the sparse matrix, following \citet{skrub2026}.} (PCA) on the embeddings of each string column %
with 30 PC; (2) Standard scaling before PCA, which equalises per-dimension variance, 
preventing high-variance dimensions from dominating the principal 
components \citep{10.1098/rsta.2015.0202,sklearn_feature_scaling}; (3) Retain the first 30 embedding dimensions (No PCA). While being a natural choice for Matryoshka-trained models 
\citep{kusupati2024matryoshkarepresentationlearning} like Qwen-3-8B, 
whose leading dimensions are optimized for semantic content, this strategy is applied across all encoders to match the dimensionality of the PCA-based pipelines. 

\quad \textbf{Learners.} The resulting numerical tables are used as input to tabular learners of varying sophistication: Ridge \citep{a92f3c16-7c6e-31d3-b403-82d2b0a469e4}, Extra-Trees \citep{geurts2006extremely}, XGBoost, RealMLP, TabM, TabICLv2 and TabPFN-2.5. 

\textbf{E2E architectures} We apply on the raw tables end-to-end learners that handle strings without the need of an external encoder: CatBoost\footnote{CatBoost falls in this group as we do not couple it with an external encoder; instead, we declare string columns as categorical features and rely on CatBoost's native ordered target statistics.}, Mambular, TabSTAR, and ConTextTab\footnote{ConTextTab is pretrained on the T4 web corpus, raising an overlap question with STRABLE; an ablation removing it from the pool (Appendix~\ref{sec:contamination_contexttab}) shows its inclusion does not affect the benchmark's conclusions.}.

\subsection{Different LLM embeddings hint at different post-processing needs}
\label{subsec:postprocessing}

LLM-based encoders produce high-dimensional embeddings (up to 4096 dimensions for LLaMA-3.1-8B), which require dimensionality reduction before being passed to most tabular learners. The choice 
of reduction technique affects the downstream score: we compare the three variants %
and report the average score across five learners: Ridge, XGBoost, ExtraTrees, TabPFN-2.5, TabICLv2 in \autoref{fig:pca_postprocessing_comparison}.

\paragraph{Sensitivity to dimensionality reduction.} We observe distinct behaviors between model architectures when reducing embeddings to 30 dimensions. \textbf{Encoder-only models tolerate PCA reduction.} MiniLM-L6-v2, E5-base-v2 and BGE-large show only marginal gains under standard scaling or raw-dimension slicing, and Nemotron-1B is essentially flat: for these encoders, default PCA is a reasonable choice. Conversely, \textbf{decoder-only models hint at greater post-processing sensitivity.} LLaMA-3.1-8B, Qwen-3-8B, and OPT-6.7B all improve when default 30-PCA is replaced either by standard scaling prior to PCA or by retaining the first 30 raw dimensions (six post-processing variants in \autoref{fig:llama_qwen_pp_6methods}). Standard scaling rescales each dimension to have unit variance before PCA is applied, making every dimension contribute equally. The performance recovery of \texttt{StandScal + PCA} therefore suggests the degradation comes from a few embedding dimensions with large variance dominating the leading principal components. We confirm this empirically: decoder embeddings concentrate most of their variance in a few dimensions, while encoders spread it more evenly (\autoref{tab:gini_appendix}, \autoref{tab:gini_pairwise}). This aligns with a known characteristic of next-token-prediction models: their representations tend to cluster together into a narrow cone \citep{ethayarajh2019contextualcontextualizedwordrepresentations, gao2019representationdegenerationproblemtraining}, which we observe in \autoref{fig:cosine_before}: LLaMA-3.1-8B and Qwen-3-8B have an average cosine similarity of $\approx 0.57$, compared to just 0.25 for MiniLM-L6-v2. %

\paragraph{The importance of reduced-dimension embeddings.}
Our pipelines reduce the embeddings to 30 dimensions, e.g. by PCA. Previous research indicates that %
larger dimensions offer diminishing returns \citep{grinsztajn2023vectorizingstringentriesdata}. Indeed, given the prevalence of short strings in the benchmark (median $\approx$ 17), high-dimensional representations %
capture unnecessary complexity. Nevertheless, we vary the pipeline by using 60 and 120 PCA dimensions rather than 30; this yields worse results, increasing the runtime (\autoref{fig:llama_tabpfn_pca30_vs_60_vs_120_delta}). %

\begin{figure}[t!]
    \centering
    \includegraphics[width=\linewidth]{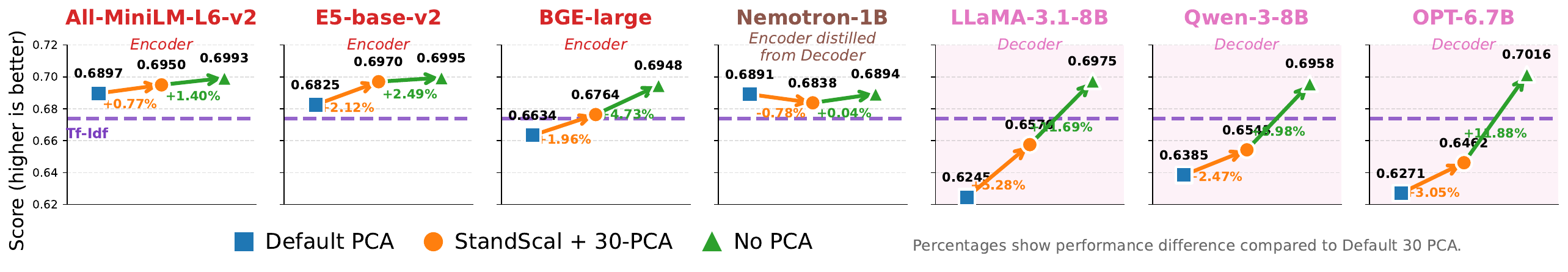}
    \caption{\textbf{Post-processing affects LLM-based embeddings, especially for decoder-only models.} Average score across 108 tables and five learners for 7 LM encoders under 
    three post-processing variants. Each panel is one encoder; 
    bars show the mean score under default 30-PCA (blue), standard 
    scaling before 30-PCA (orange), and direct slicing of the first 
    30 raw embedding dimensions (green). Percentages indicate 
    relative improvement over the default 30-PCA baseline. The dashed purple line marks Tf-Idf. A per-learner breakdown is provided in \autoref{fig:pca_postprocessing_comparison_appendix}.}
    \label{fig:pca_postprocessing_comparison}
\end{figure}

\subsection{Modular pipelines trump today's E2E learners on the Pareto frontier}
\label{subsec:modular_vs_e2e}

\paragraph{Modular pipelines set the ceiling for predictive performance.}
While E2E models are explicitly designed to learn joint representations of heterogeneous data, they are outperformed by modular pipelines on our benchmark. In terms of absolute predictive rank (\autoref{fig:cd_diagram_encoder-learner}), modular pipelines using post-processed Large Language Models (e.g., Qwen-3-8B) coupled with advanced tabular learners achieve the highest overall performance, consistently outperforming native E2E architectures like TabSTAR and ConTextTab. These E2E models were also shown to be weaker tabular learners than TabPFN on numeric-only tables \citep{erickson2025tabarenalivingbenchmarkmachine}, as also seen in \autoref{fig:strable_motivation_b}.

\paragraph{Lightweight encoders with advanced learners dominate the Pareto frontier.} Pure predictive rankings mask the computational overhead introduced by LLMs. When considering runtime, post-processed LLMs fail to dominate the Pareto frontier (\autoref{fig:comparative-pareto_optimality}). 
This gap is explained by the composition of the tables in the benchmark: for half of the tables, categorical strings are the leading string type (\autoref{fig:top10_ranking_by_leading_textfeature_type}). Consequently, lightweight encoders like Tf-Idf capture the necessary signal at a fraction of the computational cost, pushing them to the sweet spot of the Pareto frontier together with learners like TabICLv2 or TabPFN-2.5.
\begin{figure}[t!]
    \includegraphics[width=\textwidth]{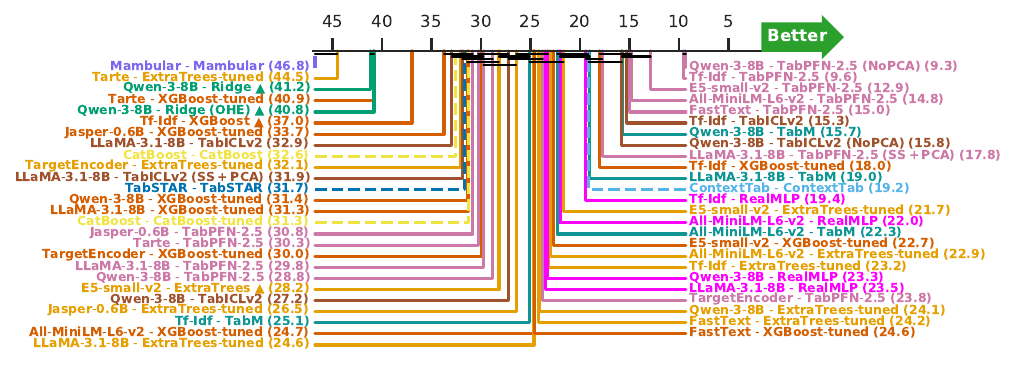}%
    \caption{\textbf{Critical difference diagram for 
    encoder-learner pipelines.} Pipelines' average rank across 
    the 108 datasets are shown in parentheses; lower is 
    better. Dashed lines are E2E, continuous lines are Modular. Pipelines connected by horizontal bars are not statistically distinguishable at the indicated level (test 
    statistic in \Cref{app:stat_tests}). \textbf{Modular pipelines cluster 
    at the top of the ranking.} Pipelines marked with $\blacktriangle$ 
    show only the best-performing encoder for that learner; \Cref{fig:cd_diagram_all} details the full 
    set of combinations. Abbreviations: \textit{SS+PCA} = standard scaling before 30-PCA; \textit{NoPCA} = first 30 raw embedding dimensions; \textit{OHE} = one-hot encoding (cardinality threshold 30).} 
    \label{fig:cd_diagram_encoder-learner}
    \definecolor{ridgegreen}{HTML}{009e73}
    \definecolor{pfnpink}{HTML}{cc79a7}
    \bigskip\medskip
    
    \begin{minipage}{0.55\textwidth}
        \centering
        \textbf{\sffamily\footnotesize (a) Trade-off between prediction performance and run time.}
        \vspace{0.3em}
        \includegraphics[width=\linewidth]{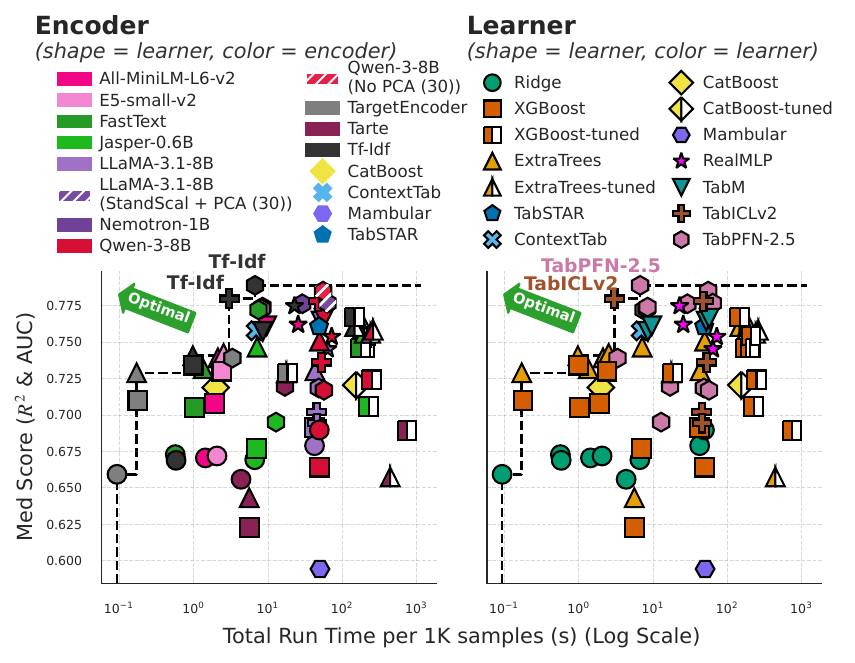}
    \end{minipage}
    \hfill
    \begin{minipage}{0.43\textwidth}
        \centering
        \textbf{\sffamily\footnotesize (b) Convergence of STRABLE benchmark rankings to the oracle.}
        \vspace{0.3em}
        \includegraphics[width=\linewidth]{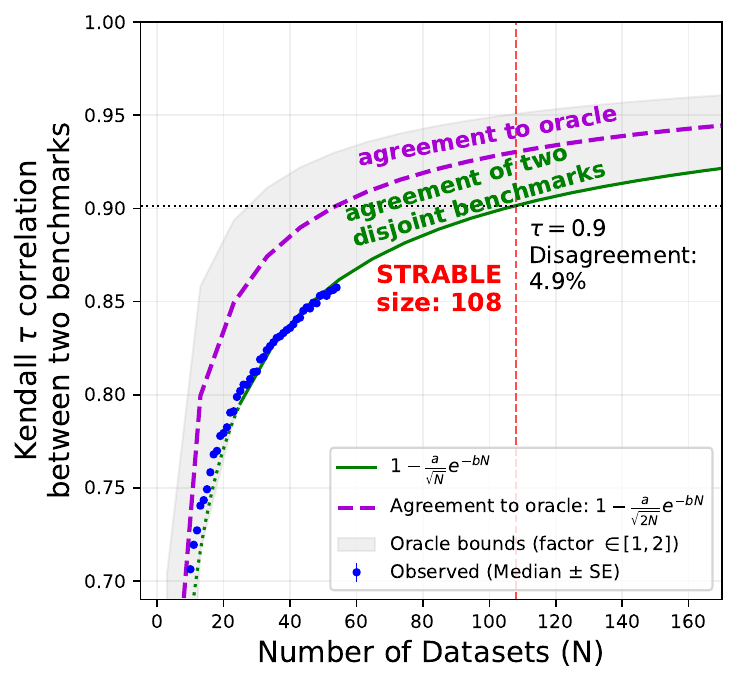}
    \end{minipage}
    
    \caption{\textbf{Pareto-optimality plot and benchmark ranking stability.} \textbf{(a)} Each point is a pipeline, colored by encoder on the left and by learner on the right. The dotted line is the pareto-optimality frontier. \textbf{Encoders explain much of the runtime}: for a given encoder, performance varies depending on the learner while runtime varies less (aside from tuning or not). \textbf{Simple and advanced learners benefit differently from varying encoders}: for a simple learner such as ridge \textcolor{ridgegreen}{$\mdblkcircle$}, complex encoders improve prediction performance. More sophisticated learners do not benefit from the most complex encoders (as with TabPFN-2.5 $\textcolor{pfnpink}{\hexagonblack}$), unless the encoder is treated with a post-processing method. \textbf{(b)} Blue points: observed Kendall-$\tau$ between disjoint STRABLE subsets (sub-sampling detailed in \autoref{fig:sampling_diagram_kendalltau}). Green curve: fitted asymptotic form. Purple curve: implied convergence to the oracle ranking. Grey band: finite-$N$ correction range (\Cref{sec:finite_n_correction}), reaching $\tau \approx 0.93$ at $N=108$. 
    \textbf{In expectation, a single benchmark is closer to the oracle than to another finite benchmark}, so $\mathbb{E}[\tau]$ between two benchmarks is a \textbf{lower bound} on oracle agreement.}
    \label{fig:combined_pareto_oracle}%
            \label{fig:comparative-pareto_optimality}%
        \label{fig:benchmark_stability_kendalltau}%

\end{figure}

\paragraph{The simplest models benefit more from heavier encoders.} For the simplest learners - linear baselines and ExtraTrees - using sophisticated encoders, such as LLMs, leads to better performance than Tf-Idf. This effect is more visible on average performance (\autoref{fig:comparative-pareto_optimality}a) than on ranks (\autoref{fig:cd_diagram_encoder-learner}), suggesting gains are not consistent across datasets, but are substantial in magnitude when they occur. Among foundation models, TabPFN-2.5 ranks slightly above TabICLv2 — likely an artifact of pretraining rather than the methods themselves (\autoref{tab:pca_col_estim}). %

\section{STRABLE generalizes beyond its specific datasets}
\label{sec:strable_generalizability}

\subsection{Sufficient datasets to approach oracle rankings}
\label{sec:sufficient_datasets}

A benchmark should provide model comparisons that generalize to datasets outside of the benchmark. %
In the following, we quantify the stability of model rankings depending on the number of datasets. To this end, we model our benchmark datasets as sampled from an unknown \emph{population} distribution of datasets. We then want to know how close the model ranking $R_N$ on our benchmark with $N$ datasets is to the oracle ranking $R_\infty$ on the population distribution, or to the model ranking $R_N'$ on a second benchmark with $N$ datasets sampled independently from the population distribution. To model the agreement between rankings, we use the Kendall-$\tau$ correlation \citep{kendall1938new}.
For two model rankings, Kendall-$\tau$ measures the fraction of pairs of models whose order disagrees in the two rankings: %
\begin{equation*}
    \tau_{N, N} := \tau(R_N, R_N') = 1 - 2\frac{\#\text{disagreeing model pairs}}{\#\text{model pairs}} \in [-1, 1]~,
\end{equation*}
and similarly $\tau_{N, \infty} := \tau(R_N, R_\infty)$.
Theory (detailed in \Cref{app:theory}) can relate the agreement to the oracle ranking to agreement between two finite-size benchmarks, $\tau_{N, N}$: for two independent benchmarks $R_N, R_N'$, their deviations from $R_\infty$ are independent. Hence, asymptotically, we expect twice the disagreement: $\bbE[1-\tau_{N, N}] \approx 2\bbE[1-\tau_{N, \infty}]$ (exact finite-$N$ factor $\approx 1.41$, \Cref{sec:finite_n_correction}). %

In other words, a comparison of two finite-size benchmarks accumulates variance from both sides, thus represents a "worst-case" scenario. For any given number of datasets, a benchmark is closer to the truth (the oracle) than it is to another finite benchmark, and the agreement between two independent benchmarks gives a lower bound for the convergence of a single benchmark to the oracle ranking. 
 
In practice, we estimate $\bbE[\tau_{N, N}]$ by picking disjoint subsets $D, D'$ of STRABLE and averaging over different choices of subsets. As these subsets can be at most 54 datasets, half the size of STRABLE, we derive the asymptotic dependency on the number of datasets theoretically, fit its parameters on subsamples of different sizes $N \leq 54$, extrapolate it to $N=108$, and then use the theoretical correction to estimate the oracle correlation $\bbE[\tau_{N, \infty}]$ for $N=108$. 
\Cref{fig:benchmark_stability_kendalltau}b shows the empirical estimate and asymptotic fit. 
The fitted theoretical agreement to oracle (purple line) illustrates how a single benchmark converges faster to the oracle ranking than the convergence between two benchmarks.

The estimates on \cref{fig:benchmark_stability_kendalltau}b show that for two benchmarks of size $N=54$, $\bbE[\tau(R, R')] \approx 0.86$, corresponding to 7\% disagreeing model pairs. Extrapolated to $N=108$, we expect $\bbE[\tau(R, R')] \approx 0.9$, corresponding to 5\% disagreeing pairs, whereas for the oracle comparison at $N=108$ we get $\bbE[\tau_{N,\infty}] \approx 0.93$, corresponding to 3.5\% disagreeing pairs. Overall, this shows that the size of STRABLE allows us to extract model rankings that are close to the oracle. 

\subsection{Pipeline rankings are stable across domains and preprocessing choices}
\label{sec:stability_benchmark}

\paragraph{Across application fields}
To assess generalizability beyond specific domains, we apply a Leave-One-Domain-Out procedure: for each of the 8 categories defined in \autoref{tab:dataset_distribution_table1.5}, we measure the Kendall-$\tau$ rank correlation between the model ranking on that category alone and the ranking on the full benchmark (\autoref{fig:combined_stability_a}). To separate genuine domain effects from sampling noise, we construct a size-matched null reference per domain by drawing $B = 200$ random subsets of size $n_D$ from the full benchmark and computing $\tau$ against the full-benchmark ranking. The 95\% confidence interval is shown as a grey band behind each bar; its width reflects the statistical power at each domain size. A bar falling inside its band is indistinguishable from a random sample of the same size, and thus representative of the full benchmark. Only Food ($n=6$, $\tau=0.27$) falls below its null band, indicating that its high-lexical-diversity product reviews disrupt the ranking more than a random size-6 subset would (domain-level meta-features in \autoref{tab:domain_metafeatures_stability}); Education is borderline ($\tau=0.54$). All other domains lie within their bands, supporting the conclusion that STRABLE's pipeline ranking is largely domain-independent.

\begin{figure}[t!]
    \centering
    
    \begin{subfigure}[t]{0.32\textwidth}
        \centering
        \textbf{\sffamily\footnotesize (a) STRABLE rankings are robust across application fields.}
        \vspace{0.3em}
        \includegraphics[width=\linewidth, height=4cm, keepaspectratio]{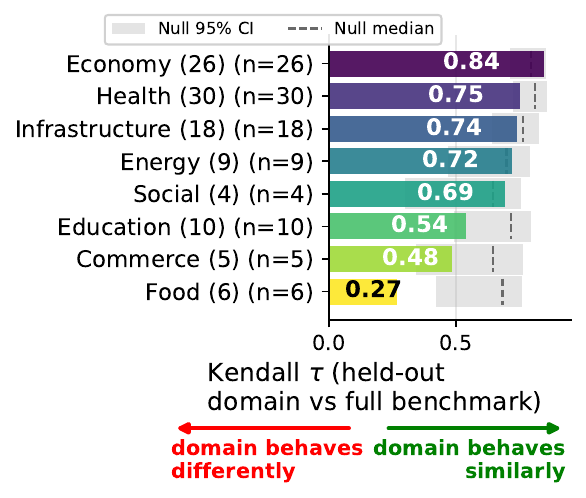}
        \phantomcaption
        \label{fig:combined_stability_a}
    \end{subfigure}\hfill
    \begin{subfigure}[t]{0.32\textwidth}
        \centering
        \textbf{\sffamily\footnotesize (b) Data preparation choices do not alter pipeline rankings.}
        \vspace{0.3em}
        \includegraphics[width=\linewidth, height=4cm, keepaspectratio]{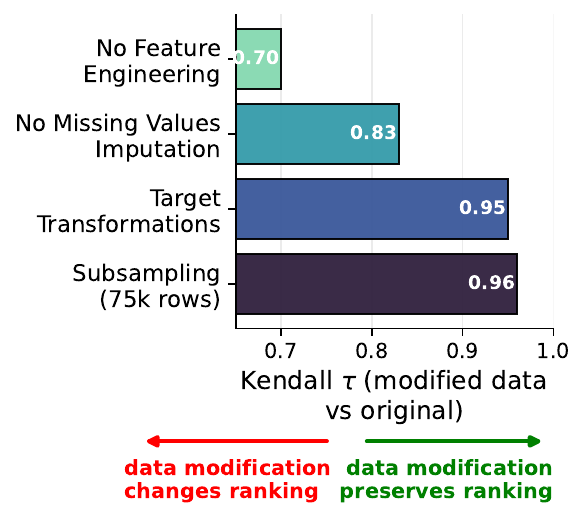}
        \phantomcaption
        \label{fig:combined_stability_b}
    \end{subfigure}\hfill
    \begin{subfigure}[t]{0.32\textwidth}
        \centering
        \textbf{\sffamily\footnotesize (c) Avg words/cell is the main ranking disruptor.}
        \vspace{0.3em}
        \includegraphics[width=\linewidth, height=4cm, keepaspectratio]{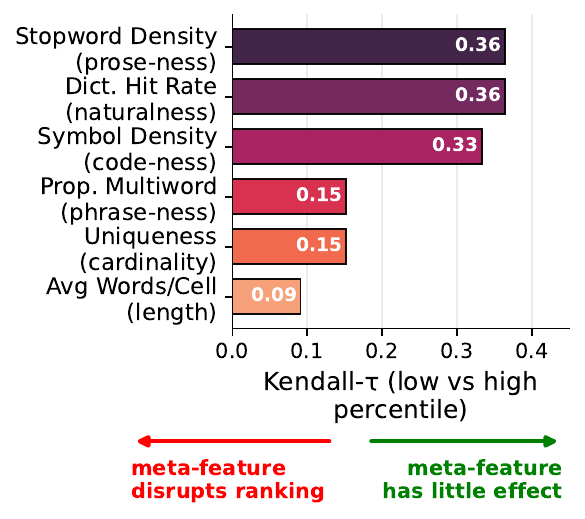}
        \phantomcaption
        \label{fig:combined_stability_c}
    \end{subfigure}
    
    \caption{\textbf{(a)} Kendall-$\tau$ correlation between application-specific subsets and the full benchmark (numbers in parentheses show the number of tables per application field). \textbf{(b)} Each row reports Kendall-$\tau$ between STRABLE's ranking and the ranking of the opposite data preparation (e.g., applying feature engineering or missing-value imputation, which STRABLE does not; or removing target transformations and subsampling, which STRABLE does). All values are $\tau \approx 0.7$ or higher. \textbf{(c)} Kendall-$\tau$ between the rankings on the lower and upper 33rd percentile of each string meta-feature; lower $\tau$ indicates stronger disruption.}
    \label{fig:combined_stability}
\end{figure}

\paragraph{Across data preparation choices}
\label{sec:stability}
Beyond domain selection, benchmark conclusions can also be 
sensitive to choices in the data-preparation pipeline. We test four such choices, detailed in Section~\ref{sec:building_the_benchmark}: 
subsampling large tables (default 75K rows vs. full size on 8  of the 12 datasets exceeding the cap), manual feature engineering (raw vs. parsing on 44/108 tables of dates, ranges, coordinates, drug strengths, fiscal periods), target transformation for skewed regression targets (skewness-minimizing transform applied to 61 regression tasks vs. raw target), and missing-value handling (per-learner native handling vs. mean/mode imputation). For each choice, we recompute the full pipeline ranking under an alternative policy and measure its Kendall-$\tau$ correlation with the default ranking. \autoref{fig:combined_stability_b} summarizes the result: rankings are highly preserved across all four choices ($\tau \in [0.7, 0.96]$), and so are effect sizes: the Tf-Idf-vs-LLM score gap shifts by a median of $\le 0.005$ and at most $0.03$ under any policy (\autoref{tab:gap_deltas}). Per-policy rankings are reported in \autoref{fig:raw_vs_engineered}--\ref{fig:full_vs_sampled_comparison} in the Appendix.

\subsection{Dataset features that drive ranking shifts}
\label{sec:heterogeneity}

Section~\ref{sec:stability_benchmark} established that STRABLE 
rankings are robust to data preprocessing and applications. We now turn the question around: \emph{what} dataset features cause  
rankings to shift? We identify 
these regimes by computing the Kendall-$\tau$ between 
rankings on the upper and lower 33-percentile of six string 
meta-features.

\paragraph{String length is the dominant disruptor of model rankings.} 
We compute six meta-features capturing different facets of 
``string-ness'' (Appendix~\ref{sec:profiling_methodology}) and rank them by their ability to disrupt 
pipeline rankings (\autoref{fig:combined_stability_c}). 
Average words per cell is the single biggest disruptor 
($\tau \approx 0.09$), and after it cardinality ($\tau \approx 0.15$): long string cells and cardinality almost entirely overturn 
the ranking of pipelines. Content-type features (stopword 
density, dictionary hit rate, symbol density) form a secondary 
cluster ($\tau \approx 0.33$--$0.36$).

\paragraph{Free text is the regime where the ranking changes.}
On tables whose leading string type is Categorical, Names, or Structured Code
(\autoref{fig:top10_ranking_by_leading_textfeature_type}),
the top-10 pipelines mirror the global ranking:
Tf-Idf and lightweight LM encoders paired with TabPFN-2.5 dominate. The five tables dominated by Free Text tell a different story:
every large LLM (LLaMA-3.1-8B, Qwen-3-8B, Jasper-0.6B) enters the top-10 under default 30-PCA paired with TabPFN-2.5, \textbf{indicating that long text carries
LLM-accessible signal that TabPFN-2.5 can exploit}. Since most STRABLE string columns are categorical, Tf-Idf suffices for them, while free-text columns benefit from LLM encoders. A cardinality threshold (CT=30) operationalizes this split by routing high-cardinality columns to the LLM and low-cardinality ones to OHE or the learner's native handling (\autoref{fig:ridge_xgb_tabpfn_llama_qwen_nemotron_CT30}). Columns routed to the LLM under CT=30 contain 2.6$\times$ more words per cell (\autoref{fig:CT_30_threshold_vs_string_index}), confirming that the threshold isolates the long-text signal LLM encoders are best equipped to handle.

\begin{figure}[t]
    \makebox[\linewidth][c]{\includegraphics[width=\linewidth]{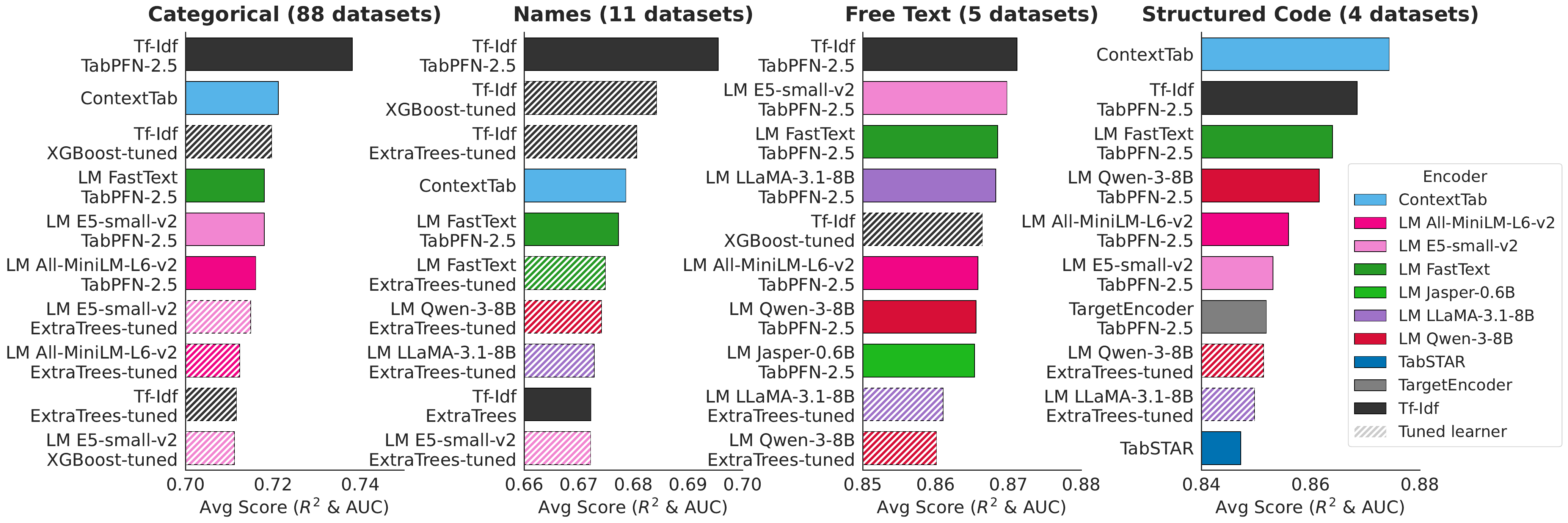}}
    \caption{\textbf{Top-10 pipelines per leading string type.} Datasets are grouped by their most frequent string type. In the Free Text regime large LLMs enter the top-10 paired with TabPFN-2.5; all other types mirror the global ranking (lightweight encoders at the top paired with TabPFN-2.5, and LM encoders paired with light learners like ExtraTrees). ConTextTab leads the Structured Code panel, plausibly aided by code-rich T4 pretraining tables — though this raises a contamination concern (\Cref{sec:contamination_contexttab}). Hatched bars indicate tuned models.}
\label{fig:top10_ranking_by_leading_textfeature_type}
\end{figure}

\section{Discussion and Conclusion}
\label{sec:discussion_and_conclusion}

\paragraph{STRABLE identifies effective pipelines.}
We introduce STRABLE to foster research in tabular learning with 
strings, providing a robust arena of 108 diverse learning 
problems that embraces the heterogeneity of real-world tables 
rather than the numerical purity of prior benchmarks. STRABLE provides a robust and stable benchmarking of learning on string tables: the 
ranking it produces converges toward the oracle ranking 
($\tau \approx 0.93$ at $N=108$) and remains robust to application fields as well as
data preparation choices
(Section~\ref{sec:stability}). Our benchmark yields concrete 
guidelines for practitioners: lightweight string encoders 
combined with sophisticated tabular learners (such as TabPFN-2.5) 
currently dominate the Pareto frontier of performance and runtime, 
outperforming end-to-end architectures designed for 
string-tabular learning. LLMs become performance-competitive 
once their embeddings undergo appropriate post-processing — 
standard scaling or direct slicing of the first $n$ dimensions — to overcome the bottleneck of traditional PCA.

\paragraph{Refining the role of large LLMs for tabular data with strings.}
Prior work has reported that larger language models improve performance on tabular data with strings \citep{grinsztajn2023vectorizingstringentriesdata}. STRABLE refines this picture along three axes. First, sophisticated encoders deliver consistent gains paired with simple baseline learners (Ridge, ExtraTrees), and --- once their embeddings receive appropriate post-processing --- become competitive with lightweight encoders even when paired with state-of-the-art tabular learners (TabPFN-2.5, TabICLv2). Second, STRABLE's strings are short (median 17 characters), which reduces the role of natural language understanding compared to settings with longer free-text fields. Third, decoder-only LLM embeddings are sensitive to default PCA-based dimensionality reduction, which can understate their performance. Together with the heterogeneity findings of Section~\ref{sec:heterogeneity}, these results suggest that the value of encoding with large LLMs is contingent on the string distribution, the learner they are paired with, and the dimensionality reduction strategy applied to their embeddings.

\paragraph{Limitations.}While our work establishes a new standard for this domain, it is not without limitations. STRABLE reflects the string distribution in data-science tables, with short categorical strings rather than long-form text. It thus enables only limited study of sentence-heavy tables. In addition, it does not address time-series specific validation protocols, leaving the exploration of temporal dynamics as well as grouped data for future work. 

\paragraph{Looking forward} Adapting generative LLMs to tables requires careful alignment of their embedding geometries with tabular learners. Consequently, this research highlights the need for better dimensionality reduction methods for these embeddings. Through this benchmark we hope to catalyze a new wave of research into hybrid architectures that adapt to the diversity of tables, including leveraging semantic understanding when needed.

\begin{ack}
The authors acknowledge support in part by the French Agence Nationale de la Recherche under the TaFoMo project (ANR-25-CE23-1822) and the Hi! PARIS and ANR/France 2030 program (ANR-23-IACL-0005). This work was performed using HPC resources from GENCI–IDRIS (Grant 2025-AD011017153).
L.P. acknowledges funding by the Deutsche Forschungsgemeinschaft (DFG, German Research Foundation) under SFB 1597 (SmallData), grant number 499552394. F.H. acknowledges the financial support of the Hector Foundation. AA, ES, and RR are supported by an Israel Ministry of Science and Technology (MOST) grant on multi-modal AI. ES is supported by a Google PhD Fellowship.
\end{ack}

\section*{Author Contributions}
G.B. led the code development, led the experiments, led the writing of the paper, led the evaluation and visualizations, led dataset curation, and led the management of the collaboration.
M.J.K. led dataset curation, led the code development and contributed to the experiments.
F.L. contributed to the writing of the Appendix A of the paper, to the evaluation, and helped with reviewing and editing.
L.P. helped with writing the paper, contributed to the evaluation and visualizations, and helped with reviewing and editing.
A.A. helped with writing the paper, contributed to reviewing and editing, and helped with some of the TabSTAR experiments.
E.S. contributed to reviewing and editing.
R.R. supervised E.S. and A.A. during the project.
F.H. supervised L.P. during the project.
M.L.M. contributed to writing the paper, contributed to reviewing and editing, helped with the management of the collaboration, and helped with supervising the project, helped with evaluation and visualizations. 
D.H. contributed to writing the paper, contributed to reviewing and editing, helped with the management of the collaboration, and helped with supervising the project, helped with evaluation and visualizations. 
G.V. contributed to writing the paper, contributed to reviewing and editing, helped with the management of the collaboration, helped with supervising the project, helped with evaluation and visualizations,  secured funding for the project, and supervised G.B. during the project.

We thank Célestin Eve for helpful discussions, specifically for the idea of investigating whether disagreeing model pairs involved the top-ranked model, which led to the proof in Section \ref{app:position1-bias}. We thank Yangyi Shen for identifying a look-up task within table 26 (Appendix~\ref{descr_data}), which leaves the rankings unchanged (Kendall-$\tau = 0.98$ on mean scores and $0.99$ on mean ranks;
identical top-10).

\section*{Competing Interests}
D.H. is one of the authors of RealMLP. 
L.P. and F.H. are a subset of the authors of TabPFN-2.6. 
D.H. and G.V. are a subset of the authors of TabICLv2.
L.P. and D.H. are the maintainers of TabArena.  
G.V. is a core contributor to skrub.
G.V. is (co)founder of scikit-learn.
L.P., A.A., and F.H. are affiliated with Prior Labs, a company focused on developing tabular foundation models. 
G.V. is affiliated with probabl, a company focused on developing software for industrial-grade data science. M.J.K. is affiliated to SAP at publication time.
The authors declare no other competing interests.

\bibliographystyle{plainnat}
\bibliography{references}

@inproceedings{recht2019imagenet,
  title={Do imagenet classifiers generalize to imagenet?},
  author={Recht, Benjamin and Roelofs, Rebecca and Schmidt, Ludwig and Shankar, Vaishaal},
  booktitle={International conference on machine learning},
  pages={5389--5400},
  year={2019},
  organization={PMLR}
}

@article{kendall1938new,
  title={A new measure of rank correlation},
  author={Kendall, Maurice G},
  journal={Biometrika},
  volume={30},
  number={1-2},
  pages={81--93},
  year={1938},
  publisher={Oxford University Press}
}

@misc{purucker2026iidgeneraltabularfoundation,
      title={Beyond IID: How General Are Tabular Foundation Models, Really?}, 
      author={Lennart Purucker and Andrej Tschalzev and Nick Erickson and Gioia Blayer and David Holzmüller and Alan Arazi and Alexander Pfefferle and Mustafa Tajjar and Gaël Varoquaux and Frank Hutter},
      year={2026},
      eprint={2606.30410},
      archivePrefix={arXiv},
      primaryClass={cs.LG},
      url={https://arxiv.org/abs/2606.30410}, 
}

@misc{kim2026exploringdifferencestabularenterprise,
      title={Exploring Differences Between Tabular Enterprise Data and Public Benchmarks}, 
      author={Myung Jun Kim and Maximilian Schambach and Frank Essenberger and Andre Sres and Johannes Höhne},
      year={2026},
      eprint={2606.30452},
      archivePrefix={arXiv},
      primaryClass={cs.LG},
      url={https://arxiv.org/abs/2606.30452}, 
}

@misc{kossen2022selfattentiondatapointsgoingindividual,
      title={Self-Attention Between Datapoints: Going Beyond Individual Input-Output Pairs in Deep Learning}, 
      author={Jannik Kossen and Neil Band and Clare Lyle and Aidan N. Gomez and Tom Rainforth and Yarin Gal},
      year={2022},
      eprint={2106.02584},
      archivePrefix={arXiv},
      primaryClass={cs.LG},
      url={https://arxiv.org/abs/2106.02584}, 
}

@misc{arazi2026multabenchbenchmarkingmultimodaltabular,
      title={MulTaBench: Benchmarking Multimodal Tabular Learning with Text and Image}, 
      author={Alan Arazi and Eilam Shapira and Shoham Grunblat and Mor Ventura and Elad Hoffer and Gioia Blayer and David Holzmüller and Lennart Purucker and Gaël Varoquaux and Frank Hutter and Roi Reichart},
      year={2026},
      eprint={2605.10616},
      archivePrefix={arXiv},
      primaryClass={cs.LG},
      url={https://arxiv.org/abs/2605.10616}, 
}

@article{roelofs2019meta,
  title={A meta-analysis of overfitting in machine learning},
  author={Roelofs, Rebecca and Shankar, Vaishaal and Recht, Benjamin and Fridovich-Keil, Sara and Hardt, Moritz and Miller, John and Schmidt, Ludwig},
  journal={Advances in neural information processing systems},
  volume={32},
  year={2019}
}

@article{grinsztajn2022tree,
  title={Why do tree-based models still outperform deep learning on typical tabular data?},
  author={Grinsztajn, L{\'e}o and Oyallon, Edouard and Varoquaux, Ga{\"e}l},
  journal={Advances in neural information processing systems},
  volume={35},
  pages={507--520},
  year={2022}
}

@misc{hardt2025emerging,
  author = {Moritz Hardt},
  title = {The Emerging Science of Machine Learning Benchmarks},
  year = {2025},
  howpublished = {Online at \url{https://mlbenchmarks.org}},
  note = {Manuscript}
}

@misc{kim2024cartepretrainingtransfertabular,
      title={CARTE: Pretraining and Transfer for Tabular Learning}, 
      author={Myung Jun Kim and Léo Grinsztajn and Gaël Varoquaux},
      year={2024},
      eprint={2402.16785},
      archivePrefix={arXiv},
      primaryClass={cs.LG},
      url={https://arxiv.org/abs/2402.16785}, 
}

@article{kim2025table,
  title={Table Foundation Models: on knowledge pre-training for tabular learning},
  author={Kim, Myung Jun and Lefebvre, F{\'e}lix and Brison, Ga{\"e}tan and Perez-Lebel, Alexandre and Varoquaux, Ga{\"e}l},
  journal={TMLR},
  year={2025}
}

@inproceedings{bischl2021openmlbenchmarkingsuites,
  title={OpenML Benchmarking Suites},
  author={Bischl, Bernd and Casalicchio, Giuseppe and Feurer, Matthias and Gijsbers, Pieter and Hutter, Frank and Lang, Michel and Mantovani, Rafael Gomes and van Rijn, Jan N and Vanschoren, Joaquin},
  booktitle={Proceedings of the NeurIPS 2021 Datasets and Benchmarks Track},
  year={2021}
}

@misc{romano2021pmlbv10opensource,
      title={PMLB v1.0: An open source dataset collection for benchmarking machine learning methods}, 
      author={Joseph D. Romano and Trang T. Le and William La Cava and John T. Gregg and Daniel J. Goldberg and Natasha L. Ray and Praneel Chakraborty and Daniel Himmelstein and Weixuan Fu and Jason H. Moore},
      year={2021},
      eprint={2012.00058},
      archivePrefix={arXiv},
      primaryClass={cs.LG},
      url={https://arxiv.org/abs/2012.00058}, 
}

@misc{rubachev2024tabredanalyzingpitfallsfilling,
      title={TabReD: Analyzing Pitfalls and Filling the Gaps in Tabular Deep Learning Benchmarks}, 
      author={Ivan Rubachev and Nikolay Kartashev and Yury Gorishniy and Artem Babenko},
      year={2024},
      eprint={2406.19380},
      archivePrefix={arXiv},
      primaryClass={cs.LG},
      url={https://arxiv.org/abs/2406.19380}, 
}

@inproceedings{gorishniy2025tabm,
  title={TabM: Advancing tabular deep learning with parameter-efficient ensembling},
  author={Gorishniy, Yury and Kotelnikov, Akim and Babenko, Artem},
  booktitle={The Thirteenth International Conference on Learning
Representations},
    year={2025}
}

@article{holzmuller2024better,
  title={Better by default: Strong pre-tuned mlps and boosted trees on tabular data},
  author={Holzm{\"u}ller, David and Grinsztajn, L{\'e}o and Steinwart, Ingo},
  journal={Advances in Neural Information Processing Systems},
  volume={37},
  pages={26577--26658},
  year={2024}
}

@inproceedings{qu2025tabicltabularfoundationmodel,
  title={TabICL: A Tabular Foundation Model for In-Context Learning on Large Data},
  author={Qu, Jingang and Holzm{\"u}ller, David and Varoquaux, Ga{\"e}l and Le Morvan, Marine},
  booktitle={Forty-second International Conference on Machine Learning},
    year={2025}
}

@article{hollmann2025accurate,
  title={Accurate predictions on small data with a tabular foundation model},
  author={Hollmann, Noah and M{\"u}ller, Samuel and Purucker, Lennart and Krishnakumar, Arjun and K{\"o}rfer, Max and Hoo, Shi Bin and Schirrmeister, Robin Tibor and Hutter, Frank},
  journal={Nature},
  volume={637},
  number={8045},
  pages={319--326},
  year={2025},
  publisher={Nature Publishing Group UK London}
}

@article{pedregosa2011scikit,
  title={Scikit-learn: Machine learning in Python},
  author={Pedregosa, Fabian and Varoquaux, Ga{\"e}l and Gramfort, Alexandre and Michel, Vincent and Thirion, Bertrand and Grisel, Olivier and Blondel, Mathieu and Prettenhofer, Peter and Weiss, Ron and Dubourg, Vincent and others},
  journal={the Journal of machine Learning research},
  volume={12},
  pages={2825--2830},
  year={2011},
  publisher={JMLR. org}
}

@article{geurts2006extremely,
  title={Extremely randomized trees},
  author={Geurts, Pierre and Ernst, Damien and Wehenkel, Louis},
  journal={Machine learning},
  volume={63},
  number={1},
  pages={3--42},
  year={2006},
  publisher={Springer}
}

@article{gardner2024large,
  title={Large scale transfer learning for tabular data via language modeling},
  author={Gardner, Josh and Perdomo, Juan C and Schmidt, Ludwig},
  journal={Advances in Neural Information Processing Systems},
  volume={37},
  pages={45155--45205},
  year={2024}
}

@misc{erickson2025tabarenalivingbenchmarkmachine,
      title={TabArena: A Living Benchmark for Machine Learning on Tabular Data}, 
      author={Nick Erickson and Lennart Purucker and Andrej Tschalzev and David Holzmüller and Prateek Mutalik Desai and David Salinas and Frank Hutter},
      year={2025},
      eprint={2506.16791},
      archivePrefix={arXiv},
      primaryClass={cs.LG},
      url={https://arxiv.org/abs/2506.16791}, 
}

@article{mcelfresh2024neuralnetsoutperformboosted,
  title={When do neural nets outperform boosted trees on tabular data?},
  author={McElfresh, Duncan and Khandagale, Sujay and Valverde, Jonathan and Prasad C, Vishak and Ramakrishnan, Ganesh and Goldblum, Micah and White, Colin},
  journal={Advances in Neural Information Processing Systems},
  volume={36},
  pages={76336--76369},
  year={2023}
}

@article{gijsbers2023amlbautomlbenchmark,
  title={Amlb: an automl benchmark},
  author={Gijsbers, Pieter and Bueno, Marcos LP and Coors, Stefan and LeDell, Erin and Poirier, S{\'e}bastien and Thomas, Janek and Bischl, Bernd and Vanschoren, Joaquin},
  journal={Journal of Machine Learning Research},
  volume={25},
  number={101},
  pages={1--65},
  year={2024}
}

@inproceedings{salinas2024tabrepolargescalerepository,
  title={TabRepo: A Large Scale Repository of Tabular Model Evaluations and its AutoML Applications},
  author={Salinas, David and Erickson, Nick},
  booktitle={AutoML Conference 2024 (ABCD Track)},
  year={2024}
}

@misc{zabërgja2025tabulardatadeeplearning,
      title={Tabular Data: Is Deep Learning all you need?}, 
      author={Guri Zabërgja and Arlind Kadra and Christian M. M. Frey and Josif Grabocka},
      year={2025},
      eprint={2402.03970},
      archivePrefix={arXiv},
      primaryClass={cs.LG},
      url={https://arxiv.org/abs/2402.03970}, 
}

@Misc{datanami2020dataprep,
  title = {Data Prep Still Dominates Data Scientists' Time, Survey Finds},
  author = {Datanami},
  year  = {2020},
  url   = {https://www.datanami.com/2020/07/06/data-prep-still-dominates-data-scientists-time-survey-finds/},
}

@article{stonebraker2019machine,
  title={Machine learning and big data: What is important?},
  author={Stonebraker, Michael and El Kindi Rezig},
  journal={IEEE Data Eng. Bull.},
  volume={42},
  number={4},
  pages={3--7},
  year={2019}
}

@article{cerda2020encoding,
  title={Encoding high-cardinality string categorical variables},
  author={Cerda, Patricio and Varoquaux, Ga{\"e}l},
  journal={IEEE Transactions on Knowledge and Data Engineering},
  volume={34},
  number={3},
  pages={1164--1176},
  year={2020},
  publisher={IEEE}
}

@article{micci2001preprocessing,
  title={A preprocessing scheme for high-cardinality categorical attributes in classification and prediction problems},
  author={Micci-Barreca, Daniele},
  journal={ACM SIGKDD explorations newsletter},
  volume={3},
  number={1},
  pages={27--32},
  year={2001},
  publisher={ACM New York, NY, USA}
}

@misc{ye2025closerlookdeeplearning,
      title={A Closer Look at Deep Learning Methods on Tabular Datasets}, 
      author={Han-Jia Ye and Si-Yang Liu and Hao-Run Cai and Qi-Le Zhou and De-Chuan Zhan},
      year={2025},
      eprint={2407.00956},
      archivePrefix={arXiv},
      primaryClass={cs.LG},
      url={https://arxiv.org/abs/2407.00956}, 
}

@misc{mraz2025benchmarkingfoundationmodelstabular,
      title={Towards Benchmarking Foundation Models for Tabular Data With Text}, 
      author={Martin Mráz and Breenda Das and Anshul Gupta and Lennart Purucker and Frank Hutter},
      year={2025},
      eprint={2507.07829},
      archivePrefix={arXiv},
      primaryClass={cs.LG},
      url={https://arxiv.org/abs/2507.07829}, 
}

@inproceedings{knauer2024pmlbminitabularclassificationbenchmark,
  title={PMLBmini: A Tabular Classification Benchmark Suite for Data-Scarce Applications},
  author={Knauer, Ricardo and Grimm, Marvin and Rodner, Erik},
  booktitle={AutoML Conference 2024 (ABCD Track)},
  year={2024}
}

@ARTICLE{585893,
  author={Wolpert, D.H. and Macready, W.G.},
  journal={IEEE Transactions on Evolutionary Computation}, 
  title={No free lunch theorems for optimization}, 
  year={1997},
  volume={1},
  number={1},
  pages={67-82},
  doi={10.1109/4235.585893}}

@misc{skrub2026,
  title= {Skrub software},
  author= {Skrub},
  year= {2026},
  howpublished={\url{https://skrub-data.org}}
}

@article{wang2022text,
  title={Text embeddings by weakly-supervised contrastive pre-training},
  author={Wang, Liang and Yang, Nan and Huang, Xiaolong and Jiao, Binxing and Yang, Linjun and Jiang, Daxin and Majumder, Rangan and Wei, Furu},
  journal={arXiv preprint arXiv:2212.03533},
  year={2022}
}

@inproceedings{reimers-2019-sentence-bert,
  title = "Sentence-BERT: Sentence Embeddings using Siamese BERT-Networks",
  author = "Reimers, Nils and Gurevych, Iryna",
  booktitle = "Proceedings of the 2019 Conference on Empirical Methods in Natural Language Processing",
  month = "11",
  year = "2019",
  publisher = "Association for Computational Linguistics",
}

@article{prokhorenkova2019catboostunbiasedboostingcategorical,
  title={CatBoost: unbiased boosting with categorical features},
  author={Prokhorenkova, Liudmila and Gusev, Gleb and Vorobev, Aleksandr and Dorogush, Anna Veronika and Gulin, Andrey},
  journal={Advances in neural information processing systems},
  volume={31},
  year={2018}
}

@article{qwen3embedding,
  title={Qwen3 Embedding: Advancing Text Embedding and Reranking Through Foundation Models},
  author={Zhang, Yanzhao and Li, Mingxin and Long, Dingkun and Zhang, Xin and Lin, Huan and Yang, Baosong and Xie, Pengjun and Yang, An and Liu, Dayiheng and Lin, Junyang and Huang, Fei and Zhou, Jingren},
  journal={arXiv preprint arXiv:2506.05176},
  year={2025}
}

@misc{zhang2025jaspertokencompression600mtechnicalreport,
      title={Jasper-Token-Compression-600M Technical Report}, 
      author={Dun Zhang and Ziyang Zeng and Yudong Zhou and Shuyang Lu},
      year={2025},
      eprint={2511.14405},
      archivePrefix={arXiv},
      primaryClass={cs.IR},
      url={https://arxiv.org/abs/2511.14405}, 
}

@misc{grattafiori2024llama3herdmodels,
      title={The Llama 3 Herd of Models}, 
      author={Aaron Grattafiori and Abhimanyu Dubey and Abhinav Jauhri and Abhinav Pandey and Abhishek Kadian and Ahmad Al-Dahle and Aiesha Letman and Akhil Mathur and Alan Schelten and Alex Vaughan and Amy Yang and Angela Fan and Anirudh Goyal and Anthony Hartshorn and Aobo Yang and Archi Mitra and Archie Sravankumar and Artem Korenev and Arthur Hinsvark and Arun Rao and Aston Zhang and Aurelien Rodriguez and Austen Gregerson and Ava Spataru and Baptiste Roziere and Bethany Biron and Binh Tang and Bobbie Chern and Charlotte Caucheteux and Chaya Nayak and Chloe Bi and Chris Marra and Chris McConnell and Christian Keller and Christophe Touret and Chunyang Wu and Corinne Wong and Cristian Canton Ferrer and Cyrus Nikolaidis and Damien Allonsius and Daniel Song and Danielle Pintz and Danny Livshits and Danny Wyatt and David Esiobu and Dhruv Choudhary and Dhruv Mahajan and Diego Garcia-Olano and Diego Perino and Dieuwke Hupkes and Egor Lakomkin and Ehab AlBadawy and Elina Lobanova and Emily Dinan and Eric Michael Smith and Filip Radenovic and Francisco Guzmán and Frank Zhang and Gabriel Synnaeve and Gabrielle Lee and Georgia Lewis Anderson and Govind Thattai and Graeme Nail and Gregoire Mialon and Guan Pang and Guillem Cucurell and Hailey Nguyen and Hannah Korevaar and Hu Xu and Hugo Touvron and Iliyan Zarov and Imanol Arrieta Ibarra and Isabel Kloumann and Ishan Misra and Ivan Evtimov and Jack Zhang and Jade Copet and Jaewon Lee and Jan Geffert and Jana Vranes and Jason Park and Jay Mahadeokar and Jeet Shah and Jelmer van der Linde and Jennifer Billock and Jenny Hong and Jenya Lee and Jeremy Fu and Jianfeng Chi and Jianyu Huang and Jiawen Liu and Jie Wang and Jiecao Yu and Joanna Bitton and Joe Spisak and Jongsoo Park and Joseph Rocca and Joshua Johnstun and Joshua Saxe and Junteng Jia and Kalyan Vasuden Alwala and Karthik Prasad and Kartikeya Upasani and Kate Plawiak and Ke Li and Kenneth Heafield and Kevin Stone and Khalid El-Arini and Krithika Iyer and Kshitiz Malik and Kuenley Chiu and Kunal Bhalla and Kushal Lakhotia and Lauren Rantala-Yeary and Laurens van der Maaten and Lawrence Chen and Liang Tan and Liz Jenkins and Louis Martin and Lovish Madaan and Lubo Malo and Lukas Blecher and Lukas Landzaat and Luke de Oliveira and Madeline Muzzi and Mahesh Pasupuleti and Mannat Singh and Manohar Paluri and Marcin Kardas and Maria Tsimpoukelli and Mathew Oldham and Mathieu Rita and Maya Pavlova and Melanie Kambadur and Mike Lewis and Min Si and Mitesh Kumar Singh and Mona Hassan and Naman Goyal and Narjes Torabi and Nikolay Bashlykov and Nikolay Bogoychev and Niladri Chatterji and Ning Zhang and Olivier Duchenne and Onur Çelebi and Patrick Alrassy and Pengchuan Zhang and Pengwei Li and Petar Vasic and Peter Weng and Prajjwal Bhargava and Pratik Dubal and Praveen Krishnan and Punit Singh Koura and Puxin Xu and Qing He and Qingxiao Dong and Ragavan Srinivasan and Raj Ganapathy and Ramon Calderer and Ricardo Silveira Cabral and Robert Stojnic and Roberta Raileanu and Rohan Maheswari and Rohit Girdhar and Rohit Patel and Romain Sauvestre and Ronnie Polidoro and Roshan Sumbaly and Ross Taylor and Ruan Silva and Rui Hou and Rui Wang and Saghar Hosseini and Sahana Chennabasappa and Sanjay Singh and Sean Bell and Seohyun Sonia Kim and Sergey Edunov and Shaoliang Nie and Sharan Narang and Sharath Raparthy and Sheng Shen and Shengye Wan and Shruti Bhosale and Shun Zhang and Simon Vandenhende and Soumya Batra and Spencer Whitman and Sten Sootla and Stephane Collot and Suchin Gururangan and Sydney Borodinsky and Tamar Herman and Tara Fowler and Tarek Sheasha and Thomas Georgiou and Thomas Scialom and Tobias Speckbacher and Todor Mihaylov and Tong Xiao and Ujjwal Karn and Vedanuj Goswami and Vibhor Gupta and Vignesh Ramanathan and Viktor Kerkez and Vincent Gonguet and Virginie Do and Vish Vogeti and Vítor Albiero and Vladan Petrovic and Weiwei Chu and Wenhan Xiong and Wenyin Fu and Whitney Meers and Xavier Martinet and Xiaodong Wang and Xiaofang Wang and Xiaoqing Ellen Tan and Xide Xia and Xinfeng Xie and Xuchao Jia and Xuewei Wang and Yaelle Goldschlag and Yashesh Gaur and Yasmine Babaei and Yi Wen and Yiwen Song and Yuchen Zhang and Yue Li and Yuning Mao and Zacharie Delpierre Coudert and Zheng Yan and Zhengxing Chen and Zoe Papakipos and Aaditya Singh and Aayushi Srivastava and Abha Jain and Adam Kelsey and Adam Shajnfeld and Adithya Gangidi and Adolfo Victoria and Ahuva Goldstand and Ajay Menon and Ajay Sharma and Alex Boesenberg and Alexei Baevski and Allie Feinstein and Amanda Kallet and Amit Sangani and Amos Teo and Anam Yunus and Andrei Lupu and Andres Alvarado and Andrew Caples and Andrew Gu and Andrew Ho and Andrew Poulton and Andrew Ryan and Ankit Ramchandani and Annie Dong and Annie Franco and Anuj Goyal and Aparajita Saraf and Arkabandhu Chowdhury and Ashley Gabriel and Ashwin Bharambe and Assaf Eisenman and Azadeh Yazdan and Beau James and Ben Maurer and Benjamin Leonhardi and Bernie Huang and Beth Loyd and Beto De Paola and Bhargavi Paranjape and Bing Liu and Bo Wu and Boyu Ni and Braden Hancock and Bram Wasti and Brandon Spence and Brani Stojkovic and Brian Gamido and Britt Montalvo and Carl Parker and Carly Burton and Catalina Mejia and Ce Liu and Changhan Wang and Changkyu Kim and Chao Zhou and Chester Hu and Ching-Hsiang Chu and Chris Cai and Chris Tindal and Christoph Feichtenhofer and Cynthia Gao and Damon Civin and Dana Beaty and Daniel Kreymer and Daniel Li and David Adkins and David Xu and Davide Testuggine and Delia David and Devi Parikh and Diana Liskovich and Didem Foss and Dingkang Wang and Duc Le and Dustin Holland and Edward Dowling and Eissa Jamil and Elaine Montgomery and Eleonora Presani and Emily Hahn and Emily Wood and Eric-Tuan Le and Erik Brinkman and Esteban Arcaute and Evan Dunbar and Evan Smothers and Fei Sun and Felix Kreuk and Feng Tian and Filippos Kokkinos and Firat Ozgenel and Francesco Caggioni and Frank Kanayet and Frank Seide and Gabriela Medina Florez and Gabriella Schwarz and Gada Badeer and Georgia Swee and Gil Halpern and Grant Herman and Grigory Sizov and Guangyi and Zhang and Guna Lakshminarayanan and Hakan Inan and Hamid Shojanazeri and Han Zou and Hannah Wang and Hanwen Zha and Haroun Habeeb and Harrison Rudolph and Helen Suk and Henry Aspegren and Hunter Goldman and Hongyuan Zhan and Ibrahim Damlaj and Igor Molybog and Igor Tufanov and Ilias Leontiadis and Irina-Elena Veliche and Itai Gat and Jake Weissman and James Geboski and James Kohli and Janice Lam and Japhet Asher and Jean-Baptiste Gaya and Jeff Marcus and Jeff Tang and Jennifer Chan and Jenny Zhen and Jeremy Reizenstein and Jeremy Teboul and Jessica Zhong and Jian Jin and Jingyi Yang and Joe Cummings and Jon Carvill and Jon Shepard and Jonathan McPhie and Jonathan Torres and Josh Ginsburg and Junjie Wang and Kai Wu and Kam Hou U and Karan Saxena and Kartikay Khandelwal and Katayoun Zand and Kathy Matosich and Kaushik Veeraraghavan and Kelly Michelena and Keqian Li and Kiran Jagadeesh and Kun Huang and Kunal Chawla and Kyle Huang and Lailin Chen and Lakshya Garg and Lavender A and Leandro Silva and Lee Bell and Lei Zhang and Liangpeng Guo and Licheng Yu and Liron Moshkovich and Luca Wehrstedt and Madian Khabsa and Manav Avalani and Manish Bhatt and Martynas Mankus and Matan Hasson and Matthew Lennie and Matthias Reso and Maxim Groshev and Maxim Naumov and Maya Lathi and Meghan Keneally and Miao Liu and Michael L. Seltzer and Michal Valko and Michelle Restrepo and Mihir Patel and Mik Vyatskov and Mikayel Samvelyan and Mike Clark and Mike Macey and Mike Wang and Miquel Jubert Hermoso and Mo Metanat and Mohammad Rastegari and Munish Bansal and Nandhini Santhanam and Natascha Parks and Natasha White and Navyata Bawa and Nayan Singhal and Nick Egebo and Nicolas Usunier and Nikhil Mehta and Nikolay Pavlovich Laptev and Ning Dong and Norman Cheng and Oleg Chernoguz and Olivia Hart and Omkar Salpekar and Ozlem Kalinli and Parkin Kent and Parth Parekh and Paul Saab and Pavan Balaji and Pedro Rittner and Philip Bontrager and Pierre Roux and Piotr Dollar and Polina Zvyagina and Prashant Ratanchandani and Pritish Yuvraj and Qian Liang and Rachad Alao and Rachel Rodriguez and Rafi Ayub and Raghotham Murthy and Raghu Nayani and Rahul Mitra and Rangaprabhu Parthasarathy and Raymond Li and Rebekkah Hogan and Robin Battey and Rocky Wang and Russ Howes and Ruty Rinott and Sachin Mehta and Sachin Siby and Sai Jayesh Bondu and Samyak Datta and Sara Chugh and Sara Hunt and Sargun Dhillon and Sasha Sidorov and Satadru Pan and Saurabh Mahajan and Saurabh Verma and Seiji Yamamoto and Sharadh Ramaswamy and Shaun Lindsay and Shaun Lindsay and Sheng Feng and Shenghao Lin and Shengxin Cindy Zha and Shishir Patil and Shiva Shankar and Shuqiang Zhang and Shuqiang Zhang and Sinong Wang and Sneha Agarwal and Soji Sajuyigbe and Soumith Chintala and Stephanie Max and Stephen Chen and Steve Kehoe and Steve Satterfield and Sudarshan Govindaprasad and Sumit Gupta and Summer Deng and Sungmin Cho and Sunny Virk and Suraj Subramanian and Sy Choudhury and Sydney Goldman and Tal Remez and Tamar Glaser and Tamara Best and Thilo Koehler and Thomas Robinson and Tianhe Li and Tianjun Zhang and Tim Matthews and Timothy Chou and Tzook Shaked and Varun Vontimitta and Victoria Ajayi and Victoria Montanez and Vijai Mohan and Vinay Satish Kumar and Vishal Mangla and Vlad Ionescu and Vlad Poenaru and Vlad Tiberiu Mihailescu and Vladimir Ivanov and Wei Li and Wenchen Wang and Wenwen Jiang and Wes Bouaziz and Will Constable and Xiaocheng Tang and Xiaojian Wu and Xiaolan Wang and Xilun Wu and Xinbo Gao and Yaniv Kleinman and Yanjun Chen and Ye Hu and Ye Jia and Ye Qi and Yenda Li and Yilin Zhang and Ying Zhang and Yossi Adi and Youngjin Nam and Yu and Wang and Yu Zhao and Yuchen Hao and Yundi Qian and Yunlu Li and Yuzi He and Zach Rait and Zachary DeVito and Zef Rosnbrick and Zhaoduo Wen and Zhenyu Yang and Zhiwei Zhao and Zhiyu Ma},
      year={2024},
      eprint={2407.21783},
      archivePrefix={arXiv},
      primaryClass={cs.AI},
      url={https://arxiv.org/abs/2407.21783}, 
}

@ARTICLE{726791,
  author={Lecun, Y. and Bottou, L. and Bengio, Y. and Haffner, P.},
  journal={Proceedings of the IEEE}, 
  title={Gradient-based learning applied to document recognition}, 
  year={1998},
  volume={86},
  number={11},
  pages={2278-2324},
  doi={10.1109/5.726791}}

@INPROCEEDINGS{5206848,
  author={Deng, Jia and Dong, Wei and Socher, Richard and Li, Li-Jia and Kai Li and Li Fei-Fei},
  booktitle={2009 IEEE Conference on Computer Vision and Pattern Recognition}, 
  title={ImageNet: A large-scale hierarchical image database}, 
  year={2009},
  volume={},
  number={},
  pages={248-255},
  doi={10.1109/CVPR.2009.5206848}}

@misc{grinsztajn2023vectorizingstringentriesdata,
      title={Vectorizing string entries for data processing on tables: when are larger language models better?}, 
      author={Léo Grinsztajn and Edouard Oyallon and Myung Jun Kim and Gaël Varoquaux},
      year={2023},
      eprint={2312.09634},
      archivePrefix={arXiv},
      primaryClass={stat.ML},
      url={https://arxiv.org/abs/2312.09634}, 
}

@inproceedings{arazi_tabstar_2025,
    title = {{TabSTAR}: {A} {Tabular} {Foundation} {Model} for {Tabular} {Data} with {Text} {Fields}},
    volume = {38},
    url = {https://proceedings.neurips.cc/paper_files/paper/2025/file/faf6e23e198314c7728eaa6ac44ae079-Paper-Conference.pdf},
    booktitle = {Advances in {Neural} {Information} {Processing} {Systems}},
    publisher = {Curran Associates, Inc.},
    author = {Arazi, Alan and Shapira, Eilam and Reichart, Roi},
    editor = {Belgrave, D. and Zhang, C. and Lin, H. and Pascanu, R. and Koniusz, P. and Ghassemi, M. and Chen, N.},
    year = {2025},
    pages = {172108--172161},
}

@misc{spinaci2025contexttabsemanticsawaretabularincontext,
      title={ConTextTab: A Semantics-Aware Tabular In-Context Learner}, 
      author={Marco Spinaci and Marek Polewczyk and Maximilian Schambach and Sam Thelin},
      year={2025},
      eprint={2506.10707},
      archivePrefix={arXiv},
      primaryClass={cs.LG},
      url={https://arxiv.org/abs/2506.10707}, 
}

@inproceedings{Chen_2016, series={KDD ’16},
   title={XGBoost: A Scalable Tree Boosting System},
   DOI={10.1145/2939672.2939785},
   booktitle={Proceedings of the 22nd ACM SIGKDD International Conference on Knowledge Discovery and Data Mining},
   publisher={ACM},
   author={Chen, Tianqi and Guestrin, Carlos},
   year={2016},
   month=aug, pages={785–794},
   collection={KDD ’16} 
}

@misc{grinsztajn2025tabpfn25advancingstateart,
      title={TabPFN-2.5: Advancing the State of the Art in Tabular Foundation Models}, 
      author={Léo Grinsztajn and Klemens Flöge and Oscar Key and Felix Birkel and Philipp Jund and Brendan Roof and Benjamin Jäger and Dominik Safaric and Simone Alessi and Adrian Hayler and Mihir Manium and Rosen Yu and Felix Jablonski and Shi Bin Hoo and Anurag Garg and Jake Robertson and Magnus Bühler and Vladyslav Moroshan and Lennart Purucker and Clara Cornu and Lilly Charlotte Wehrhahn and Alessandro Bonetto and Bernhard Schölkopf and Sauraj Gambhir and Noah Hollmann and Frank Hutter},
      year={2025},
      eprint={2511.08667},
      archivePrefix={arXiv},
      primaryClass={cs.LG},
      url={https://arxiv.org/abs/2511.08667}, 
}

@article{bischl2025openml,
  title={OpenML: Insights from 10 years and more than a thousand papers},
  author={Bischl, Bernd and Casalicchio, Giuseppe and Das, Taniya and Feurer, Matthias and Fischer, Sebastian and Gijsbers, Pieter and Mukherjee, Subhaditya and M{\"u}ller, Andreas C and N{\'e}meth, L{\'a}szl{\'o} and Oala, Luis and others},
  journal={Patterns},
  year={2025},
  publisher={Elsevier}
}

@article{vanschoren2014openml,
  title={OpenML: networked science in machine learning},
  author={Vanschoren, Joaquin and Van Rijn, Jan N and Bischl, Bernd and Torgo, Luis},
  journal={ACM SIGKDD Explorations Newsletter},
  volume={15},
  number={2},
  pages={49--60},
  year={2014},
  publisher={ACM New York, NY, USA}
}

@inproceedings{muennighoff2023mtebmassivetextembedding,
    title = "{MTEB}: Massive Text Embedding Benchmark",
    author = "Muennighoff, Niklas  and
      Tazi, Nouamane  and
      Magne, Loic  and
      Reimers, Nils",
    booktitle = "Proceedings of the 17th Conference of the European Chapter of the Association for Computational Linguistics",
    month = may,
    year = "2023",
    address = "Dubrovnik, Croatia",
    publisher = "Association for Computational Linguistics",
    pages = "2014--2037",
}

@inproceedings{mikolov2017advancespretrainingdistributedword,
    title = "Advances in Pre-Training Distributed Word Representations",
    author = "Mikolov, Tomas  and
      Grave, Edouard  and
      Bojanowski, Piotr  and
      Puhrsch, Christian  and
      Joulin, Armand",
    booktitle = "Proceedings of the Eleventh International Conference on Language Resources and Evaluation ({LREC} 2018)",
    month = may,
    year = "2018",
    address = "Miyazaki, Japan",
    publisher = "European Language Resources Association (ELRA)",
}

@article{conover1979multiple,
  title={On multiple comparisons procedures},
  author={Conover, William J and Iman, Ronald L},
  journal={Technical Report LA-7677-MS, Los Alamos Scientific Laboratory},
  year={1979}
}

@article{10.1214/aoms/1177731944,
author = {Milton Friedman},
title = {{A Comparison of Alternative Tests of Significance for the Problem of $m$ Rankings}},
volume = {11},
journal = {The Annals of Mathematical Statistics},
number = {1},
publisher = {Institute of Mathematical Statistics},
pages = {86 -- 92},
year = {1940},
doi = {10.1214/aoms/1177731944},
}

@article{enevoldsen2025mmtebmassivemultilingualtext,
  author = {Kenneth Enevoldsen and Isaac Chung and Imene Kerboua and Márton Kardos and Ashwin Mathur and David Stap and Jay Gala and Wissam Siblini and Dominik Krzemiński and Genta Indra Winata and Saba Sturua and Saiteja Utpala and Mathieu Ciancone and Marion Schaeffer and Gabriel Sequeira and Diganta Misra and Shreeya Dhakal and Jonathan Rystrøm and Roman Solomatin and Ömer Çağatan and Akash Kundu and Martin Bernstorff and Shitao Xiao and Akshita Sukhlecha and Bhavish Pahwa and Rafał Poświata and Kranthi Kiran GV and Shawon Ashraf and Daniel Auras and Björn Plüster and Jan Philipp Harries and Loïc Magne and Isabelle Mohr and Mariya Hendriksen and Dawei Zhu and Hippolyte Gisserot-Boukhlef and Tom Aarsen and Jan Kostkan and Konrad Wojtasik and Taemin Lee and Marek Šuppa and Crystina Zhang and Roberta Rocca and Mohammed Hamdy and Andrianos Michail and John Yang and Manuel Faysse and Aleksei Vatolin and Nandan Thakur and Manan Dey and Dipam Vasani and Pranjal Chitale and Simone Tedeschi and Nguyen Tai and Artem Snegirev and Michael Günther and Mengzhou Xia and Weijia Shi and Xing Han Lù and Jordan Clive and Gayatri Krishnakumar and Anna Maksimova and Silvan Wehrli and Maria Tikhonova and Henil Panchal and Aleksandr Abramov and Malte Ostendorff and Zheng Liu and Simon Clematide and Lester James Miranda and Alena Fenogenova and Guangyu Song and Ruqiya Bin Safi and Wen-Ding Li and Alessia Borghini and Federico Cassano and Hongjin Su and Jimmy Lin and Howard Yen and Lasse Hansen and Sara Hooker and Chenghao Xiao and Vaibhav Adlakha and Orion Weller and Siva Reddy and Niklas Muennighoff},
  doi = {10.48550/arXiv.2502.13595},
  journal = {arXiv preprint arXiv:2502.13595},
  publisher = {arXiv},
  title = {MMTEB: Massive Multilingual Text Embedding Benchmark},
  url = {https://arxiv.org/abs/2502.13595},
  year = {2025},
}

@misc{kusupati2024matryoshkarepresentationlearning,
      title={Matryoshka Representation Learning}, 
      author={Aditya Kusupati and Gantavya Bhatt and Aniket Rege and Matthew Wallingford and Aditya Sinha and Vivek Ramanujan and William Howard-Snyder and Kaifeng Chen and Sham Kakade and Prateek Jain and Ali Farhadi},
      year={2024},
      eprint={2205.13147},
      archivePrefix={arXiv},
      primaryClass={cs.LG},
      url={https://arxiv.org/abs/2205.13147}, 
}

@misc{qu2026tabiclv2betterfasterscalable,
      title={TabICLv2: A better, faster, scalable, and open tabular foundation model}, 
      author={Jingang Qu and David Holzmüller and Gaël Varoquaux and Marine Le Morvan},
      year={2026},
      eprint={2602.11139},
      archivePrefix={arXiv},
      primaryClass={cs.LG},
      url={https://arxiv.org/abs/2602.11139}, 
}

@book{vershynin2018,
  author    = {Vershynin, Roman},
  title     = {High-Dimensional Probability},
  publisher = {Cambridge University Press},
  year      = {2018}
}

@article{a92f3c16-7c6e-31d3-b403-82d2b0a469e4,
 ISSN = {00401706},
 URL = {http://www.jstor.org/stable/1267351},
 author = {Arthur E. Hoerl and Robert W. Kennard},
 journal = {Technometrics},
 number = {1},
 pages = {55--67},
 publisher = {[Taylor & Francis, Ltd., American Statistical Association, American Society for Quality]},
 title = {Ridge Regression: Biased Estimation for Nonorthogonal Problems},
 urldate = {2026-03-20},
 volume = {12},
 year = {1970}
}

@misc{shi2021benchmarkingmultimodalautomltabular,
      title={Benchmarking Multimodal AutoML for Tabular Data with Text Fields}, 
      author={Xingjian Shi and Jonas Mueller and Nick Erickson and Mu Li and Alexander J. Smola},
      year={2021},
      eprint={2111.02705},
      archivePrefix={arXiv},
      primaryClass={cs.LG},
      url={https://arxiv.org/abs/2111.02705}, 
}

@misc{thielmann2025mambularsequentialmodeltabular,
      title={Mambular: A Sequential Model for Tabular Deep Learning}, 
      author={Anton Frederik Thielmann and Manish Kumar and Christoph Weisser and Arik Reuter and Benjamin Säfken and Soheila Samiee},
      year={2025},
      eprint={2408.06291},
      archivePrefix={arXiv},
      primaryClass={cs.LG},
      url={https://arxiv.org/abs/2408.06291}, 
}

@misc{gorla2026illusiongeneralizationtabularlanguage,
      title={The Illusion of Generalization in Tabular Language Models}, 
      author={Aditya Gorla and Ratish Puduppully},
      year={2026},
      eprint={2602.04031},
      archivePrefix={arXiv},
      primaryClass={cs.LG},
      url={https://arxiv.org/abs/2602.04031}, 
}

@misc{gao2019representationdegenerationproblemtraining,
      title={Representation Degeneration Problem in Training Natural Language Generation Models}, 
      author={Jun Gao and Di He and Xu Tan and Tao Qin and Liwei Wang and Tie-Yan Liu},
      year={2019},
      eprint={1907.12009},
      archivePrefix={arXiv},
      primaryClass={cs.CL},
      url={https://arxiv.org/abs/1907.12009}, 
}

@misc{ethayarajh2019contextualcontextualizedwordrepresentations,
      title={How Contextual are Contextualized Word Representations? Comparing the Geometry of BERT, ELMo, and GPT-2 Embeddings}, 
      author={Kawin Ethayarajh},
      year={2019},
      eprint={1909.00512},
      archivePrefix={arXiv},
      primaryClass={cs.CL},
      url={https://arxiv.org/abs/1909.00512}, 
}

@misc{vogel2026universaltabularembeddingsbenchmark,
      title={Towards Universal Tabular Embeddings: A Benchmark Across Data Tasks}, 
      author={Liane Vogel and Kavitha Srinivas and Niharika D'Souza and Sola Shirai and Oktie Hassanzadeh and Horst Samulowitz},
      year={2026},
      eprint={2604.21696},
      archivePrefix={arXiv},
      primaryClass={cs.LG},
      url={https://arxiv.org/abs/2604.21696}, 
}

@article{10.1098/rsta.2015.0202,
    author = {Jolliffe, Ian T. and Cadima, Jorge},
    title = {Principal component analysis: a review and recent developments},
    journal = {Philosophical Transactions of the Royal Society A: Mathematical, Physical and Engineering Sciences},
    volume = {374},
    number = {2065},
    pages = {20150202},
    year = {2016},
    month = {04},
    issn = {1364-503X},
    doi = {10.1098/rsta.2015.0202},
    url = {https://doi.org/10.1098/rsta.2015.0202},
    eprint = {https://royalsocietypublishing.org/rsta/article-pdf/doi/10.1098/rsta.2015.0202/1381479/rsta.2015.0202.pdf},
}

@misc{sklearn_feature_scaling,
  title={Importance of feature scaling},
  author={{scikit-learn developers}},
  howpublished={\url{https://scikit-learn.org/stable/auto_examples/preprocessing/plot_scaling_importance.html}},
  note={scikit-learn documentation, accessed April 2026},
  year={2026}
}

@book{Kuhn_13,
  author = {Kuhn, Max and Johnson, Kjell},
  data = {NA},
  isbn = {978-1-4614-6848-6},
  online = {book},
  publisher = {Springer},
  title = {Applied Predictive Modeling},
  year = 2013
}

@ARTICLE{RePEc:tpr:restat:v:61:y:1979:i:1:p:146-49,
title = {A Formula for the Gini Coefficient},
author = {Dorfman, Robert},
year = {1979},
journal = {The Review of Economics and Statistics},
volume = {61},
number = {1},
pages = {146-49},
url = {https://EconPapers.repec.org/RePEc:tpr:restat:v:61:y:1979:i:1:p:146-49}
}

@article{c4091bd3-d888-3152-8886-c284bf66a93a,
 ISSN = {00994987},
 URL = {http://www.jstor.org/stable/3001968},
 author = {Frank Wilcoxon},
 journal = {Biometrics Bulletin},
 number = {6},
 pages = {80--83},
 publisher = {[International Biometric Society, Wiley]},
 title = {Individual Comparisons by Ranking Methods},
 urldate = {2026-05-05},
 volume = {1},
 year = {1945}
}

@Inbook{Haynes2013,
author="Haynes, Winston",
editor="Dubitzky, Werner
and Wolkenhauer, Olaf
and Cho, Kwang-Hyun
and Yokota, Hiroki",
title="Holm's Method",
bookTitle="Encyclopedia of Systems Biology",
year="2013",
publisher="Springer New York",
address="New York, NY",
pages="902--902",
isbn="978-1-4419-9863-7",
doi="10.1007/978-1-4419-9863-7_1214",
url="https://doi.org/10.1007/978-1-4419-9863-7_1214"
}

\newpage
\begin{appendices}
\listofappendices
\counterwithin{figure}{section}
\counterwithin{table}{section}
\counterwithin{equation}{section}
\crefalias{section}{appendix}
\crefalias{subsection}{appendix}

\newpage
\section{Detailed theoretical analysis}\label{app:theory}
\subsection{Problem setting and assumptions}

We consider the problem of evaluating and ranking $k$ machine learning models, indexed by $i \in \{1, \dots, k\}$, over a population of tabular datasets.

\begin{assumption}[Dataset distribution]\label{assum:dataset_distribution}
    There exists an unknown distribution $\mathcal{D}$ of tabular datasets. A benchmark $B = \{d_1, \dots, d_N\}$ consists of $N$ datasets sampled independently and identically distributed (i.i.d.)~from $\mathcal{D}$.
\end{assumption}

\begin{assumption}[Performance decomposition]\label{assum:perf_model}
    Let $X_{i,d}$ denote the observed performance metric (e.g., accuracy or $R^2$ score) of model $i$ on dataset $d$. We assume this performance can be decomposed into mean performance and deviation:
    \begin{align}\label{eq:model_perf}
        X_{i,d} = \mu_i + \epsilon_{i,d},
    \end{align}
    where $\mu_i \defeq \mathbb{E}_{d \sim \mathcal{D}}[X_{i,d}]$ represents the intrinsic expected performance of model $i$ over the population of datasets.
\end{assumption}

To facilitate tractable theoretical analysis, we introduce the following assumption regarding the performance deviations.

\begin{assumption}[Homoskedastic Gaussian noise]\label{assum:gaussian_noise}
    The deviations $\epsilon_{i,d}$ are independent random variables following a Gaussian distribution with mean zero and constant variance $\sigma^2$ across all models and datasets, \emph{i.e.}, $\epsilon_{i,d} \sim \mathcal{N}(0, \sigma^2)$.
\end{assumption}

\begin{remark}[Limitations]
    The formulation in \autoref{eq:model_perf} simplifies the problem by treating performance variations as stochastic noise around a global mean $\mu_i$. This basic model has two main limitations: (1) it implies that deviations are purely specific to the instance $(i,d)$; and (2) it ignores structural factors such as dataset difficulty or specific model-dataset affinities (inductive biases). While restrictive, this allows us to derive closed-form analytical bounds on ranking reliability. We relax \cref{assum:perf_model} to incorporate explicit inductive biases in \autoref{sec:inductive_bias}.
\end{remark}

\subsection{Preliminary definitions and results}

Let $B$ be a benchmark of size $N$. We define the observed average performance of model $i$ on benchmark $B$ as:
\begin{align}
    \bar{X}_i^{(B)} \defeq \frac{1}{N} \sum_{d\in B} X_{i,d}.
\end{align}

We define the observed performance gap between two models $i$ and $j$ on benchmark $B$ as
\begin{align}\label{eq:perf_gap_def}
    Y_{ij}^{(B)} \defeq \bar{X}_i^{(B)} - \bar{X}_j^{(B)}.
\end{align}
Using \autoref{eq:model_perf}, it can be decomposed as:
\begin{align}\label{eq:perf_gap}
    Y_{ij}^{(B)} = \frac{1}{N} \sum_{d\in B} (X_{i,d} - X_{j,d}) = \Delta_{ij} + \bar\eta_{ij}^{(B)},
\end{align}
where $\Delta_{ij} \defeq \mu_i - \mu_j$ denotes the expected performance gap, and $\bar\eta_{ij}^{(B)} \defeq \frac{1}{N} \sum_{d\in B} (\epsilon_{i,d} - \epsilon_{j,d})$.

\begin{lemma}[Expected sign of the performance gap]\label{lem:expectation_sign}
For any pair of models $(i,j)$, the expected sign of the performance gap on a benchmark $B$ of size $N$ is given by
\begin{align}
    \mathbb{E}\left[\mathrm{sgn}\left(Y_{ij}^{(B)}\right)\right] = \mathrm{sgn}(\Delta_{ij}) \, \mathrm{erf}\left( \frac{|\Delta_{ij}|\sqrt{N}}{2\sigma} \right),
\end{align}
where $\mathrm{erf}$ is the error function, defined for all $x\ge0$ as $\mathrm{erf}(x) = \frac{2}{\sqrt{\pi}} \int_0^x e^{-t^2} dt$.
\end{lemma}

\begin{proof}
Let $p_{ij}(N) \defeq \mathbb{P}\left[ \mathrm{sgn}(Y_{ij}^{(B)}) = \mathrm{sgn}(\Delta_{ij}) \right]$ denote the probability that the observed ranking of models $i$ and $j$ on benchmark $B$ matches the oracle ranking. The term $\mathrm{sgn}\left(Y_{ij}^{(B)}\right) \mathrm{sgn}(\Delta_{ij})$ takes the value $1$ with probability $p_{ij}(N)$ (when the signs match) and $-1$ with probability $1 - p_{ij}(N)$ (when they disagree). Therefore, its expected value is:
\begin{align}
    \mathbb{E}\left[ \mathrm{sgn}\left(Y_{ij}^{(B)}\right) \mathrm{sgn}(\Delta_{ij}) \right] &= 1 \cdot p_{ij}(N) + (-1) \cdot (1 - p_{ij}(N)) \\
    &= 2 p_{ij}(N) - 1.\label{eq:expectation_sign}
\end{align}
Without loss of generality, let us assume that the true gap is positive, i.e., $\Delta_{ij} > 0$. In this case, the probability that the ranking is correct corresponds to the probability that the observed sample mean difference is positive:
\begin{align}
    p_{ij}(N) = \mathbb{P}\left[ Y_{ij}^{(B)} > 0 \right].
\end{align}
Under the homoskedastic independent Gaussian noise assumption (\cref{assum:gaussian_noise}), we have $\bar\eta_{ij}^{(B)}\sim\mathcal{N}\left(0,\frac{2\sigma^2}{N}\right)$, and consequently
\begin{align}\label{eq:perf_gap_law}
    Y_{ij}^{(B)} \sim \mathcal{N}\left(\Delta_{ij},\frac{2\sigma^2}{N}\right).
\end{align}
From \autoref{eq:perf_gap_law}, we define the standardized variable $Z$ as:
\begin{align}
    Z \defeq \frac{Y_{ij}^{(B)} - \Delta_{ij}}{\sqrt{\frac{2\sigma^2}{N}}} \sim \mathcal{N}(0, 1).
\end{align}
The probability of correct ranking becomes:
\begin{align}
    p_{ij}(N) &= \mathbb{P}\left[ Y_{ij}^{(B)} > 0 \right] \\
    &= \mathbb{P}\left[ Z > -\frac{\Delta_{ij}}{\sqrt{\frac{2\sigma^2}{N}}} \right] \\
    &= \Phi\left( \frac{\Delta_{ij}\sqrt{N}}{\sigma\sqrt{2}} \right),
\end{align}
where $\Phi$ is the cumulative distribution function of the standard normal distribution. By symmetry, the result holds for $\Delta_{ij}<0$ using absolute values:
\begin{align}
    p_{ij}(N) = \Phi\left( \frac{|\Delta_{ij}|\sqrt{N}}{\sigma\sqrt{2}} \right).
\end{align}
The error function satisfies the identity $\mathrm{erf}(x) = 2\Phi(x\sqrt{2}) - 1$.
Using this relation, we can simplify the expectation term $2 p_{ij}(N) - 1$ to get
\begin{align}
    2 p_{ij}(N) - 1 = \mathrm{erf}\left( \frac{|\Delta_{ij}|\sqrt{N}}{2\sigma} \right).\label{eq:pij_erf}
\end{align}
Finally, from \autoref{eq:expectation_sign}, we know that $\mathbb{E}\left[\mathrm{sgn}\left(Y_{ij}^{(B)}\right)\right] = \mathrm{sgn}(\Delta_{ij}) (2p_{ij}(N) - 1)$, which completes the proof.
\end{proof}

\subsection{Agreement between two independent benchmarks}\label{sec:2_bench}
We investigate how the size of a benchmark affects its consistency by asking the following question: given two independent benchmarks $B_1$ and $B_2$, both of size $N$, how well do their model rankings agree? Specifically, we determine how large $N$ must be to achieve a certain level of agreement. To do so, we derive the expected Kendall-$\tau$ correlation $\mathbb{E}[\tau_{N,N}]$ between the rankings produced by two independent benchmarks of size $N$ as a function of $N$.

Let $Y_{ij}^{(B_1)}$ and $Y_{ij}^{(B_2)}$ denote the observed performance gaps between models $i$ and $j$ on benchmarks $B_1$ and $B_2$, respectively.

The Kendall-$\tau$ correlation between the rankings on $B_1$ and $B_2$ is defined as:
\begin{align}
    \tau_{N,N} \defeq \frac{2}{k(k-1)} \sum_{i<j} \mathrm{sgn}\left(Y_{ij}^{(B_1)}\right) \mathrm{sgn}\left(Y_{ij}^{(B_2)}\right).
\end{align}

\begin{proposition}[Expected ranking consistency] \label{prop:kendall_tau_2_bench}
Let $\tau_{N,N}$ denote the Kendall-$\tau$ rank correlation between the rankings of $k$ models on two independent benchmarks of size $N$. The expected correlation is given by
    \begin{align}\label{eq:kendall_tau_2_bench}
        \mathbb{E}[\tau_{N,N}] = \frac{2}{k(k-1)} \sum_{1 \le i < j \le k} \mathrm{erf}^2\left( \frac{|\Delta_{ij}|\sqrt{N}}{2\sigma} \right).
    \end{align}
\end{proposition}
\begin{proof}
Since the benchmarks are independent samples of size $N$ drawn from $\mathcal{D}$, the gaps $Y_{ij}^{(B_1)}$ and $Y_{ij}^{(B_2)}$ are independent and identically distributed (\cref{assum:dataset_distribution}).
Therefore, the expected Kendall-$\tau$ is
\begin{align}
    \mathbb{E}[\tau_{N,N}] = \frac{2}{k(k-1)} \sum_{i<j} \mathbb{E}\left[ \mathrm{sgn}\left(Y_{ij}^{(B_1)}\right) \right] \mathbb{E}\left[ \mathrm{sgn}\left(Y_{ij}^{(B_2)}\right) \right],
\end{align}
and, from \cref{lem:expectation_sign}, we have
\begin{align}
    \mathbb{E}\left[ \mathrm{sgn}\left(Y_{ij}^{(B_1)}\right) \right] \mathbb{E}\left[ \mathrm{sgn}\left(Y_{ij}^{(B_2)}\right) \right] &= \left[ \mathrm{sgn}(\Delta_{ij}) \, \mathrm{erf}\left( \frac{|\Delta_{ij}|\sqrt{N}}{2\sigma} \right) \right]^2, \\
    &= \mathrm{erf}^2\left( \frac{|\Delta_{ij}|\sqrt{N}}{2\sigma} \right),
\end{align}
which completes the proof.
\end{proof}

\autoref{eq:kendall_tau_2_bench} shows that $\tau_{N,N}$ converges to $1$ as $N$ increases, meaning that ultimately the two benchmarks agree on the ranking. The speed of this convergence is controlled by two parameters: the more separable the models are (large values of $|\Delta_{ij}|$), and the smaller the evaluation noise (small $\sigma$), the faster the convergence.

\paragraph{Asymptotic analysis} We are interested in the behaviour of the expected Kendall-$\tau$ when $N$ becomes large. Indeed, this can help quantify the number of datasets needed to reach a certain degree of agreement between two similarly-sized benchmarks.

\begin{corollary}[Asymptotic agreement rate] \label{cor:asymp_2_bench}
    Consider the expected disagreement $1 - \mathbb{E}[\tau_{N,N}]$ between two independent benchmarks of size $N$. As $N \to \infty$, this disagreement decays exponentially according to the smallest pairwise performance gap:
    \begin{align}\label{eq:asymp_kendall_tau_2_bench}
        1 - \mathbb{E}[\tau_{N,N}] \sim \frac{C}{\sqrt{N}} \exp\left(-\frac{\Delta_{\min}^2 N}{4\sigma^2}\right),
    \end{align}
    where $\Delta_{\min} \defeq \min_{i<j} |\Delta_{ij}|$ is the minimum expected performance gap, and $C = \frac{8\sigma M_{\mathrm{min}}}{k(k-1)\Delta_{\min}\sqrt{\pi}}$ is a constant determined by the number of pairs $M_{\mathrm{min}}$ achieving this minimum gap.
\end{corollary}

\begin{proof}
For large real $x$, the error function has the following asymptotic expansion
\begin{align}
    \mathrm{erf}(x) = 1 - \frac{e^{-x^2}}{x\sqrt{\pi}}\left(1 + \sum_{n=1}^\infty (-1)^n \frac{(2n - 1)!!}{\left(2x^2\right)^n}\right)
\end{align}
where $(2n - 1)!!$ is the double factorial of $(2n-1)$, which is the product of all odd numbers up to $(2n-1)$.

Therefore, for large $x$,
\begin{align}\label{eq:equiv_erf}
    \mathrm{erf}(x) = 1 - \frac{e^{-x^2}}{x\sqrt{\pi}} + O\left(\frac{e^{-x^2}}{x^3}\right),
\end{align}
and thus,
\begin{align}
    \mathrm{erf}^2(x) = 1 - \frac{2e^{-x^2}}{x\sqrt{\pi}} + O\left(\frac{e^{-x^2}}{x^3}\right).
\end{align}

Plugging this into \autoref{eq:kendall_tau_2_bench}, with $x=\frac{|\Delta_{ij}|\sqrt{N}}{2\sigma}$, we get that for large $N$
\begin{align}
    \mathbb{E}[\tau_{N,N}] = 1 - \frac{2}{k(k-1)} \sum_{i<j} \left[ \frac{4\sigma \exp\left(-\frac{\Delta_{ij}^2N}{4\sigma^2}\right)}{|\Delta_{ij}|\sqrt{\pi N}} + O\left(\frac{\exp\left(-\frac{\Delta_{ij}^2N}{4\sigma^2}\right)}{N^{3/2}}\right) \right].
\end{align}

Let $\Delta_{\min} = \min_{i<j} |\Delta_{ij}|$. The summation is dominated by the terms corresponding to the smallest performance gap, as the exponential decay is slowest for $\Delta_{\min}$. Therefore, we have for large $N$:
\begin{align}
    1 - \mathbb{E}[\tau_{N,N}] \sim \frac{C}{\sqrt{N}} \exp\left(-\frac{\Delta_{\min}^2 N}{4\sigma^2}\right)
\end{align}
with $C = \frac{8\sigma M_{\mathrm{min}}}{k(k-1)\Delta_{\min}\sqrt{\pi}}$, and $M_{\mathrm{min}}$ the number of pairs achieving the minimum gap.
\end{proof}

\cref{cor:asymp_2_bench} shows that asymptotically the disagreement between benchmarks decays exponentially with $N$, and is only controlled by the hardest-to-distinguish pair of models.

\subsection{Convergence to the oracle ranking}\label{sec:1_bench}
We now address a complementary question: what benchmark size $N$ is required for the empirical ranking of models on $B=\{d_1,\ldots,d_N\}$ to reliably converge to the true oracle ranking over the population $\mathcal{D}$? To quantify this, we examine the expected Kendall-$\tau$ correlation, denoted $\mathbb{E}[\tau_{N,\infty}]$, between the empirical ranking observed on a sample of size $N$ and the ground-truth ranking determined by the population expectations $\{\mu_i\}_{i=1}^k$.

The Kendall-$\tau$ correlation between the observed ranking on $B$ and the oracle ranking is defined as:
\begin{align}
    \tau_{N,\infty} \defeq \frac{2}{k(k-1)} \sum_{i<j} \mathrm{sgn}\left(Y_{ij}^{(B)}\right) \mathrm{sgn}\left(\Delta_{ij}\right).
\end{align}

\begin{proposition}[Expected convergence rate to oracle] \label{prop:kendall_tau_1_bench}
Let $\tau_{N,\infty}$ be the Kendall-$\tau$ rank correlation between the ranking of $k$ models observed on a benchmark of size $N$ and the oracle ranking. The expected correlation is given by:
    \begin{align}\label{eq:kendall_tau_1_bench}
        \mathbb{E}[\tau_{N,\infty}] = \frac{2}{k(k-1)} \sum_{i<j} \mathrm{erf}\left( \frac{|\Delta_{ij}|\sqrt{N}}{2\sigma} \right).
    \end{align}
\end{proposition}
\begin{proof}
By the linearity of expectation, we have:
\begin{align}
    \mathbb{E}[\tau_{N,\infty}] = \frac{2}{k(k-1)} \sum_{i<j} \mathbb{E}\left[ \mathrm{sgn}\left(Y_{ij}^{(B)}\right) \mathrm{sgn}(\Delta_{ij}) \right].
\end{align}
Combining \cref{lem:expectation_sign} with this equation completes the proof.
\end{proof}

\begin{corollary}[Asymptotic convergence rate to the oracle] \label{cor:asymp_1_bench}
    Let $\Delta_{\min} = \min_{i<j} |\Delta_{ij}|$ be the minimum expected performance gap. As $N \to \infty$, the expected disagreement $1 - \mathbb{E}[\tau_{N,\infty}]$ between the ranking on a benchmark of size $N$ and the oracle ranking decays exponentially as:
    \begin{align}\label{eq:asymp_kendall_tau_1_bench}
        1 - \mathbb{E}[\tau_{N,\infty}] \sim \frac{C}{2\sqrt{N}} \exp\left(-\frac{\Delta_{\min}^2 N}{4\sigma^2}\right),
    \end{align}
    where $C = \frac{8\sigma M_{\mathrm{min}}}{k(k-1)\Delta_{\min}\sqrt{\pi}}$ is the constant defined in \cref{cor:asymp_2_bench}, with $M_{\mathrm{min}}$ representing the number of model pairs separated by exactly $\Delta_{\min}$.
\end{corollary}
\begin{proof}
    Substituting the asymptotic expansion of the error function from \autoref{eq:equiv_erf} into \autoref{eq:kendall_tau_1_bench}, we obtain:
    \begin{align}
        \mathbb{E}[\tau_{N,\infty}] = 1 - \frac{2}{k(k-1)} \sum_{i<j} \left[ \frac{2\sigma\exp{\left(-\frac{|\Delta_{ij}|^2 N}{4\sigma^2}\right)}}{|\Delta_{ij}|\sqrt{\pi N}} + O\left(\frac{\exp{\left(-\frac{|\Delta_{ij}|^2 N}{4\sigma^2}\right)}}{N^{3/2}}\right) \right].
    \end{align}
    As $N \to \infty$, the summation is dominated by the terms corresponding to the minimal gap $\Delta_{\min}$, as the exponential decay is slowest for these terms. Retaining only the leading order terms yields:
    \begin{align}
        1 - \mathbb{E}[\tau_{N,\infty}] \sim \frac{1}{2} \left[ \frac{8\sigma M_{\mathrm{min}}}{k(k-1)\Delta_{\min}\sqrt{\pi}} \right] \frac{1}{\sqrt{N}} \exp\left(-\frac{\Delta_{\min}^2 N}{4\sigma^2}\right).
    \end{align}
    Identifying the bracketed term as the constant $C$ from \cref{cor:asymp_2_bench}, we recover the result.
\end{proof}
Notably, \cref{cor:asymp_1_bench} implies that the expected disagreement with the oracle is asymptotically half the disagreement between two independent benchmarks of size $N$.

\paragraph{Finite-$N$ correction.}\label{sec:finite_n_correction}
The relation $\bbE[1-\tau_{N,N}] \approx 2\,\bbE[1-\tau_{N,\infty}]$ is exact as $N \to \infty$; at finite $N$, we derive the exact factor relating the two quantities. Using the identity $1 - \mathrm{erf}^2(x_{ij}) = (1-\mathrm{erf}(x_{ij}))(1+\mathrm{erf}(x_{ij}))$, this factor is the disagreement-weighted average of $1+\mathrm{erf}(x_{ij})$ over model pairs, where $x_{ij} = \sqrt{N/2}\,|\hat{\Delta}_{ij}|/\hat{\nu}_{ij}$ is estimated from the mean per-dataset score gaps and their standard deviations (\cref{sec:inductive_bias}). Since $0 \le \mathrm{erf}(x) < 1$, the factor lies in $[1,2]$: it approaches $2$ when model pairs are well separated and $1$ when they are nearly tied. As the disagreement sum is dominated by near-tied pairs, the factor admits the closed form
\begin{equation*}
  \frac{\int_0^\infty \left(1-\mathrm{erf}^2(x)\right)dx}{\int_0^\infty \left(1-\mathrm{erf}(x)\right)dx} = \sqrt{2} \approx 1.41,
\end{equation*}
which matches the value measured on STRABLE. We use this factor for the oracle curve in \cref{fig:combined_pareto_oracle}b, where the grey band spans the full range $[1,2]$ (i.e., $\bbE[\tau_{108,\infty}] \in [0.90, 0.95]$). \Cref{tab:finite_n_bias} reports the resulting oracle agreement and its difference from the asymptotic factor $2$, which shrinks as $N$ grows.

\begin{table}[h]
  \caption{Agreement to the oracle ranking $\bbE[\tau_{N,\infty}]$ with the finite-$N$ factor $\sqrt{2}$, and its difference from the asymptotic factor $2$.}
  \label{tab:finite_n_bias}
  \centering\small
  \begin{tabular}{lcc}
    \toprule
    $N$ & $\bbE[\tau_{N,\infty}]$ & difference to factor $2$ \\
    \midrule
    27  & 0.863 & $0.057$ \\
    54  & 0.903 & $0.040$ \\
    108 & 0.931 & $0.028$ \\
    \bottomrule
  \end{tabular}
\end{table}

\subsection{Extension: accounting for inductive biases}\label{sec:inductive_bias}

The simplified model described in \cref{assum:perf_model} treats deviations from the mean performance solely as random noise. However, in empirical machine learning, and especially in tabular learning, performance variance is often driven by the structural compatibility between an algorithm's inductive bias and the specific characteristics of a dataset \citep{grinsztajn2022tree}. To capture this heterogeneity, we introduce a latent variable model that accounts for model-dataset affinity.

\begin{assumption}[Inductive bias decomposition]\label{assum:inductive_bias_model}
    Each dataset $d \in \mathcal{D}$ is associated with a latent vector of meta-features $z_d$. We decompose the observed performance $X_{i,d}$ into a population mean, an interaction term capturing inductive bias, and residual noise:
    \begin{align}\label{eq:model_perf_full}
        X_{i,d} = \mu_i + \beta_i^\top z_d + \epsilon_{i,d},
    \end{align}
    where $\beta_i$ represents the sensitivity vector of model $i$ to the dataset meta-features. This assumption replaces \cref{assum:perf_model} for the remainder of this analysis.
\end{assumption}

Under this formulation, the term $\beta_i^\top z_d$ quantifies the specific affinity between model $i$ and dataset $d$. To facilitate the derivation of closed-form bounds, we specify the distributional properties of these latent features.

\begin{assumption}[Gaussian meta-features]\label{assum:gaussian_meta}
    The latent meta-features are distributed according to a multivariate Gaussian distribution centered at zero, $z_d \sim \mathcal{N}(0, \Sigma_z)$, where $\Sigma_z$ denotes the covariance of dataset characteristics across the domain $\mathcal{D}$. Furthermore, $z_d$ is independent of the observation noise $\epsilon_{i,d}$.
\end{assumption}

\begin{remark}[Centering assumption]
    The centering assumption $\mathbb{E}_{d}[z_d]=0$ is made without loss of generality, as any non-zero mean in the latent distribution can be absorbed into the intrinsic performance term $\mu_i$. Indeed, if $\mathbb{E}[z_d] = \bar{z} \neq 0$, we can rewrite \autoref{eq:model_perf_full} as $X_{i,d} = (\mu_i + \beta_i^\top \bar{z}) + \beta_i^\top (z_d - \bar{z}) + \epsilon_{i,d}$. Defining $\tilde{\mu}_i \defeq \mu_i + \beta_i^\top \bar{z}$ and $\tilde{z}_d \defeq z_d - \bar{z}$, we recover that $\mathbb{E}[\tilde{z}_d] = 0$.
\end{remark}

\begin{remark}[Relaxing the Gaussian assumption]
    We emphasize that the formal assumption of Gaussian latent meta-features $z_d$ serve primarily to simplify exposition. Our following derivations rely solely on the distribution of the benchmark average $\bar z_B = \frac{1}{N} \sum_{d \in B} z_d$. Under mild regularity conditions (specifically, the existence of a finite second moment), the central limit theorem ensures that $\bar z_B$ converges asymptotically to a Gaussian distribution as $N \to \infty$. Consequently, \cref{assum:gaussian_meta} can be relaxed without affecting the asymptotic validity of our results.
\end{remark}

\begin{lemma}[Distribution of performance gap under inductive bias]\label{lem:gap_distribution_inductive}
    Under Assumptions \ref{assum:dataset_distribution}, \ref{assum:gaussian_noise}, \ref{assum:inductive_bias_model}, and \ref{assum:gaussian_meta}, the observed performance gap $Y_{ij}^{(B)}$ between two models $i$ and $j$ on a benchmark $B$ of size $N$ follows a Gaussian distribution:
    \begin{align}\label{eq:gap_dist_inductive}
        Y_{ij}^{(B)} \sim \mathcal{N}\left(\Delta_{ij}, \frac{\nu^2_{ij}}{N}\right),
    \end{align}
    where $\nu^2_{ij} \defeq (\beta_i - \beta_j)^\top \Sigma_z (\beta_i - \beta_j) + 2\sigma^2$ represents the effective pairwise variance, which combines the variance from the observation noise with the one induced by the difference of inductive biases between the models.
\end{lemma}

\begin{proof}
    Substituting the decomposition from \autoref{eq:model_perf_full} into the definition of the performance gap (\autoref{eq:perf_gap_def}), we obtain:
    \begin{align}
        Y_{ij}^{(B)} &= \frac{1}{N} \sum_{d \in B} \left( (\mu_i - \mu_j) + (\beta_i - \beta_j)^\top z_d + (\epsilon_{i,d} - \epsilon_{j,d}) \right) \\
        &= \Delta_{ij} + \gamma_{ij}^\top \bar{z}_B + \bar{\eta}_{i,j}^{(B)},
    \end{align}
    where $\gamma_{ij} \defeq \beta_i - \beta_j$ denotes the differential sensitivity vector, $\bar{z}_B \defeq \frac{1}{N} \sum_{d \in B} z_d$ is the average latent meta-feature vector of the benchmark, and $\bar{\eta}_{i,j}^{(B)}$ is the averaged noise term defined in \autoref{eq:perf_gap}.
    
    Under \cref{assum:gaussian_meta}, the linear combination of Gaussian random variables remains Gaussian. Specifically, $\bar{z}_B \sim \mathcal{N}(0, \frac{1}{N}\Sigma_z)$. Since the noise terms $\epsilon$ are independent of $z$ (\cref{assum:gaussian_meta}) and independent across datasets, the terms $\gamma_{ij}^\top \bar{z}_B$ and $\bar{\eta}_{i,j}^{(B)}$ are independent Gaussian variables with zero means. The total variance is thus the sum of their variances:
    \begin{align}
        \mathrm{Var}\left(Y_{ij}^{(B)}\right) &= \mathrm{Var}\left(\gamma_{ij}^\top \bar{z}_B\right) + \mathrm{Var}\left(\bar{\eta}_{i,j}^{(B)}\right) \\
        &= \gamma_{ij}^\top \left( \frac{1}{N}\Sigma_z \right) \gamma_{ij} + \frac{2\sigma^2}{N} \\
        &= \frac{1}{N} \left( \gamma_{ij}^\top \Sigma_z \gamma_{ij} + 2\sigma^2 \right).
    \end{align}
    Defining $\nu^2_{ij} \defeq \gamma_{ij}^\top \Sigma_z \gamma_{ij} + 2\sigma^2$ yields the result. Note that even if $z_d$ is not strictly Gaussian, the central limit theorem ensures that $\bar{z}_B$ converges to this distribution for sufficiently large $N$, making the result asymptotically valid under weaker assumptions.
\end{proof}

Starting from \cref{lem:gap_distribution_inductive}, we can follow the same derivations as in \autoref{sec:2_bench} and \autoref{sec:1_bench} to obtain, under Assumptions \ref{assum:dataset_distribution}, \ref{assum:gaussian_noise}, \ref{assum:inductive_bias_model}, and \ref{assum:gaussian_meta}, the four results that follow.

\begin{proposition}[Expected ranking consistency] \label{prop:kendall_tau_2_bench_induc}
Let $\tau_{N,N}$ denote the Kendall-$\tau$ rank correlation between the rankings of $k$ models on two independent benchmarks of size $N$. The expected correlation is given by
    \begin{align}\label{eq:kendall_tau_2_bench_induc}
        \mathbb{E}[\tau_{N,N}] = \frac{2}{k(k-1)} \sum_{1 \le i < j \le k} \mathrm{erf}^2\left( \sqrt{\frac{N}{2}} \cdot \frac{|\Delta_{ij}|}{\nu_{ij}} \right).
    \end{align}
\end{proposition}
\begin{proof}
    Same as \cref{prop:kendall_tau_2_bench}.
\end{proof}

\begin{corollary}[Asymptotic agreement rate] \label{cor:asymp_2_bench_induc}
    Let $\rho_{ij} \defeq |\Delta_{ij}|/\nu_{ij}$ denote the signal-to-noise ratio for the pair $(i,j)$, and let $\rho_{\min} \defeq \min_{i<j} \rho_{ij}$. As $N \to \infty$, the expected disagreement decays as:
    \begin{align}\label{eq:asymp_kendall_tau_2_bench_induc}
        1 - \mathbb{E}[\tau_{N,N}] \sim \frac{\tilde C}{\sqrt{N}} \exp\left(-\frac{\rho_{\min}^2 N}{2}\right),
    \end{align}
    where $\tilde C = \frac{4\sqrt{2} M_{\min}}{k(k-1)\rho_{\min}\sqrt{\pi}}$ and $M_{\min}$ is the number of pairs achieving the minimum ratio $\rho_{\min}$.
\end{corollary}
\begin{proof}
    We apply the asymptotic expansion $\mathrm{erf}^2(x) = 1 - \frac{2e^{-x^2}}{x\sqrt{\pi}} + O(x^{-3}e^{-x^2})$ for $x \to \infty$ to \autoref{eq:kendall_tau_2_bench_induc}, with argument $x_{ij} = \rho_{ij}\sqrt{N/2}$. The summation is dominated by the terms with the slowest exponential decay, which corresponds to the smallest coefficient $\rho_{ij}$. Unlike the homoskedastic setting in \cref{cor:asymp_2_bench}, where the convergence rate was dictated by the performance gap $\Delta_{\min}$, here it is governed by the signal-to-noise ratio $\rho_{\min}$. This implies that a pair of models with a large performance gap $\Delta_{ij}$ may still be the bottleneck for ranking convergence if their relative performance has a high variance (high $\nu_{ij}$) due to different inductive biases (high $\gamma_{ij}$).
\end{proof}

\begin{proposition}[Expected convergence rate to oracle] \label{prop:kendall_tau_1_bench_induc}
Let $\tau_{N,\infty}$ be the Kendall-$\tau$ rank correlation between the ranking of $k$ models observed on a benchmark of size $N$ and the oracle ranking. The expected correlation is given by:
    \begin{align}\label{eq:kendall_tau_1_bench_induc}
        \mathbb{E}[\tau_{N,\infty}] = \frac{2}{k(k-1)} \sum_{i<j} \mathrm{erf}\left( \sqrt{\frac{N}{2}} \cdot \frac{|\Delta_{ij}|}{\nu_{ij}} \right).
    \end{align}
\end{proposition}
\begin{proof}
    Same as \cref{prop:kendall_tau_1_bench}.
\end{proof}

\begin{corollary}[Asymptotic convergence rate to the oracle] \label{cor:asymp_1_bench_induc}
    As $N \to \infty$, the expected disagreement $1 - \mathbb{E}[\tau_{N,\infty}]$ between the ranking on a benchmark of size $N$ and the oracle ranking decays as:
    \begin{align}\label{eq:asymp_kendall_tau_1_bench_induc}
        1 - \mathbb{E}[\tau_{N,\infty}] \sim \frac{\tilde C}{2\sqrt{N}} \exp\left(-\frac{\rho_{\min}^2 N}{2}\right),
    \end{align}
    where $\tilde C$ and $\rho_{\min}$ are the constants defined in \cref{cor:asymp_2_bench_induc}.
\end{corollary}
\begin{proof}
    Applying the asymptotic expansion of the error function $\mathrm{erf}(x) \approx 1 - \frac{e^{-x^2}}{x\sqrt{\pi}}$ to \autoref{eq:kendall_tau_1_bench_induc} yields a summation dominated by the terms corresponding to the minimum signal-to-noise ratio $\rho_{\min}$. Comparing this to the expansion of $1 - \mathbb{E}[\tau_{N,N}]$ derived in \cref{cor:asymp_2_bench_induc}, which relies on $1 - \mathrm{erf}^2(x) \approx 2(1 - \mathrm{erf}(x))$, we observe that the expected disagreement with the oracle is asymptotically half the expected disagreement between two independent benchmarks. The exponential decay rate remains controlled by $\rho_{\min}$.
\end{proof}

\subsection{Disagreement at position 1: identifying the best model}
\label{app:position1}
 
The Kendall-$\tau$ analysis characterises the ranking disagreement across \emph{all} model pairs.
A practically more direct question is: how often do two independent benchmarks
disagree on \emph{which model ranks first}?
We derive a bound on this probability and compare its convergence rate
to that of Kendall-$\tau$.
 
Let $i^* \defeq \operatorname*{arg\,max}_i \mu_i$ be the oracle best model and
$\Delta_1 \defeq \mu_{i^*} - \max_{j \neq i^*} \mu_j > 0$ the margin between
the best and second-best model.
 
\begin{definition}[Position-1 disagreement]
\label{def:pos1}
Let $R_N$ and $R'_N$ be rankings produced by two independent benchmarks
$B, B'$ of size $N$.
The \emph{position-1 disagreement event} is
\[
  \varepsilon_1(N) \;\defeq\;
  \Bigl\{\,\operatorname*{arg\,max}_i \,\bar{X}_i^{(B)} \;\neq\; \operatorname*{arg\,max}_i \,\bar{X}_i^{(B')}\,\Bigr\}.
\]
\end{definition}
 
\begin{proposition}[Position-1 disagreement probability]
\label{prop:pos1}
Assumptions \ref{assum:dataset_distribution}-- \ref{assum:gaussian_noise},
\[
  \mathbb{P}\bigl[\varepsilon_1(N)\bigr]
  \;\leq\;
  2(k-1)\;\Phi\!\left(-\frac{\Delta_1\sqrt{N}}{\sigma\sqrt{2}}\right),
\]
where $\Phi$ is the standard normal CDF and $k$ is the number of models.
\end{proposition}
 
\begin{proof}
\textbf{Step 1.}
The event $\varepsilon_1(N)$ requires at least one benchmark to rank the wrong
model first, so
\[
  \varepsilon_1(N)
  \;\subseteq\;
  \bigl\{\operatorname*{arg\,max}_i \,\bar{X}_i^{(B)} \neq i^*\bigr\}
  \;\cup\;
  \bigl\{\operatorname*{arg\,max}_i \,\bar{X}_i^{(B')}  \neq i^*\bigr\}.
\]
Since $B$ and $B'$ are i.i.d., both events have equal probability, and
by the union bound:
\[
  \mathbb{P}[\varepsilon_1(N)]
  \;\leq\;
  2\,\mathbb{P}\bigl[\operatorname*{arg\,max}_B \neq i^*\bigr].
\]
 
\textbf{Step 2.}
The best model fails to rank first on $B$ iff at least one rival beats it:
\[
  \bigl\{\operatorname*{arg\,max}_i \,\bar{X}_i^{(B)} \neq i^*\bigr\}
  \;=\;
  \bigcup_{j \neq i^*}
  \bigl\{\bar{X}_j^{(B)} > \bar{X}_{i^*}^{(B)}\bigr\}.
\]
A second union bound over the $k-1$ rivals gives
\[
  \mathbb{P}\bigl[\operatorname*{arg\,max}_i \,\bar{X}_i^{(B)} \neq i^*\bigr]
  \;\leq\;
  \sum_{j \neq i^*}
  \mathbb{P}\bigl[\bar{X}_j^{(B)} > \bar{X}_{i^*}^{(B)}\bigr].
\]
 
\textbf{Step 3.}
Under Assumption~\ref{assum:gaussian_noise},
$\bar{X}_{i^*}^{(B)} - \bar{X}_j^{(B)} \sim \mathcal{N}(\Delta_{i^*j},\, 2\sigma^2/N)$,
where $\Delta_{i^*j} = \mu_{i^*} - \mu_j$.
Standardising:
\[
  \mathbb{P}\bigl[\bar{X}_j^{(B)} > \bar{X}_{i^*}^{(B)}\bigr]
  \;=\;
  \Phi\!\left(-\frac{\Delta_{i^*j}\sqrt{N}}{\sigma\sqrt{2}}\right).
\]
Since $\Delta_{i^*j} \geq \Delta_1$ for all $j \neq i^*$, and $\Phi$ is
increasing, a larger gap yields a more negative argument and thus a smaller
value:
\[
  \Phi\!\left(-\frac{\Delta_{i^*j}\sqrt{N}}{\sigma\sqrt{2}}\right)
  \;\leq\;
  \Phi\!\left(-\frac{\Delta_1\sqrt{N}}{\sigma\sqrt{2}}\right).
\]
Summing over $k-1$ rivals and combining with Steps~1--2 gives the result.
\end{proof}
 
\begin{corollary}[Asymptotic decay of position-1 disagreement]
\label{cor:pos1-decay}
As $N\to\infty$,
\[
  \mathbb{P}[\varepsilon_1(N)]
  \;\leq\;
  C_1 \cdot N^{-1/2} \cdot \exp\!\left(-\frac{\Delta_1^2 N}{4\sigma^2}\right),
  \qquad
  C_1 = \frac{2(k-1)\,\sigma}{\Delta_1\sqrt{\pi}},
\]
obtained from the Gaussian tail bound
$\Phi(-t) \leq e^{-t^2/2}/(t\sqrt{2\pi})$ \citep{vershynin2018}
applied with $t = \Delta_1\sqrt{N}/(\sigma\sqrt{2})$.
\end{corollary}
 
\paragraph{Comparison with Kendall-$\tau$ convergence.}
Corollary~\ref{cor:asymp_2_bench} showed that Kendall-$\tau$ disagreement
decays with exponent $-\Delta_{\min}^2/(4\sigma^2)$, where
$\Delta_{\min} = \min_{i < j}|\Delta_{ij}|$ is the smallest gap over
\emph{all} model pairs.
Since $\Delta_1$ is a gap between $i^*$ and its nearest rival,
whereas $\Delta_{\min}$ is the smallest gap between any two models in the entire set (which may be achieved by two closely matched mid-ranking models rather than by the top pair) we have $\Delta_1 \geq \Delta_{\min}$.
Setting each bound equal to a threshold $\varepsilon$ and solving,
the required benchmark sizes satisfy
\[
  N^*_{\text{pos-1}}
  \;\propto\;
  \frac{4\sigma^2}{\Delta_1^2}\log\frac{1}{\varepsilon},
  \qquad
  N^*_{\text{Kendall}}
  \;\propto\;
  \frac{4\sigma^2}{\Delta_{\min}^2}\log\frac{1}{\varepsilon},
  \qquad
  \frac{N^*_{\text{pos-1}}}{N^*_{\text{Kendall}}}
  = \left(\frac{\Delta_{\min}}{\Delta_1}\right)^2 \leq 1.
\]
Position-1 identification is therefore guaranteed at a smaller benchmark
size than full ranking stability, by a factor of $(\Delta_{\min}/\Delta_1)^2$.

\begin{remark}[Tightness of the bound]
The union bound in Proposition~\ref{prop:pos1} is not tight: it treats all
$k-1$ rivals as if each had probability $\Phi(-\Delta_1\sqrt{N}/(\sigma\sqrt{2}))$
of beating $i^*$, whereas rivals with larger gaps $\Delta_{i^*j} \gg \Delta_1$
contribute negligibly.
The bound suffices to establish the exponential decay rate and the
qualitative comparison with Kendall-$\tau$.
\end{remark}

\subsection{Top-1 disagreement under inductive biases}
\label{app:position1-bias}
 
Section~\ref{app:position1} assumed that the difficulty of distinguishing
$i^*$ from a rival $j$ is determined solely by their mean gap $\Delta_{i^*j}$,
with all model pairs sharing the same noise level $\sigma$.
This assumption breaks down when models have inductive biases: systematic
tendencies to perform well on certain dataset types and poorly on others.
Two models with different inductive biases will have a performance difference
that varies \emph{systematically} across datasets, not just randomly.
 
Lemma~\ref{lem:gap_distribution_inductive} showed that the sample average
difference satisfies
\[
  \bar{X}_{i^*}^{(B)} - \bar{X}_j^{(B)}
  \;\sim\;
  \mathcal{N}\!\left(\Delta_{i^*j},\;\frac{\nu_{i^*j}^2}{N}\right),
\]
where $\nu_{i^*j}^2 = (\beta_{i^*}-\beta_j)^\top \Sigma_z (\beta_{i^*}-\beta_j) + 2\sigma^2$.
The first term captures inductive bias variability: it is large when $i^*$ and $j$
respond differently to dataset characteristics ($\beta_{i^*} \neq \beta_j$) and
when the benchmark spans diverse dataset types (large $\Sigma_z$).
The second term $2\sigma^2$ is the pure noise contribution from Section~\ref{app:position1}.
Critically, $\nu_{i^*j}^2$ differs across pairs: a rival whose biases mirror
$i^*$'s recovers $\nu_{i^*j}^2 \approx 2\sigma^2$, while a rival with
very different biases has $\nu_{i^*j}^2 \gg 2\sigma^2$.
The probability that rival $j$ beats $i^*$ on benchmark $B$ is
$\Phi(-\Delta_{i^*j}\sqrt{N}/\nu_{i^*j})$, which depends on $\Delta_{i^*j}$
and $\nu_{i^*j}$ only through their ratio $\rho_{i^*j} = \Delta_{i^*j}/\nu_{i^*j}$.
A large mean gap $\Delta_{i^*j}$ does not therefore guarantee a small error
probability if $\nu_{i^*j}$ is comparably large: what governs the difficulty
of beating rival $j$ is whether the gap is large \emph{relative to} the
variability of their performance difference across datasets.
 
\begin{definition}[Per-pair SNR for top-1]
\label{def:snr-pos1}
For each rival $j \neq i^*$, define the signal-to-noise ratio
$\rho_{i^*j} = \Delta_{i^*j} / \nu_{i^*j}$,
and let $\rho_1 \defeq \min_{j \neq i^*} \rho_{i^*j}$ be the minimum SNR
over all rivals of $i^*$.
\end{definition}
 
The rival $j^\dagger$ achieving $\rho_1$ is the hardest to reliably beat at
position~1.
Unlike Section~\ref{app:position1}, $j^\dagger$ need not be the second-best
model in expectation.
To see why, suppose the second-best model $j_1$ has gap $\Delta_{i^*j_1} = 0.05$
and similar biases to $i^*$, giving $\nu_{i^*j_1} = \sigma\sqrt{2}$ and
$\rho_{i^*j_1} = 0.05/(\sigma\sqrt{2})$.
A weaker model $j_2$ has $\Delta_{i^*j_2} = 0.20$ but very different biases,
giving $\nu_{i^*j_2} = 5\sigma\sqrt{2}$ and
$\rho_{i^*j_2} = 0.04/(\sigma\sqrt{2}) < \rho_{i^*j_1}$.
Despite being further behind on average, $j_2$ is harder to reliably beat
because its performance relative to $i^*$ swings widely across datasets.
 
\begin{proposition}[Position-1 disagreement under inductive biases]
\label{prop:pos1-bias}
Under Assumptions~\ref{assum:dataset_distribution}--\ref{assum:gaussian_meta},
\[
  \mathbb{P}[\varepsilon_1(N)]
  \;\leq\;
  2(k-1)\;\Phi\!\left(-\rho_1\sqrt{N}\right).
\]
\end{proposition}
 
\begin{proof}
Steps~1--2 are identical to Proposition~\ref{prop:pos1}.
In Step~3, Lemma~\ref{lem:gap_distribution_inductive} gives
$\bar{X}_{i^*}^{(B)} - \bar{X}_j^{(B)} \sim \mathcal{N}(\Delta_{i^*j}, \nu_{i^*j}^2/N)$,
so standardising with standard deviation $\nu_{i^*j}/\sqrt{N}$ yields
\[
  \mathbb{P}\bigl[\bar{X}_j^{(B)} > \bar{X}_{i^*}^{(B)}\bigr]
  \;=\;
  \Phi(-\rho_{i^*j}\sqrt{N}).
\]
Since $\rho_{i^*j} \geq \rho_1$ for all $j \neq i^*$ and $\Phi$ is increasing,
$\Phi(-\rho_{i^*j}\sqrt{N}) \leq \Phi(-\rho_1\sqrt{N})$.
Summing over $k-1$ rivals and applying Step~1 gives the result.
\end{proof}
 
\begin{corollary}[Asymptotic decay under inductive biases]
\label{cor:pos1-bias-decay}
As $N \to \infty$,
\[
  \mathbb{P}[\varepsilon_1(N)]
  \;\leq\;
  \tilde{C}_1 \cdot N^{-1/2} \cdot \exp\!\left(-\frac{\rho_1^2 N}{2}\right),
  \qquad
  \tilde{C}_1 = \frac{2(k-1)}{\rho_1\sqrt{2\pi}}.
\]
The decay exponent $-\rho_1^2/2$ reduces to $-\Delta_1^2/(4\sigma^2)$ of
Corollary~\ref{cor:pos1-decay} when $\nu_{i^*j^\dagger} = \sigma\sqrt{2}$
(no inductive bias), and is smaller whenever the inductive bias inflates
$\nu_{i^*j^\dagger}$.
\end{corollary}
 
\paragraph{Comparison with Corollary~\ref{cor:asymp_2_bench_induc}.}
Corollary~\ref{cor:asymp_2_bench_induc} showed that Kendall-$\tau$ disagreement
under inductive biases decays with exponent $-\rho_{\min}^2/2$, where
$\rho_{\min} = \min_{i < j}|\Delta_{ij}|/\nu_{ij}$ is the minimum SNR over
all model pairs.
Since $\rho_1$ is a minimum over only the pairs involving $i^*$, while
$\rho_{\min}$ is a minimum over all pairs including those with no relation
to $i^*$, we have $\rho_1 \geq \rho_{\min}$.
This means the position-1 bound decays with a more negative exponent
$-\rho_1^2/2 \leq -\rho_{\min}^2/2$, so fewer datasets are needed to reach
any fixed error threshold $\varepsilon$.
Setting each bound equal to $\varepsilon$ and solving:
\[
  N^*_{\text{pos-1}}
  \;\propto\;
  \frac{2}{\rho_1^2}\log\frac{1}{\varepsilon},
  \qquad
  N^*_{\text{Kendall}}
  \;\propto\;
  \frac{2}{\rho_{\min}^2}\log\frac{1}{\varepsilon},
  \qquad
  \frac{N^*_{\text{pos-1}}}{N^*_{\text{Kendall}}}
  \;=\;
  \left(\frac{\rho_{\min}}{\rho_1}\right)^2 \leq 1.
\]
In other words, guaranteeing that two benchmarks agree on the top-ranked model
requires fewer datasets than guaranteeing agreement on the full ranking.
How many fewer depends on how much larger $\rho_1$ is than $\rho_{\min}$:
if the hardest pair to separate is one involving $i^*$ then $\rho_1 = \rho_{\min}$
and there is no advantage; if the hardest pair involves two mid-ranking models
with nothing to do with $i^*$ then $\rho_1 \gg \rho_{\min}$ and the advantage
can be substantial.
 
\begin{remark}[Implications for STRABLE]
\label{rem:strable-pos1}
Using Proposition~\ref{prop:pos1} alone, a practitioner would assess the
difficulty of identifying Tf-Idf + TabPFN-2.5 as the top-1 model by looking
at the mean gap $\Delta_1$ to the second-best pipeline: large mean gap implies
easy identification.
Proposition~\ref{prop:pos1-bias} asks a harder question: is that gap
\emph{consistent} across dataset types?
In STRABLE, Section \ref{sec:heterogeneity} shows that dataset characteristics such as string
length and string diversity cause substantial shifts in the relative ordering
of pipelines, meaning that the performance difference between the top pipeline
and its rivals varies systematically with dataset type. This is a direct empirical signature of non-zero inductive bias terms
$(\beta_{i^*} - \beta_j)^\top \Sigma_z (\beta_{i^*} - \beta_j)$,
which inflate $\nu_{i^*j}$ and lower the SNR $\rho_{i^*j}$.
Since $\rho_1 \leq \Delta_1/(\sigma\sqrt{2})$, Proposition~A.6 requires
at least as many datasets as Proposition~A.5 to reach the same guarantee,
and strictly more whenever any rival has inductive biases different from $i^*$.

Inductive biases introduce a new way for a rival model to be dangerous: a rival model may be far behind on average performance with respect to $i^*$ but highly variable, therefore it can still frequently outscore $i^*$ on specific benchmarks. Once accounted for that, one may need more datasets to wash out those accidental wins and stably identify i* as the winner.
\end{remark}

\section{Current benchmark landscape}

\autoref{tab:benchmark_comparison} compares STRABLE against existing tabular benchmarks. Existing suites either focus on numerical data, rely on heavy curation to remove raw strings, or lack the scale of general-purpose benchmarks.

\begin{table}[t]
\centering
\caption{Comparison of STRABLE against existing tabular benchmarks.}
\label{tab:benchmark_comparison}
\resizebox{\textwidth}{!}{%
\begin{tabular}{@{}llp{8cm}@{}}
\toprule
\textbf{Benchmark} & \textbf{Attention to string features} & \textbf{Size and Scope} \\ \midrule
\multicolumn{3}{@{}l}{\textit{String-excluding benchmarks}} \\
\textbf{OpenML-CC18} \citep{bischl2021openmlbenchmarkingsuites} & Mostly numerical & High ($N=72$); curated classification tasks. \\
\textbf{PMLB / PMLBmini} \citep{romano2021pmlbv10opensource} & Numerical / Low-cardinality & High ($N \approx 290$); inclusive of simplified datasets. \\
\textbf{Grinsztajn et al. (2022)} \citep{grinsztajn2022tree} & One-Hot Encoding & Moderate ($N=45$); removes high-cardinality features. \\
\textbf{TabReD} \citep{rubachev2024tabredanalyzingpitfallsfilling} & Removed / Numerical & Low ($N=8$); industry-grade but removes string signals. \\
\textbf{TabArena} \citep{erickson2025tabarenalivingbenchmarkmachine} & Curated / IID & Moderate ($N=51$); does not include complex string signals. \\ \midrule
\multicolumn{3}{@{}l}{\textit{String-flattening benchmarks}} \\
\textbf{McElfresh et al. (2024)} \citep{mcelfresh2024neuralnetsoutperformboosted} & Standard Vectorization & High ($N=176$); relies on OpenML pre-processed formats. \\
\textbf{AMLB} \citep{gijsbers2023amlbautomlbenchmark} & AutoML System Specific Encoding & High ($N \approx 104$); focuses on AutoML framework evaluation. \\
\textbf{TabRepo} \citep{salinas2024tabrepolargescalerepository} & N-gram and Method Specific Encoding & High ($N=211$); large-scale repository of model evaluations. \\
\textbf{TALENT} \citep{ye2025closerlookdeeplearning} & Standard Vectorization & High ($N=300$); broad scope but standard numerical focus. \\
\textbf{Zab\"ergja et al. (2025)} \citep{zabërgja2025tabulardatadeeplearning} & Standard Vectorization & Moderate ($N=68$); formally treats strings as mathematical vectors. \\
\textbf{BeyondArena} \citep{purucker2026iidgeneraltabularfoundation} & Fixed language model embeddings & High ($N=142$); unified benchmark for diverse IID and non-IID tasks. \\ 
\textbf{TEmBed} \citep{vogel2026universaltabularembeddingsbenchmark} & Text Serialization & Moderate ($N=69$); evaluates embeddings across cell, row, column, and table levels.\\
\midrule
\multicolumn{3}{@{}l}{\textit{Narrow string-aware benchmarks}} \\
\textbf{Shi et al. (2021)} \citep{shi2021benchmarkingmultimodalautomltabular} & Raw free-text & Low ($N=18$); text-dominant tables with few tabular features. \\
\textbf{CARTE} \citep{kim2024cartepretrainingtransfertabular} & LLM-embedded strings & Moderate ($N=51$); curated datasets with discrete entries; lower density of text columns. \\
\textbf{MulTaBench} \citep{arazi2026multabenchbenchmarkingmultimodaltabular} & Target-aware fine-tuned encoders & Moderate ($N=40$ core / $80$ extended); isolates multimodal tasks where fine-tuned encoders improve performance, plus supplemental datasets. \\
\textbf{TextTabBench} \citep{mraz2025benchmarkingfoundationmodelstabular} & Long free-text & Low ($N=13$); limited dataset diversity. \\ \midrule
\textbf{STRABLE (Ours)} & \textbf{Raw heterogeneous strings} & \textbf{High ($N=108$); raw, uncurated, diverse string data.} \\
\bottomrule
\end{tabular}%
}
\end{table}

\section{Dataset collection}\label{app:datasets}

\subsection{Dataset sources and characteristics}

STRABLE is collected from $33$ sources, spanning across 8 different domains (Table \ref{tab:dataset_distribution_table1.5}). We targeted tables published by open-data providers (e.g., FDA, EMA, OSHA, World Bank), for three reasons. First, these providers cover the recurrent application areas of tabular ML — health, economic and business forecasting, education, energy, and infrastructure — which are also the domains represented in established benchmarks (e.g., TALENT spans finance, physics, and education). Second, their tables retain string columns rather than pre-encoded features, which is the setting STRABLE is designed to evaluate and which existing benchmarks exclude or flatten (Table \ref{tab:benchmark_comparison}). Third, this string-heavy composition matches enterprise data: \citet{kim2026exploringdifferencestabularenterprise} report that enterprise tables are predominantly string-valued with low feature entropy due to repetitive values. Consistently, free text accounts for only 8.23\% of STRABLE's string columns (\Cref{subsec:string-taxonomy}), while categorical, name and code entries dominate. Section \ref{sec:stability_benchmark} verifies empirically that pipeline rankings remain consistent across these domains.

\begin{table}
  \caption{Tables per field and task type.}
  \label{tab:dataset_distribution_table1.5}
  \centering
  \begin{tabular}{lcccc}
    \toprule
    \textbf{Field} & \textbf{b-cls} & \textbf{m-cls} & \textbf{reg} & \textbf{Total} \\
    \midrule
    Commerce        & 2 & 2 & 1  & \textbf{5} \\
    Economy         & 2 & 1 & 23 & \textbf{26} \\
    Education       & 1 & 0 & 9  & \textbf{10} \\
    Energy          & 0 & 0 & 9  & \textbf{9} \\
    Food            & 0 & 3 & 3  & \textbf{6} \\
    Health          & 8 & 9 & 13 & \textbf{30} \\
    Infra.          & 0 & 4 & 14 & \textbf{18} \\
    Social          & 0 & 0 & 4  & \textbf{4} \\
    \textbf{Total}  & \textbf{13} & \textbf{19} & \textbf{76} & \textbf{108} \\
    \bottomrule
    \end{tabular}
\end{table}

\begin{description}[itemsep=1pt,topsep=0pt]
    \item [\textbf{Commerce (4 sources):}] European-Commission, webrobots.io, mercari.com, Yelp Open Dataset.
    \item [\textbf{Economy (7 sources):}] aijobs.net, kaggle, Consumer-Financial-Protection-Bureau, Federal-Deposit-Insurance-Corporation, data.ct.gov, lendingclub.com, worldbankfinancesone.
    \item [\textbf{Education (4 sources):}] commonlit.org, FSA, Institute of Museum and Library Services, SCIMAGO.
    \item [\textbf{Energy (3 sources):}] energydata.info, fueleconomy.gov, world-resource-institute.
    \item [\textbf{Food (6 sources)}:] BeerAdvocate.com, flavorsofcacao.com, whiskyanalysis.com, Michelin, theramenrater.com, majestic.co.uk.
    \item [\textbf{Health (6 sources)}:] ClinicalTrials.gov, European-Medicines-Agency, FDA, HRSA, Medicaid, osha.gov.
    \item [\textbf{Infrastructure (2 sources)}:] HIFLD, data.sfgov.org.
    \item [\textbf{Social (1 source)}:] OHCA
\end{description}

In addition, \autoref{fig:dataset_exploration} shows the datasets distribution per application field and year of assembling of the dataset (when unavailable, publication year was used); distribution of respective performances across all regression and classification tasks. \autoref{tab:dataset_features_part2} shows the Median and inter-quartile range of different datasets features. In comparison, TextTaBench \citep{mraz2025benchmarkingfoundationmodelstabular} is distinguished by significantly longer average string lengths, while CARTE \citep{kim2024cartepretrainingtransfertabular} exhibits markedly higher cardinality distributions driven by categorical variations. STRABLE occupies a structural middle ground, featuring information-dense entries of moderate length and cardinality, distinct from both the long-context requirements of TextTaBench and the high-cardinality entity matching tasks of CARTE.

\begin{figure}[t!]
\begin{minipage}{.33\linewidth}
    \includegraphics[width=\linewidth]{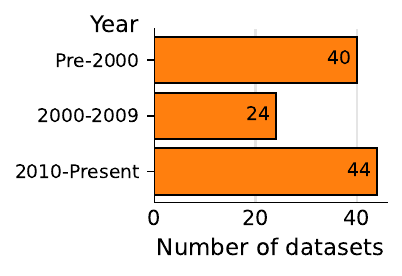}%
\end{minipage}%
\begin{minipage}{.335\linewidth}
    \includegraphics[width=\linewidth]{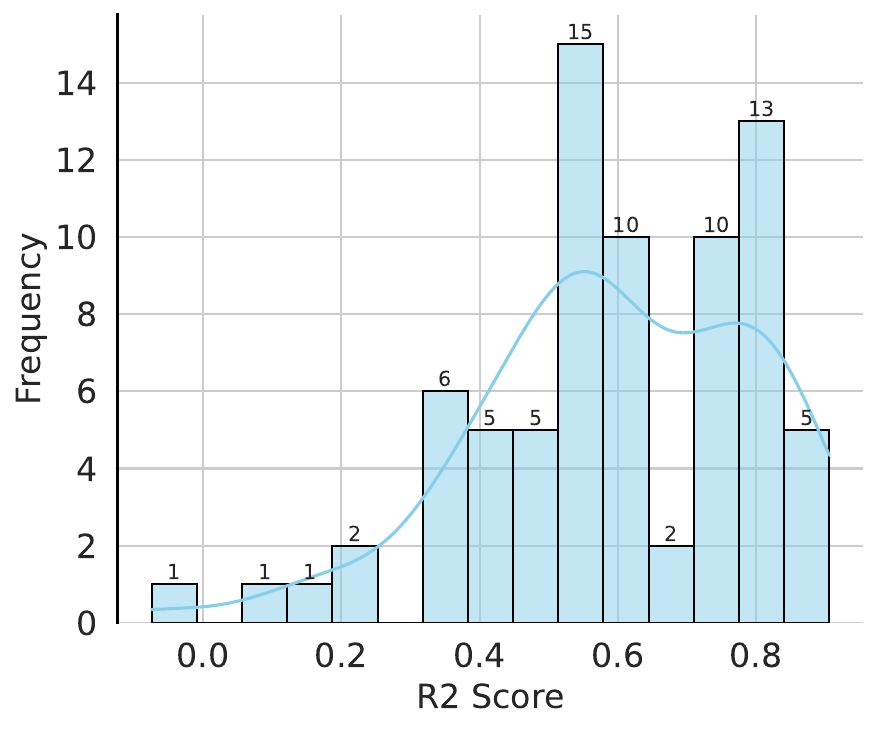}%
\end{minipage}%
\begin{minipage}{.335\linewidth}
    \includegraphics[width=\linewidth]{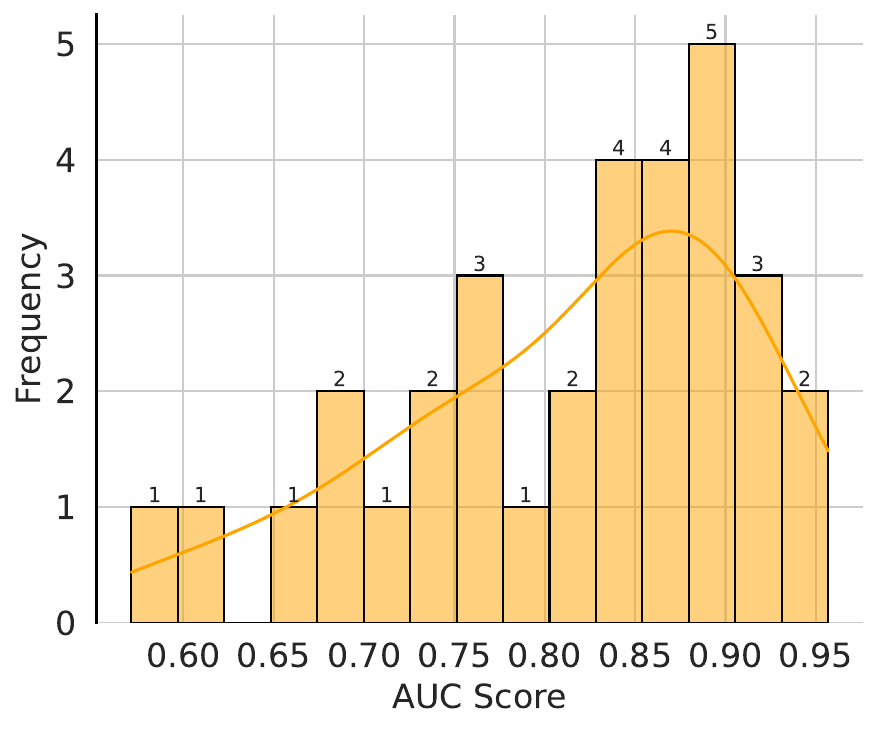}%
\end{minipage}%
    \caption{From left to right: Number of datasets per publication or collection year, distribution of the R2 across all regression tasks, distribution of AUC across all classification tasks.}
    \label{fig:dataset_exploration}
\end{figure}

\begin{table}[t]
\centering
\caption{Summary statistics of curated datasets by category: Median [IQR]}
\label{tab:dataset_features_part2}
\setlength{\tabcolsep}{3pt} \footnotesize
\begin{tabular}{lccccc}
\toprule
Category & Number of Rows & Number of Columns & Number of String Columns & \makecell{Avg. String\\Length} & Cardinality \\
\midrule
Commerce & 75000 [42186, 75000] & 14 [14, 19] & 12 [9, 13] & 49 [22, 58] & 19158 [11476, 35516] \\
Economy & 7796 [4723, 16160] & 13 [11, 19] & 10 [9, 17] & 17 [12, 24] & 1110 [288, 2567] \\
Education & 10928 [5045, 21906] & 16 [13, 29] & 10 [8, 12] & 16 [11, 26] & 1894 [1518, 4850] \\
Energy & 3978 [1238, 31448] & 21 [19, 29] & 17 [14, 23] & 13 [12, 22] & 416 [135, 3287] \\
Food & 2993 [2056, 4554] & 13 [8, 18] & 8 [6, 13] & 19 [10, 46] & 888 [377, 1648] \\
Health & 4366 [1743, 14988] & 16 [10, 25] & 12 [8, 22] & 31 [14, 47] & 696 [380, 1554] \\
Infrastructure & 12504 [4530, 48014] & 28 [24, 38] & 17 [15, 21] & 15 [13, 17] & 2195 [1013, 5674] \\
Social & 10081 [5637, 22476] & 14 [12, 20] & 12 [10, 16] & 14 [12, 17] & 490 [305, 672] \\
\bottomrule
\end{tabular}
\end{table}

\subsection{Detailed description of datasets}\label{descr_data}

We provide detailed description and url links to the datasets. For broken links, refer to the link in the abstract to access the datasets.

\begin{enumerate}[topsep=0pt, itemsep=0pt]
\item \textbf{ACA Federal Upper Limits}\footnote{\url{https://www.medicaid.gov/medicaid/prescription-drugs/federal-upper-limit}} Price limits for multi-source drugs under the Medicaid program. The task is to predict the federal upper price limit. 
\item \textbf{AI/ML Salaries}\footnote{\url{https://ai-jobs.net/salaries/download/salaries.csv}} Salary and basic information for workers in the machine learning and data science industry. The task is to predict worker salaries. 
\item \textbf{Animal and Veterinary Event}\footnote{\url{https://open.fda.gov/apis/animalandveterinary/event/}} Health problems reported in animals following the use of drug products. The task is to predict the severity of clinical signs. 
\item \textbf{Antenna Structure Registration}\footnote{\url{https://hifld-geoplatform.opendata.arcgis.com/datasets/geoplatform::antenna-structure-registration/}} FCC registration data for antenna structures. The task is to predict the height of the structures. 
\item \textbf{Awarded Grants IMLS}\footnote{\url{https://www.imls.gov/grants/awarded-grants}} Grants awarded by the Institute of Museum and Library Services. The task is to predict the specific grant amount. 
\item \textbf{Beer Ratings}\footnote{\url{https://www.kaggle.com/datasets/ruthgn/beer-profile-and-ratings-data-set}} Tasting profiles and consumer reviews for over 3,000 unique beers. The task is to predict overall review ratings. \item \textbf{Broadband Availability}\footnote{\url{https://www.fcc.gov/economics-analytics/broadband-data-collection}} Data on internet speed and availability across the US. The task is to predict the maximum available download speed. 
\item \textbf{California Housing}\footnote{\url{https://www.kaggle.com/datasets/camnugent/california-housing-prices}} Median house values and demographics from the 1990 California census. The task is to predict median house prices. 
\item \textbf{Child Adult Healthcare Quality}\footnote{\url{https://data.medicaid.gov/dataset/quality-of-care-child-core-set}} Quality of care metrics for Medicaid and CHIP beneficiaries. The task is to predict healthcare performance scores. \item \textbf{China Overseas Finance Inventory}\footnote{\url{https://www.wri.org/data/china-overseas-finance-inventory-database}} Chinese investments in power generation projects worldwide. The task is to predict the total investment amount. 
\item \textbf{Chocolate Bar Ratings}\footnote{\url{https://www.kaggle.com/datasets/rtatman/chocolate-bar-ratings}} Expert ratings and information on cocoa batches. The task is to predict the professional rating score. 
\item \textbf{Clear Corpus}\footnote{\url{https://www.commonlit.org/blog/introducing-the-clear-corpus-an-open-dataset-to-advance-research-28ff8cfea84a}} Reading passage excerpts for elementary school students. The task is to predict the readability of the text. 
\item \textbf{Cohort Default Rate}\footnote{\url{https://nsldsfap.ed.gov/cdr-searchable-database}} Student loan default rates for US postsecondary institutions. The task is to predict the default rate percentage. 
\item \textbf{College Credit Card Marketing}\footnote{\url{https://www.consumerfinance.gov/data-research/student-banking/marketing-agreements-and-data/}} Marketing agreements between credit card issuers and universities. The task is to predict the number of open accounts. 
\item \textbf{College Deposit Product Marketing}\footnote{\url{https://www.consumerfinance.gov/data-research/student-banking/marketing-agreements-and-data/}} Agreements regarding deposit products offered to college students. The task is to predict associated financial metrics. 
\item \textbf{Colleges and Universities}\footnote{\url{https://hifld-geoplatform.opendata.arcgis.com/datasets/geoplatform::colleges-and-universities/}} Locations and characteristics of US postsecondary institutions. The task is to predict student enrollment. 
\item \textbf{Commitments in Trust Funds}\footnote{\url{https://finances.worldbank.org/Trust-Funds-and-FIFs/Commitments-in-Trust-funds/DS00072}} Approved commitments in World Bank trust fund ledgers. The task is to predict the total commitment amount. 
\item \textbf{Community Banking}\footnote{\url{https://www.fdic.gov/community-banking-research-program/reference-data}} Financial metrics for community banks in the US. The task is to predict bank asset sizes or performance ratios. 
\item \textbf{Conflict Events}\footnote{\url{https://data.humdata.org/dataset/acled-data}} Geospatial data on political violence and protest events. The task is to predict the number of fatalities. 
\item \textbf{Contract Awards IPF}\footnote{\url{https://finances.worldbank.org/Procurement/Contract-Awards-in-Investment-Project-Financing/DS00073}} Contracts financed by the World Bank under Investment Project Financing. The task is to predict the award amount. 
\item \textbf{Contributions to FIFs}\footnote{\url{https://finances.worldbank.org/Trust-Funds-and-FIFs/Contributions-to-Financial-Intermediary-Funds/DS00051}} Financial contributions to multilateral Financial Intermediary Funds. The task is to predict donor contribution levels. 
\item \textbf{Corporate Procurement Contracts}\footnote{\url{https://finances.worldbank.org/Procurement/Corporate-Procurement-Contract-Awards/DS00028}} Listing of contract awards executed by the World Bank Group. The task is to predict the contract value. 
\item \textbf{Cosmetic Event}\footnote{\url{https://open.fda.gov/apis/cosmetic/event/}} Adverse events reported for cosmetic products to the FDA. The task is to predict the severity of the reported reaction. 
\item \textbf{COVID-19 Clinical Trials}\footnote{\url{https://clinicaltrials.gov/ct2/results?cond=COVID-19}} Metadata for clinical trials related to COVID-19. The task is to predict trial status or enrollment numbers. 
\item \textbf{Device Classification}\footnote{\url{https://open.fda.gov/apis/device/classification/}} FDA classification of medical devices based on intended use. The task is to predict the specific device class. 
\item \textbf{Device COVID-19 Serology}\footnote{\url{https://open.fda.gov/apis/device/covid19serology/}} Performance data for COVID-19 antibody tests. Each row is one test device applied to one serum sample from a fixed reference panel. The task is to predict whether the sample truly contains
antibodies. Since every device is evaluated on the same reference samples, the
label is almost fully determined by the pair (panel, sample number). Under
row-level i.i.d. splits, this table behaves as a look-up task.
\item \textbf{Device PMA}\footnote{\url{https://open.fda.gov/apis/device/pma/}} Premarket approval applications for high-risk medical devices. The task is to predict the final decision status. 
\item \textbf{Disbursements in Trust Funds}\footnote{\url{https://finances.worldbank.org/Trust-Funds-and-FIFs/Disbursements-in-Trust-Funds/DS00074}} Cash payments made to recipients of World Bank trust funds. The task is to predict the disbursement amount. 
\item \textbf{Discretionary Grant}\footnote{\url{https://data.hrsa.gov/data/reports}} Grants awarded based on competitive applications by HRSA. The task is to predict the total grant award amount. 
\item \textbf{Drug Drugs@FDA}\footnote{\url{https://open.fda.gov/apis/drug/drugsfda/}} Information about FDA-approved brand name and generic drugs. The task is to predict approval years. 
\item \textbf{Drug Enforcement}\footnote{\url{https://open.fda.gov/apis/drug/enforcement/}} FDA enforcement actions related to drug products. The task is to predict the recall classification level. 
\item \textbf{Drug NDC}\footnote{\url{https://open.fda.gov/apis/drug/ndc/}} National Drug Code Directory containing all drugs in the US. The task is to predict the drug category. 
\item \textbf{Drug Shortages}\footnote{\url{https://www.fda.gov/drugs/drug-safety-and-availability/drug-shortages}} Information on current and resolved drug shortages. The task is to predict the duration of the shortage. 
\item \textbf{Electric Generating Plants}\footnote{\url{https://hifld-geoplatform.opendata.arcgis.com/datasets/geoplatform::power-plants/}} Operational characteristics of power plants in the US. The task is to predict net generation capacity. 
\item \textbf{Electric Retail Service Territories}\footnote{\url{https://hifld-geoplatform.opendata.arcgis.com/datasets/geoplatform::electric-retail-service-territories/}} Areas served by electric utility companies. The task is to predict the utility ownership type. 
\item \textbf{EMA Medicines}\footnote{\url{https://www.ema.europa.eu/en/medicines/download-medicine-data}} Information on medicines authorized by the European Medicines Agency. The task is to predict authorization status. 
\item \textbf{External Clinician Dashboard}\footnote{\url{https://data.hrsa.gov/data/reports}} Performance metrics for clinicians in HRSA programs. The task is to predict clinician productivity scores. 
\item \textbf{FIF Cash Transfers}\footnote{\url{https://finances.worldbank.org/Trust-Funds-and-FIFs/Financial-Intermediary-Funds-Cash-Transfers/DS00029}} Cash transfers from FIFs to implementing agencies. The task is to predict the transfer amount. 
\item \textbf{FIF Commitments}\footnote{\url{https://finances.worldbank.org/Trust-Funds-and-FIFs/Financial-Intermediary-Funds-Commitments/DS00052}} Funding commitments made by Financial Intermediary Funds. The task is to predict the commitment value. 
\item \textbf{FIF Funding Decisions}\footnote{\url{https://finances.worldbank.org/Trust-Funds-and-FIFs/Financial-Intermediary-Funds-Funding-Decisions/DS00053}} Decisions on funding allocations by FIF governing bodies. The task is to predict the decision outcome. 
\item \textbf{Financial Management Medicaid}\footnote{\url{https://data.medicaid.gov/dataset/state-expenditures}} State expenditures on Medicaid programs and services. The task is to predict total program costs. 
\item \textbf{Financial Product Complaint}\footnote{\url{https://www.consumerfinance.gov/data-research/consumer-complaints/}} Consumer complaints regarding financial products and services. The task is to predict the response category. \item \textbf{Food Enforcement}\footnote{\url{https://open.fda.gov/apis/food/enforcement/}} Recall and enforcement actions for food products. The task is to predict the reason for the recall. 
\item \textbf{Food Event}\footnote{\url{https://open.fda.gov/apis/food/event/}} Adverse events and product complaints for food and supplements. The task is to predict the consumer outcome. 
\item \textbf{Food Prices}\footnote{\url{https://data.humdata.org/dataset/wfp-food-prices}} Historical food prices from markets worldwide. The task is to predict the price of staple crops. 
\item \textbf{Foreign Gift and Contract}\footnote{\url{https://studentaid.gov/data-center/school/foreign-gifts}} Reports of gifts or contracts from foreign sources to US colleges. The task is to predict the gift value. 
\item \textbf{FTS Funding}\footnote{\url{https://fts.unocha.org/}} Humanitarian aid flows tracked by the Financial Tracking Service. The task is to predict funding per crisis. 
\item \textbf{FTS Requirements and Funding}\footnote{\url{https://fts.unocha.org/}} Requirements vs. funding for humanitarian response plans. The task is to predict the funding gap. 
\item \textbf{Gainful Employment}\footnote{\url{https://studentaid.gov/data-center/school/ge}} Debt-to-earnings ratios for graduates of vocational programs. The task is to predict if a program passes standards. 
\item \textbf{Global Dams Database}\footnote{\url{https://www.energydata.info/dataset/global-dams-database}} Geographic and structural information on dams worldwide. The task is to predict dam capacity. 
\item \textbf{Global Power Plant}\footnote{\url{https://www.wri.org/data/global-power-plant-database}} Comprehensive database of power plants worldwide. The task is to predict annual electricity generation. 
\item \textbf{Grant}\footnote{\url{https://taggs.hhs.gov/}} General information on federal grants awarded by HHS. The task is to predict the project funding amount. 
\item \textbf{Health Professional Shortage Areas}\footnote{\url{https://data.hrsa.gov/topics/health-workforce/shortage-areas}} Regions with a shortage of healthcare providers. The task is to predict the shortage score. 
\item \textbf{Historic Perimeters Wildfires}\footnote{\url{https://hifld-geoplatform.opendata.arcgis.com/datasets/geoplatform::historic-perimeters-wildfires/}} Historical boundaries of major wildfires in the US. The task is to predict total acres burned. 
\item \textbf{Historical Earthquake Locations}\footnote{\url{https://hifld-geoplatform.opendata.arcgis.com/datasets/geoplatform::historical-earthquake-locations/}} Global database of significant historical earthquakes. The task is to predict earthquake magnitude. 
\item \textbf{Historical Volcanic Locations}\footnote{\url{https://hifld-geoplatform.opendata.arcgis.com/datasets/geoplatform::historical-volcanic-locations/}} Locations and eruption history of significant volcanoes. The task is to predict the eruption type. 
\item \textbf{Hospitals}\footnote{\url{https://hifld-geoplatform.opendata.arcgis.com/datasets/geoplatform::hospitals/}} Comprehensive list of US hospitals and their facilities. The task is to predict the number of hospital beds. 
\item \textbf{Hypertension Control}\footnote{\url{https://data.hrsa.gov/data/reports}} Clinical performance data on hypertension management. The task is to predict patient blood pressure control percentage. 
\item \textbf{IBRD Statement of Loans}\footnote{\url{https://finances.worldbank.org/Loans-and-Credits/IBRD-Statement-of-Loans-Latest-Available-Snapshot/DS00049}} Historical record of loans and guarantees issued by the IBRD. The task is to predict the loan status. 
\item \textbf{IDA Statement of Credits}\footnote{\url{https://finances.worldbank.org/Loans-and-Credits/IDA-Statement-of-Credits-and-Grants-Latest-Availabl/DS00048}} Records of credits, grants, and guarantees issued by IDA. The task is to predict the disbursement status. 
\item \textbf{IFC Advisory Projects}\footnote{\url{https://finances.worldbank.org/Advisory-Services/IFC-Advisory-Services-Projects/DS00010}} Metadata on advisory services provided by the IFC. The task is to predict the project budget. \item \textbf{IFC Investment Projects}\footnote{\url{https://finances.worldbank.org/Investment-Services/IFC-Investment-Services-Projects/DS00011}} Records of investment projects undertaken by the IFC. The task is to predict the investment amount. 
\item \textbf{Industry Payments Entity}\footnote{\url{https://www.cms.gov/openpayments}} Payments made by drug and device companies to teaching hospitals. The task is to predict the payment amount. 
\item \textbf{Industry Payments Project}\footnote{\url{https://www.cms.gov/openpayments}} Payments related to specific research projects or clinical trials. The task is to predict research funding. 
\item \textbf{Insurance Company Complaints}\footnote{\url{https://data.ct.gov/Business/Insurance-Company-Complaints/}} Consumer complaints filed against insurance companies. The task is to predict resolution status. 
\item \textbf{Journal Ranking}\footnote{\url{https://www.scimagojr.com/journalrank.php}} Scientific journal metrics including H-index and citations. The task is to predict the journal's impact factor. 
\item \textbf{Kickstarter Projects}\footnote{\url{https://www.kaggle.com/datasets/kemical/kickstarter-projects}} Funding goals and outcomes for Kickstarter campaigns. The task is to predict project success. 
\item \textbf{Lending Club Loan}\footnote{\url{https://www.lendingclub.com/info/download-data.action}} Information on loans issued through the Lending Club platform. The task is to predict the interest rate. 
\item \textbf{Local Government Renewable Action}\footnote{\url{https://www.wri.org/data/local-government-renewables-action-tracker}} Renewable energy initiatives taken by local governments. The task is to predict project capacity. 
\item \textbf{Local Law Enforcement}\footnote{\url{https://hifld-geoplatform.opendata.arcgis.com/datasets/geoplatform::local-law-enforcement-locations/}} Locations of local police and sheriff departments in the US. The task is to predict officer counts. 
\item \textbf{Managed Care Enrollment}\footnote{\url{https://data.medicaid.gov/dataset/managed-care-enrollment-report}} Enrollment statistics for Medicaid managed care plans. The task is to predict the number of enrollees. 
\item \textbf{Media Ranking}\footnote{\url{https://www.scimagojr.com/}} Rankings for media and social science publications. The task is to predict the impact factor. 
\item \textbf{Medically Underserved Areas}\footnote{\url{https://data.hrsa.gov/topics/health-workforce/shortage-areas}} Areas with populations lacking access to primary care. The task is to predict the underservice index. 
\item \textbf{Mercari Price Prediction}\footnote{\url{https://www.kaggle.com/c/mercari-price-suggestion-challenge}} Product descriptions and categories from Mercari. The task is to predict the listing price. 
\item \textbf{Michelin Ratings}\footnote{\url{https://www.kaggle.com/datasets/ngshiheng/michelin-guide-restaurants-2021}} Details on restaurants curated in the Michelin Guide. The task is to predict the award level. 
\item \textbf{MIGA Issued Projects}\footnote{\url{https://finances.worldbank.org/Guarantees/MIGA-Issued-Projects/DS00030}} Projects supported by MIGA investment guarantees. The task is to predict the maximum gross exposure. 
\item \textbf{MLR Summary Reports}\footnote{\url{https://data.medicaid.gov/dataset/medical-loss-ratio}} Medical Loss Ratio data for healthcare plans. The task is to predict the MLR percentage. 
\item \textbf{Mobile Home Parks}\footnote{\url{https://hifld-geoplatform.opendata.arcgis.com/datasets/geoplatform::mobile-home-parks/}} Geographic locations and capacities of mobile home parks. The task is to predict the number of lots. 
\item \textbf{Museums}\footnote{\url{https://www.imls.gov/research-evaluation/data-collection/museum-universe-data-file}} Information on museums and related organizations in the US. The task is to predict annual revenue. 
\item \textbf{NADAC Rates}\footnote{\url{https://data.medicaid.gov/dataset/national-average-drug-acquisition-cost}} Weekly survey of drug acquisition costs for retail pharmacies. The task is to predict the acquisition cost per unit. 
\item \textbf{National Average Drug Acquisition Cost}\footnote{\url{https://data.medicaid.gov/dataset/national-average-drug-acquisition-cost}} Survey of retail pharmacy drug acquisition costs. The task is to predict units costs. 
\item \textbf{Oil and Natural Gas Platforms}\footnote{\url{https://hifld-geoplatform.opendata.arcgis.com/datasets/geoplatform::oil-natural-gas-platforms/}} Offshore oil and gas platforms in US waters. The task is to predict platform status. 
\item \textbf{Orphan Designations}\footnote{\url{https://www.ema.europa.eu/en/medicines/download-medicine-data}} Medicines designated for rare diseases by the EMA. The task is to predict therapeutic area. 
\item \textbf{OSHA Accidents}\footnote{\url{https://www.osha.gov/fatalities}} Reports of workplace accidents and fatalities. The task is to predict injury severity. 
\item \textbf{Paediatric Investigation Plan}\footnote{\url{https://www.ema.europa.eu/en/medicines/download-medicine-data}} Research plans for the use of medicines in children. The task is to predict investigation status. 
\item \textbf{POL Terminal}\footnote{\url{https://hifld-geoplatform.opendata.arcgis.com/datasets/geoplatform::petroleum-terminals/}} Petroleum, Oil, and Lubricant storage terminals. The task is to predict storage capacity. 
\item \textbf{Power Plants}\footnote{\url{https://hifld-geoplatform.opendata.arcgis.com/datasets/geoplatform::power-plants/}} Details on fuel type and location of US power plants. The task is to predict the energy source. 
\item \textbf{Prepaid Financial Product}\footnote{\url{https://www.consumerfinance.gov/data-research/prepaid-accounts/}} Information on prepaid financial products and terms. The task is to predict fee structures. 
\item \textbf{Prison Boundaries}\footnote{\url{https://hifld-geoplatform.opendata.arcgis.com/datasets/geoplatform::prison-boundaries/}} Geospatial boundaries and capacities of US correctional facilities. The task is to predict population. 
\item \textbf{Ramen Ratings}\footnote{\url{https://www.kaggle.com/datasets/ankanhore545/top-ramen-ratings-2022}} Reviews and ratings for various ramen products globally. The task is to predict the star rating. 
\item \textbf{RASFF Window}\footnote{\url{https://ec.europa.eu/food/safety/rasff_en}} Food and feed safety alerts from the EU Rapid Alert System. The task is to predict the risk level. 
\item \textbf{RASNF Notification List}\footnote{\url{https://ec.europa.eu/safety-gate-alerts/screen/search}} The EU rapid alert system for dangerous non-food products. The task is to predict whether the product affects more than one country. 
\item \textbf{Recipient Executed Grants}\footnote{\url{https://finances.worldbank.org/Trust-Funds-and-FIFs/Recipient-executed-grants-commitments-disbursement/DS00075}} Commitments and disbursements for grants executed by recipients. The task is to predict grant value. 
\item \textbf{Schools}\footnote{\url{https://hifld-geoplatform.opendata.arcgis.com/datasets/geoplatform::public-schools/}} Locations and metadata for public K-12 schools. The task is to predict student counts. 
\item \textbf{SF Building Permits}\footnote{\url{https://data.sfgov.org/Housing-and-Buildings/Building-Permits/i98e-oved}} Building permit applications in San Francisco. The task is to predict construction cost. 
\item \textbf{Summary of Deposit}\footnote{\url{https://www.fdic.gov/bank/statistical/stats/}} Branch-level deposit data for US banks. The task is to predict total deposits per branch. 
\item \textbf{Tax Incentives}\footnote{\url{https://data.ct.gov/Business/Tax-Incentives/}} Business tax incentives granted by Connecticut. The task is to predict the tax credit amount. 
\item \textbf{Terms CC Plans}\footnote{\url{https://www.consumerfinance.gov/data-research/credit-card-data/}} Terms and conditions for various credit card plans. The task is to predict interest rates. 
\item \textbf{Tobacco Problem}\footnote{\url{https://open.fda.gov/apis/tobacco/problem/}} Health or product problems related to tobacco. The task is to predict the health problem type. 
\item \textbf{Total Contributions IBRD IDA IFC}\footnote{\url{https://finances.worldbank.org/}} Contributions to World Bank institutions by members. The task is to predict total contribution amounts. 
\item \textbf{Transmission Lines}\footnote{\url{https://hifld-geoplatform.opendata.arcgis.com/datasets/geoplatform::electric-power-transmission-lines/}} Electric power transmission infrastructure in the US. The task is to predict voltage levels. 
\item \textbf{Transmission Towers}\footnote{\url{https://hifld-geoplatform.opendata.arcgis.com/datasets/geoplatform::antenna-structure-registration/}} Structures for wireless and broadcast transmission. The task is to predict tower height. 
\item \textbf{US School Bus Fleet}\footnote{\url{https://hifld-geoplatform.opendata.arcgis.com/datasets/geoplatform::school-buses/}} Data on school bus fleets and fuel types in the US. The task is to predict bus counts. 
\item \textbf{Vehicles}\footnote{\url{https://www.fueleconomy.gov/feg/download.shtml}} Fuel economy data for cars sold in the US. The task is to predict the annual fuel cost. 
\item \textbf{Whisky Ratings}\footnote{\url{https://whiskyanalysis.com/index.php/database/}} Tasting notes and meta-critic scores for whiskies. The task is to predict the overall rating. 
\item \textbf{Wine Dataset}\footnote{\url{https://www.kaggle.com/datasets/manyregression/updated-wine-enthusiast-review}} Wine reviews and prices from Wine Enthusiast magazine. The task is to predict the score or price. 
\item \textbf{Workforce Demographics}\footnote{\url{https://data.hrsa.gov/data/reports}} Demographic information for the health professional workforce. The task is to predict regional workforce density. 
\item \textbf{Yelp Business}\footnote{\url{https://www.yelp.com/dataset}} Metadata on businesses including categories and reviews. The task is to predict the star rating. \end{enumerate}

\begin{table}[t]
\centering
\footnotesize
\caption{Median (IQR) of Runtime and MTEB Performance per Language Model}
\label{tab:runtime_summary}
\begin{tabularx}{\textwidth}{l >{\raggedright\arraybackslash}X c c}
\toprule
Language Model & Hugging Face & \makecell{Median (IQR)\\Runtime [s]} & \makecell{MTEB (En)\\Score} \\
\midrule
LM All-MiniLM-L12-v2 & sentence-transformers/all-MiniLM-L12-v2 & 12 [4, 48] & - \\
LM All-MiniLM-L6-v2 & sentence-transformers/all-MiniLM-L6-v2 & 3 [1, 12] & 56.03 \\
LM All-MPNet-base-v2 & sentence-transformers/all-mpnet-base-v2 & 13 [4, 55] & - \\
LM BGE-base & BAAI/bge-base-en-v1.5 & 11 [4, 32] & 65.14 \\
LM BGE-large & BAAI/bge-large-en-v1.5 & 35 [9, 100] & 65.89 \\
LM BGE-small & BAAI/bge-small-en-v1.5 & 12 [3, 42] & 64.30 \\
LM DeBERTa-v3-base & microsoft/deberta-v3-base & 28 [10, 103] & - \\
LM DeBERTa-v3-large & microsoft/deberta-v3-large & 48 [14, 182] & - \\
LM DeBERTa-v3-small & microsoft/deberta-v3-small & 12 [6, 35] & - \\
LM DeBERTa-v3-xsmall & microsoft/deberta-v3-xsmall & 20 [6, 75] & - \\
LM E5-base-v2 & intfloat/e5-base-v2 & 7 [2, 24] & 61.67 \\
LM E5-large-v2 & intfloat/e5-large-v2 & 16 [4, 59] & 62.79 \\
LM E5-small-v2 & intfloat/e5-small-v2 & 6 [2, 19] & 61.32 \\
LM F2LLM-0.6B & codefuse-ai/F2LLM-0.6B & 52 [14, 191] & 70.03 \\
LM F2LLM-1.7B & codefuse-ai/F2LLM-1.7B & 72 [17, 279] & 72.01 \\
LM F2LLM-4B & codefuse-ai/F2LLM-4B & 290 [74, 997] & 73.67 \\
LM FastText & - & 0 [0, 1] & - \\
LM Gemma-0.3B & google/gemma-3-270m & 47 [14, 157] & - \\
LM Jasper-0.6B & infgrad/Jasper-Token-Compression-600M & 22 [7, 77] & 74.75 \\
LM KALM-embed & HIT-TMG/KaLM-embedding-multilingual-mini-instruct-v1.5 & 56 [15, 210] & 71.29 \\
LM LLaMA-3.1-8B & meta-llama/Llama-3.1-8B & 249 [61, 1040] & - \\
LM LLaMA-3.2-1B & meta-llama/Llama-3.2-1B & 36 [9, 155] & - \\
LM LLaMA-3.2-3B & meta-llama/Llama-3.2-3B & 106 [26, 439] & - \\
LM LLaMA-Nemotron-Embed-1B-v2 & nvidia/llama-nemotron-embed-1b-v2 & 35 [8, 142] & - \\
LM ModernBERT-base & answerdotai/ModernBERT-base & 26 [8, 93] & - \\
LM ModernBERT-large & answerdotai/ModernBERT-large & 35 [11, 119] & - \\
LM OPT-0.1B & facebook/opt-125m & 13 [3, 45] & - \\
LM OPT-0.3B & facebook/opt-350m & 25 [6, 90] & - \\
LM OPT-1.3B & facebook/opt-1.3b & 57 [14, 219] & - \\
LM OPT-2.7B & facebook/opt-2.7b & 101 [24, 437] & - \\
LM OPT-6.7B & facebook/opt-6.7b & 233 [58, 939] & - \\
LM Qwen-3-0.6B & Qwen/Qwen3-Embedding-0.6B & 30 [9, 140] & 70.47 \\
LM Qwen-3-4B & Qwen/Qwen3-Embedding-4B & 156 [40, 653] & 74.61 \\
LM Qwen-3-8B & Qwen/Qwen3-Embedding-8B & 276 [72, 1174] & 75.23 \\
LM RoBERTa-base & FacebookAI/roberta-base & 7 [2, 25] & - \\
LM RoBERTa-large & FacebookAI/roberta-large & 20 [5, 73] & - \\
LM Sentence-T5-base & sentence-transformers/sentence-t5-base & 8 [3, 36] & 60.30 \\
LM Sentence-T5-large & sentence-transformers/sentence-t5-large & 17 [4, 62] & 77.67 \\
LM Sentence-T5-xl & sentence-transformers/sentence-t5-xl & 61 [14, 259] & 76.58 \\
LM Sentence-T5-XXL & sentence-transformers/sentence-t5-xxl & 268 [65, 1019] & 66.13 \\
LM UAE-large & WhereIsAI/UAE-Large-V1 & 24 [6, 99] & 66.40 \\
\bottomrule
\end{tabularx}
\end{table}

\subsection{Downsampling Strategy}\label{app:downsampling}
To ensure computational feasibility across our extensive 
benchmark of learners and encoders, we limit the maximum number 
of data points to $75{,}000$. For datasets exceeding the limit, 
they are downsampled using a fixed random seed for 
reproducibility. The sampling strategy depends on the task type: 
simple random sampling (uniform selection without replacement) 
for regression; stratified sampling to ensure that the class 
distribution in the subset matches the original target marginal 
distribution for classification. The threshold of $75{,}000$ is 
chosen to align with TabPFN-2.5's design specification for 
datasets of up to $50{,}000$ training samples 
\citep{grinsztajn2025tabpfn25advancingstateart}: under our 
$3$-fold cross-validation, this threshold ensures each training 
fold contains at most $50{,}000$ samples ($\approx 2/3 \times 
75{,}000$), with the remaining $25{,}000$ used for testing.

While the value is set by TabPFN-2.5, the cap itself is applied 
identically to all pipelines: our rank comparisons and 
Friedman/Conover--Iman tests are paired designs requiring the same 
data across pipelines, so capping only input-limited models would 
confound model quality with training-set size. Evaluating at fixed 
training sizes is standard practice in tabular benchmarks 
\citep{grinsztajn2022tree,kim2024cartepretrainingtransfertabular}, 
and a cap is needed regardless to bound the cost of 445 pipelines 
on tables with millions of rows (e.g., \texttt{financial-product-complaint}, 4.86M).

The cap affects 12 of the 108 tables, and we verified that it does 
not drive our conclusions. First, running 8 of these 12 tables at 
full size preserves the pipeline ranking (Kendall-$\tau = 0.96$, 
\autoref{fig:full_vs_sampled_comparison}). Second, ranking pipelines 
on the 96 uncapped tables alone reproduces the global ranking 
($\tau = 0.96$) with the same top-5 pipelines. The leading position 
of Tf-Idf + TabPFN-2.5 is therefore not an artifact of the cap.

\subsection{String Taxonomy and Profiling Methodology}
\label{subsec:string-taxonomy}

To better characterize the string modalities present within the STRABLE benchmark, we introduced a semantic taxonomy and applied it to profile all text columns across the curated datasets. 

\paragraph{Semantic Taxonomy} 
We categorize all non-numeric columns into six distinct semantic types:
\begin{itemize}
    \item \textbf{Categorical:} Low-uniqueness, repeating labels (e.g., ``Red'', ``General Acute Care'').
    \item \textbf{Name:} Proper nouns representing people, organizations, places, or products (e.g., ``John Doe'', ``Max Mara'').
    \item \textbf{Structured Code:} Strings with recognizable, meaningful patterns (e.g., ZIP codes, ICD/NDC medical codes, URLs).
    \item \textbf{Free Text:} Multi-word prose containing natural language and stopwords (e.g., user reviews, medical notes).
    \item \textbf{Identifier:} Near-unique, opaque keys with no inherent semantics (e.g., UUIDs, hashes, auto-generated IDs).
    \item \textbf{Datetime:} Strings encoding temporal information (e.g., ``2024-03-15'', ``March 15, 2024'', ``Q1 2024'').

\end{itemize}

\paragraph{Profiling Methodology}\label{sec:profiling_methodology}
To classify these columns at scale, we first isolate all text columns using \texttt{skrub.TableVectorizer}. We then compute a suite of deterministic indices for each column, which include:
\begin{itemize}
    \item \emph{Dictionary Hit Rate \& Stopword Density:} To distinguish natural language prose from random strings.
    \item \emph{Symbol Density, Proportion Numeric, \& Pattern Matching:} To detect structured formatting (such as slashes and dashes), numeric content, and standard regex patterns (such as dates and currencies).
    \item \emph{Token Metrics:} Average words per cell and the proportion of multi-word entries.
    \item \emph{Uniqueness Ratio:} To differentiate primary keys from repeating categorical variables.
\end{itemize}
We apply a heuristic function based on these computed indices to classify each column into one of the six taxonomic tags.

\paragraph{Validation Protocol}
To ensure the accuracy of our heuristic profiler, we conducted a rigorous validation loop on a sample of 30 datasets. We manually annotated the semantic types of the string columns and compared them against the heuristic's output, utilizing a State-of-the-Art LLM-as-a-judge as a secondary verifier. By iteratively refining the heuristic thresholds based on failure cases in this sample, our automated profiler achieved a 97\% agreement with the ground truth annotations.

\section{Evaluation pipeline}

\subsection{Pipeline components}

To support the evaluation of the STRABLE benchmark, we use a comprehensive corpus of encoding and learning components. These are categorized into modular encoder-learner pipelines and end-to-end models.

\paragraph{Encoder-learner pipelines}\label{app:encoder_learner_pipelines} We consider combinations of following encoders and learners.

\textit{Baselines of encoders:}

\begin{itemize}[topsep=0pt,itemsep=1pt]
    \item \texttt{Tf-Idf+SVD}: encodes strings of a given column using tf-idf vectorization and truncated singular value decomposition for dimensionality reduction. We use the \texttt{Skrub} package \citep{skrub2026} with its default value $30$ as the dimension.
    \item \texttt{TargetEncoder} \citep{micci2001preprocessing}: encodes categorical variables based on the global target mean and target values for observations belonging to the category. Categories that are not present in the train set are encoded with the target mean.
    \item \texttt{LM-}: a family of encoders that use pre-trained language models as feature extractors. For models on \texttt{HuggingFace}, we use the \texttt{Sentence-Transformers} package to extract the embeddings; for \texttt{FastText}, we rely on its supported package. The encoding is then followed by Principal Component Analysis (PCA) to reduce each column's representation to $30$ components. PCA is fitted with only non-null values, preserving missing entries as NaNs.
    \item \texttt{TARTE} \citep{kim2025table}: a model pre-trained from large knowledgebases. The model takes a table as the input and generates embeddings per row, with dimension of $768$.
\end{itemize}

\textit{Baselines of learners:}
\begin{itemize}[topsep=0pt,itemsep=1pt]
    \item \texttt{Ridge} \citep{pedregosa2011scikit}: simple linear model with an efficient cross-validation for hyperparameter selection. We record the results of internally tuned estimator.
    \item \texttt{XGBoost} \citep{Chen_2016}: a representative gradient boosting decision trees learner. Models are configured with a maximum of $1,000$ iterations with early-stopping (patience=300) based on a validation set. We record results of both default and tuned estimators.
    \item \texttt{ExtraTrees} \citep{geurts2006extremely}: fits a number of randomized decision trees on various sub-samples of the dataset with averaging for improved prediction. Implemented via \texttt{scikit-learn==1.7.2}, which provides native missing value support. We record results of both default and tuned estimators.
    \item \texttt{RealMLP} \citep{holzmuller2024better}: an improved MLP with architectural changes and meta-tuned default hyperparameters, specifically optimized for tabular data. We use the default configuration.

    \item \texttt{TabM} \citep{gorishniy2025tabm}: a tabular deep learning model based on parameter-efficient ensembling of MLPs via BatchEnsemble, producing multiple predictions per object. We use the default configuration.

    \item \texttt{TabPFN-2.5} \citep{grinsztajn2025tabpfn25advancingstateart}: in-context learner leveraging prior-data fitted networks trained with synthetic data. We use the official \texttt{2.5} release, with the default configurations (``TabPFN-2.5'' and ``Real-TabPFN-2.5'' for regression and classification, respectively).
    \item \texttt{TabICLv2} \citep{qu2026tabiclv2betterfasterscalable}: a tabular foundation model for in-context learning that extends TabICL with a novel synthetic-data engine, a scalable softmax for longer contexts, and improved pretraining. We use the default configuration. TabICLv2 was pretrained on up to 100 features, 
    which is below the typical post-PCA feature count of our 
    benchmark (mean 416, max 1270 features after 30-PCA per 
    high-cardinality column; see \autoref{tab:pca_col_estim}); this 
    likely contributes to its slightly lower performance compared to 
    TabPFN-2.5, which was pretrained on up to 2000 features 
    (\autoref{fig:comparative-pareto_optimality}).
\end{itemize}

\paragraph{End-to-end models}

End-to-end models jointly process encoders and learners for a given table (a generative LLM predictor, \texttt{TabuLa-8B}, is evaluated separately on five small tables in \autoref{tab:tabula}):
\begin{itemize}[topsep=0pt,itemsep=1pt]
    \item \texttt{ConTextTab} \citep{spinaci2025contexttabsemanticsawaretabularincontext}: in-context learner that has been pre-trained from the real-world T4 dataset curated in \citet{gardner2024large}. The learner combines string encodings from \texttt{All-MiniLM-L6-v2} with cross-row and cross-column attention \citep{kossen2022selfattentiondatapointsgoingindividual}. The model is instantiated from the \texttt{SAP\_RPT\_OSS} package using model defaults.
    \item \texttt{TabSTAR} \citep{arazi_tabstar_2025}: a pre-trained model from real world datasets with rich semantic information. The architecture integrates \texttt{e5-small-v2} that is tuned alongside the tabular backbone. The model is evaluated with the default parameters.
    \item \texttt{CatBoost} \citep{prokhorenkova2019catboostunbiasedboostingcategorical}: a gradient-boosted trees package commonly used to learn on tables. 
    Uses internal handling of categorical attributes
    and receives categorical feature indices explicitly from the encoding pipeline.
    We treat text features as categorical, encoded by CatBoost's categorical encoding, an improved version of target encoding.
    Models are configured with a maximum of $1,000$ iterations with early-stopping with \texttt{od\_type='Iter'} (patience=300) based on a validation set. We record results of both default and tuned estimators.
    \item \texttt{Mambular} \citep{thielmann2025mambularsequentialmodeltabular}: a tabular deep-learning architecture that treats features as a pseudo-sequence and processes them through Mamba state-space blocks. We use the \texttt{MambularClassifier} and \texttt{MambularRegressor} interfaces from \texttt{deeptab} with default parameters.

\end{itemize}

\paragraph{Note on \texttt{ConTextTab} pretraining-contamination risk.}\label{sec:contamination_contexttab}
\texttt{ConTextTab} \citep{spinaci2025contexttabsemanticsawaretabularincontext} is pretrained on 2.18M tables from the T4 corpus \citep{gardner2024large}, a Common Crawl and GitHub snapshot whose provenance has been shown to inflate \texttt{TabuLa-8B}'s results via train--test overlap \citep{gorla2026illusiongeneralizationtabularlanguage}. We therefore ran a leakage-detection procedure designed to be as strict as possible, i.e., to flag as many T4 tables as plausibly overlap with STRABLE. %

Every cell is canonicalized so that formatting differences cannot hide a match: lowercasing, whitespace stripping, number normalization (\texttt{3.140} $\equiv$ \texttt{3.14}) and null removal, followed by hashing to \texttt{uint64}. The canonicalized STRABLE tables are indexed in memory through three structures:
\begin{itemize}[topsep=0pt,itemsep=1pt]
  \item \textbf{Distinctive-value index.} Highly specific text strings (non-numeric, $\ge 5$ characters), to flag tables sharing rare identifiers.
  \item \textbf{Row-fingerprint index.} A column-order-independent hash of every row, as a fallback for low-entropy tables containing only common numbers or short codes.
  \item \textbf{Per-column value sets.} The unique values of every STRABLE column, to compute common-value ratios between T4 and STRABLE columns.
\end{itemize}

The entire T4 corpus, deduplicated by content hashing to 2.89M unique tables, is processed in a single sequential pass. The vast majority (99.99\%) of T4 tables share no rare value with STRABLE and are discarded by dictionary lookup; for the remainder we retain the top-$K$ highest-overlapping candidates per STRABLE table.

To separate genuine leaks from coincidental overlap, a value is considered \emph{rare} if it appears in $\le 100$ of the 2.89M unique tables ($\approx 0.003\%$ of the corpus), which removes false positives from shared generic vocabulary (state names, region codes). Containment is measured column-wise as $|S_{\text{col}} \cap T_{\text{col}}| / |S_{\text{col}}|$, with $S$ and $T$ the STRABLE and T4 columns; the asymmetry lets us catch T4 tables holding only a row-subset of the original.

A T4 table is flagged as contaminated with respect to a STRABLE table if it satisfies any of:
\begin{itemize}[topsep=0pt,itemsep=1pt]
  \item \textbf{Structural rule:} $\ge 2$ shared rare values and either $\ge 3$ columns with containment $> 0.7$, or mean column containment $> 0.4$. Being content-based, it is robust to column renaming.
  \item \textbf{Schema rule:} $\ge 4$ matching column names covering $\ge 50\%$ of the STRABLE table's columns, together with $\ge 3$ shared rare values.
  \item \textbf{Content rule:} $\ge 8$ shared rare values globally. This threshold is intentionally loose, to capture partial or transformed duplicates.
\end{itemize}

Under these rules, 2\% of the T4 corpus exhibits plausible overlap with STRABLE. To check whether \texttt{ConTextTab}'s inclusion distorts the comparison of the other pipelines, we re-run the aggregate analysis with \texttt{ConTextTab} removed from the pool (\autoref{fig:contexttab_ablation}): the relative standings, pairwise performance gaps and convergence trends of all other pipelines are unchanged.

\begin{figure}[t]
  \centering
  \includegraphics[width=\linewidth]{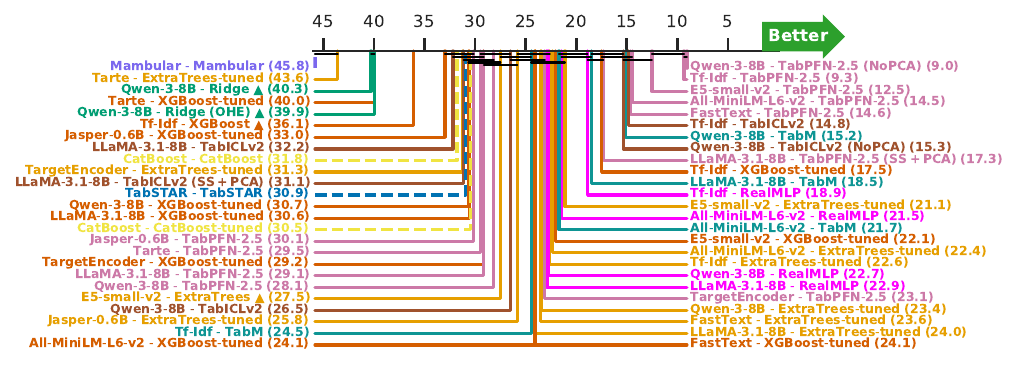}
  \caption{\textbf{Critical difference diagram of Figure 3 recomputed without ConTextTab.} Removing ConTextTab leaves the ranking unchanged (Kendall $\tau = 0.998$ with the order of Figure 3). The only change is a swap between two near-tied mid-ranking pipelines, LLaMA-3.1-8B + ExtraTrees-tuned and FastText + XGBoost-tuned (24th/25th, mean ranks within 0.01 of each other); in particular, the top-10 and bottom-5 pipelines keep exactly the same order. Qwen-3-8B + TabPFN-2.5 (NoPCA) remains in first place (mean rank 9.3 → 9.0), followed by Tf-Idf + TabPFN-2.5 (9.6 → 9.3).}
  \label{fig:contexttab_ablation}
\end{figure}

\paragraph{Cross-validation protocol}

To evaluate STRABLE, we employ a nested cross-validation protocol consisting of outer loop for performance measurement and inner loop for hyperparameter selection. For the outer loop, a dataset is partitioned into $3$ folds, with additional stratified sampling for classification tasks. In each iteration, one fold is held out as the test set, while the remaining folds constitute the training partition. Within each outer training partition, we perform hyperparameter selection via an internal $8$-fold cross-validation as done in \citet{erickson2025tabarenalivingbenchmarkmachine}.

\paragraph{Hyperparameter optimization and prediction}\label{para:hyperparameter-optimization}

For baseline estimators requiring hyperparameter optimization, we conduct a randomized search over $100$ iterations (including the default values). The configuration achieving the highest mean validation score in the inner loop is selected. Detailed search spaces can be found in \autoref{tab:strable_hyperparameter_space}. For final prediction, we employ two schemes: For learners that use a validation-set (\textit{e.g.,} \texttt{XGBoost}), we average the predictions of the $8$ folds, following \citet{erickson2025tabarenalivingbenchmarkmachine}; otherwise, we refit a model using the total training partition.

\paragraph{Evaluation metrics}

We report predictive performance and computational efficiency, with metrics of predictive power ($R^2$ for regression; AUROC for classification), durations of preprocessing and hyperparameter search, and inference latency. For regression tasks, we apply parameter-free target transformations (log, log1p, cbrt, arcsinh, signed-log) selected per dataset at preprocessing time using a skewness-minimization criterion; $R^2$ is reported on the transformed scale. We verify that this choice does not affect model rankings between raw and transformed targets across 61 regression tasks (Kendall-$\tau = 0.95$, \autoref{fig:raw_vs_transformed_label_rankings}).

\paragraph{Statistical significance tests}\label{app:stat_tests}

To assess the statistical significance of the performance differences between the $k$ evaluated pipelines across $N$ datasets, we first employ the \textbf{Friedman test} \citep{10.1214/aoms/1177731944}, a non-parametric equivalent of ANOVA for repeated measures. The null hypothesis states that all algorithms perform equivalently in terms of average rank.

Upon rejection of the null hypothesis ($p < 0.05$), we employ the \textbf{Conover-Iman post-hoc test} \citep{conover1979multiple} for pairwise comparisons. The test statistic for comparing two algorithms $i$ and $j$ is given by:

\begin{equation}
    T = \frac{|\bar{R}_i - \bar{R}_j|}{\sqrt{\hat{S}^2 \left( \frac{2}{N} \right)}}
\end{equation}

where $\bar{R}_i$ and $\bar{R}_j$ are the average ranks of the algorithms, and $\hat{S}^2$ is the pooled sample variance of the ranks. The difference between two algorithms is considered statistically significant if the p-value derived from the $t$-distribution is below the significance level $\alpha=0.05$.

\begin{table}[!h]
\small
\centering
\caption{Hyperparameter search space for STRABLE learners.}
\label{tab:strable_hyperparameter_space}
\begin{tabular}{lll}
\toprule
\multicolumn{1}{c}{\textbf{Methods}} & \multicolumn{1}{c}{\textbf{Parameters}} & \multicolumn{1}{c}{\textbf{Grid}} \\
\midrule
TabPFN-2.5 & - & Default parameters \\
\midrule
TabStar & - & Default parameters \\
\midrule
ConTextTab & - & Default parameters \\
\midrule
Ridge Regression & Alpha $(\alpha)$ & $[0.01, 0.1, 1, 10, 100]$ \\
\midrule
\multirow{9}{*}{XGBoost} & Max depth & UniformInt [$2, 6$] \\
& Min child weight & LogUniform [$1, 100$] \\
& Subsample & Uniform [$0.5, 1$] \\
& Learning rate & LogUniform [$10^{-5}, 1$] \\
& Colsample by level & Uniform [$0.5, 1$] \\
& Colsample by tree & Uniform [$0.5, 1$] \\
& Gamma & LogUniform [$10^{-8}, 7$] \\
& L2 regularization ($\lambda$) & LogUniform [$1, 4$] \\
& L1 regularization ($\alpha$) & LogUniform [$10^{-8}, 100$] \\
\midrule
\multirow{7}{*}{CatBoost} & Max depth & UniformInt [$2, 6$] \\
& Learning rate & LogUniform [$10^{-5}, 1$] \\
& Bagging temperature & Uniform [$0, 1$] \\
& $l_2$-leaf regularization & LogUniform [$1, 10$] \\
& Random strength & UniformInt [$1, 20$] \\
& One hot max size & UniformInt [$0, 25$] \\
& Leaf estimation iterations & UniformInt [$1, 20$] \\
\midrule
\multirow{3}{*}{ExtraTrees} & Max features & \{sqrt, $0.5, 0.75, 1.0$\} \\
& Min samples split & LogUniformInt [$2, 32$] \\
& Min impurity decrease & Choice \{$0, 10^{-5}, 3\cdot 10^{-5}, 10^{-4}, 3\cdot 10^{-4}, 10^{-3}$\} \\
\bottomrule
\end{tabular}
\end{table}

\paragraph{Computational resources} 

The experimental evaluations were run on a CPU of 32 cores with an additional GPU for models that require GPU computations. %
Jobs were dynamically scheduled across nodes, so no pipeline ran exclusively on a given accelerator; runtime differences between model families span orders of magnitude (log scale in \autoref{fig:combined_pareto_oracle}a), so hardware heterogeneity is negligible relative to algorithmic differences.

\begin{description}[itemsep=0.5pt, leftmargin = 1.25cm]
    \item \textbf{GPUs}: NVIDIA V100 (32GB VRAM), A100 (40GB VRAM), A40 (48GB VRAM)
    \item \textbf{CPUs}: AMD EPYC 7742 64-Core Processor, AMD EPYC 7702 64-Core Processor (512GB RAM), Intel(R) Xeon(R) CPU E5-2660 v2, Intel(R) Xeon(R) Gold 6226R CPU (256GB RAM)
\end{description}

The total CPU and GPU hours for the entire STRABLE experiment is of 842 days.

\begin{table}[H]
\footnotesize
\centering
\caption{Summary of Encoders and Architectures used in the STRABLE benchmark. Embedding dimensions, parameters and context length are taken from the English Massive Text Embedding Benchmark Leaderboard \citep{enevoldsen2025mmtebmassivemultilingualtext}}
\label{tab:llm_summary}
\begin{tabular}{llccc}
\toprule
\textbf{Category} & \textbf{Model Name} & \textbf{Dim ($d$)} & \textbf{Params} & \textbf{Context} \\
\midrule
\multirow{4}{*}{Statistical Baselines} 
& StringEncoder (Tf-Idf + SVD) & - & N/A & N/A \\
& TargetEncoder & - & N/A & N/A \\
& CatBoostEncoder & - & N/A & N/A \\
& FastText & 300 & N/A & N/A \\
\midrule
\multirow{5}{*}{Embedders (Contrastive)} 
& E5 (Small/Base/Large) & 384--1024 & 33M--335M & 512 \\
& BGE (Small/Base/Large) & 384--1024 & 0.1B--33.5B & 512 \\
& UAE-Large & 1024 & 335M & 512 \\
& All-MiniLM (L6/L12) & 384 & 22M--33.4M & 256 \\
& All-MPNet-Base-v2 & 768 & 110M & 384 \\
& KALM (Embed) & 896 & 0.5B & 512 \\
& Tarte & 768 & 25M & N/A \\
\midrule
\multirow{3}{*}{Encoder-only (MLM)} 
& RoBERTa (Base/Large) & 768--1024 & 125M--355M & 512 \\
& DeBERTa-v3 (XS/S/B/L) & 384--1024 & 22M--304M & 512 \\
& ModernBERT (Base/Large) & 768--1024 & 149M--395M & 8192 \\
\midrule
\multirow{1}{*}{Encoder-Decoder} 
& Sentence-T5 (Base/L/XL/XXL) & 768 & 0.5B--5B & 512 \\
\midrule
\multirow{7}{*}{Decoder-only (Causal)} 
& LLaMA-3.1 / 3.2 & 2048--4096 & 1B--8B & 128k \\
& LLaMA-Nemotron-Embed-1B-v2 & 2048 & 1B & 128k \\
& Qwen-3 (0.6B/4B/8B) & 1024--4096 & 0.6B--8B & 32k \\
& OPT (0.1B to 6.7B) & 768--4096 & 125M--6.7B & 2048 \\
& Gemma-0.3B & 768 & 300M & 128k \\
& F2LLM (0.6B/1.7B/4B) & 1024--2560 & 0.6B--4B & 1024 \\
& Jasper-0.6B & 2048 & 0.6B & 2048 \\
\midrule
\multirow{2}{*}{End-to-End Architectures} 
& ConTextTab & 768 & 172M & 256 \\
& TabSTAR & 384 & 47.2M & 512 \\
\bottomrule
\end{tabular}
\end{table}

\clearpage

\section{Extended results}
\begin{tcolorbox}[colback=gray!10, colframe=gray!50, arc=2pt, boxrule=1pt, left=4pt, right=4pt, top=4pt, bottom=4pt]
\textbf{Take-aways on benchmarking tabular learning with strings:}
\begin{itemize}[leftmargin=1ex]
    \item Strings carry signal that complements numerical features: every learner benefits from including string columns
    \item Real-world string columns are mostly short and repetitive (median 17 characters), not free text
    \item Modular pipelines outperform end-to-end string-tabular architectures 
    \item Larger LLMs help only when paired with weak learners, or with appropriate post-processing of their embeddings, but they are not Pareto optimal (encoders are responsible for runtime)
    \item Decoder-only LLM embeddings need standard scaling or Matryoshka-style direct slicing rather than default PCA
    \item 30 principal components suffice; higher dimensions hurt performance and inflate runtime
    \item String length (avg. words per cell) is the dominant driver of ranking shifts; large LLMs only enter the top-10 on free-text-dominant tables
    \item STRABLE's 108 datasets yield rankings close to the oracle ($\tau \approx 0.93$), stable across application fields and data-preparation choices
\end{itemize}
\end{tcolorbox}

\begin{figure}[!ht]
    \centering
    \includegraphics[width=0.8\linewidth]{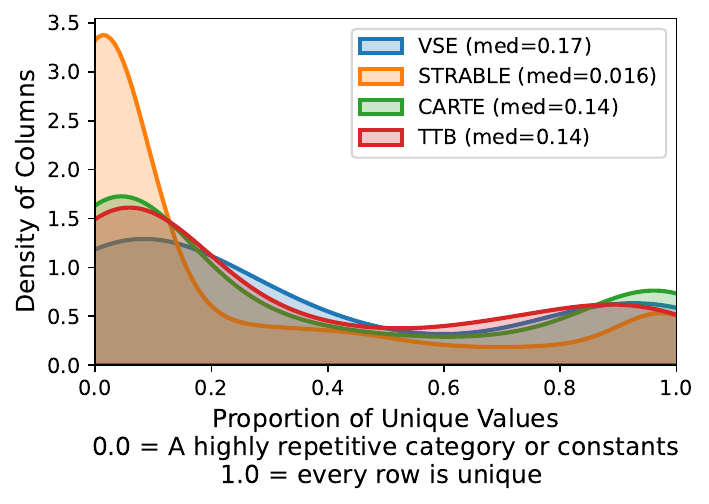}
    \caption{VSE stands for the datasets used in \citealp{grinsztajn2023vectorizingstringentriesdata}, CARTE represents the datasets used in \citealp{kim2024cartepretrainingtransfertabular} and TTB represents the datasets used in TextTabBench \cite{mraz2025benchmarkingfoundationmodelstabular}. STRABLE refers to the datasets introduced in this work. \textbf{STRABLE exhibits the lowest median proportion of unique text entries (0.016)}, compared to TTB (0.137), CARTE (0.136) and VSE (0.17).}
    \label{fig:uniqueness_ratio_comparison_4_benchmarks}

\begin{itemize}
    \setlength{\itemsep}{0pt} 
    \item \textbf{Avg Tokens / Cell}: average number of whitespace-separated tokens per cell within a sample of 1000 rows.
    \item \textbf{Avg Char / Cell}: average number of characters per cell for a sample of 1000 rows. 
    \item \textbf{Avg Unique Alphabetic Words / Cell}: average number of unique, case-insensitive alphabetic sequences (length $\geq$ 2) per cell, computed over a random sample of 1000 rows. This excludes numbers, punctuation, and single-letter characters. 
    \item \textbf{Avg Unique N-grams / Cell}: average number of unique character n-grams (length 2--4) per cell within a sample of 1000 rows. 
    \item \textbf{Proportion of Unique Values}: the ratio of unique values to total values in a column, indicating how repetitive vs. unique the entries are. \textbf{0.0} means all values are the same (e.g., a column of constants), while \textbf{1.0} means every value is unique (e.g., unique IDs or rich text). 
    \item \textbf{Text Col Ratio}: proportion of columns that are text. 
\end{itemize}

\textbf{STRABLE shows the lowest metrics.} In this regime, simple frequency-based methods such as Tf-Idf are well-suited and sufficient, as the discriminative signal is concentrated in a small set of recurring tokens rather than in semantic context — which explains the strong performance of Tf-Idf on the Pareto plot.
    \label{tab:four_benchmark_comparison}
    \captionof{table}[]{Median structural characteristics of text columns across the four benchmarks.}
    \small %
    \setlength{\tabcolsep}{4pt}
    \begin{tabular}{lcccccc}
    \toprule
     & \shortstack{Avg tokens\\/ Cell} & \shortstack{Avg char\\/ Cell} & \shortstack{Avg unique\\words / cell} & \shortstack{Avg unique\\n-grams / cell} & \shortstack{Prop. unique\\text entries} & \shortstack{Text col\\ratio} \\
    \midrule
    \textbf{CARTE} & 1.78 & 13.54 & 2.00 & 34.28 & 0.14 & 0.29 \\
    \textbf{STRABLE} & \textbf{1.31} & \textbf{11.07} & \textbf{1.35} & \textbf{26.81} & \textbf{0.02} & \textbf{0.59} \\
    \textbf{TTB} & 2.89 & 21.26 & 2.92 & 55.57 & 0.14 & 0.10 \\
    \textbf{VSE} & 1.56 & 13.89 & 1.98 & 35.33 & 0.17 & 0.33 \\
    \bottomrule
    \end{tabular}
\end{figure}

\begin{figure}[t!]
    \centering
    \includegraphics[width=\linewidth]{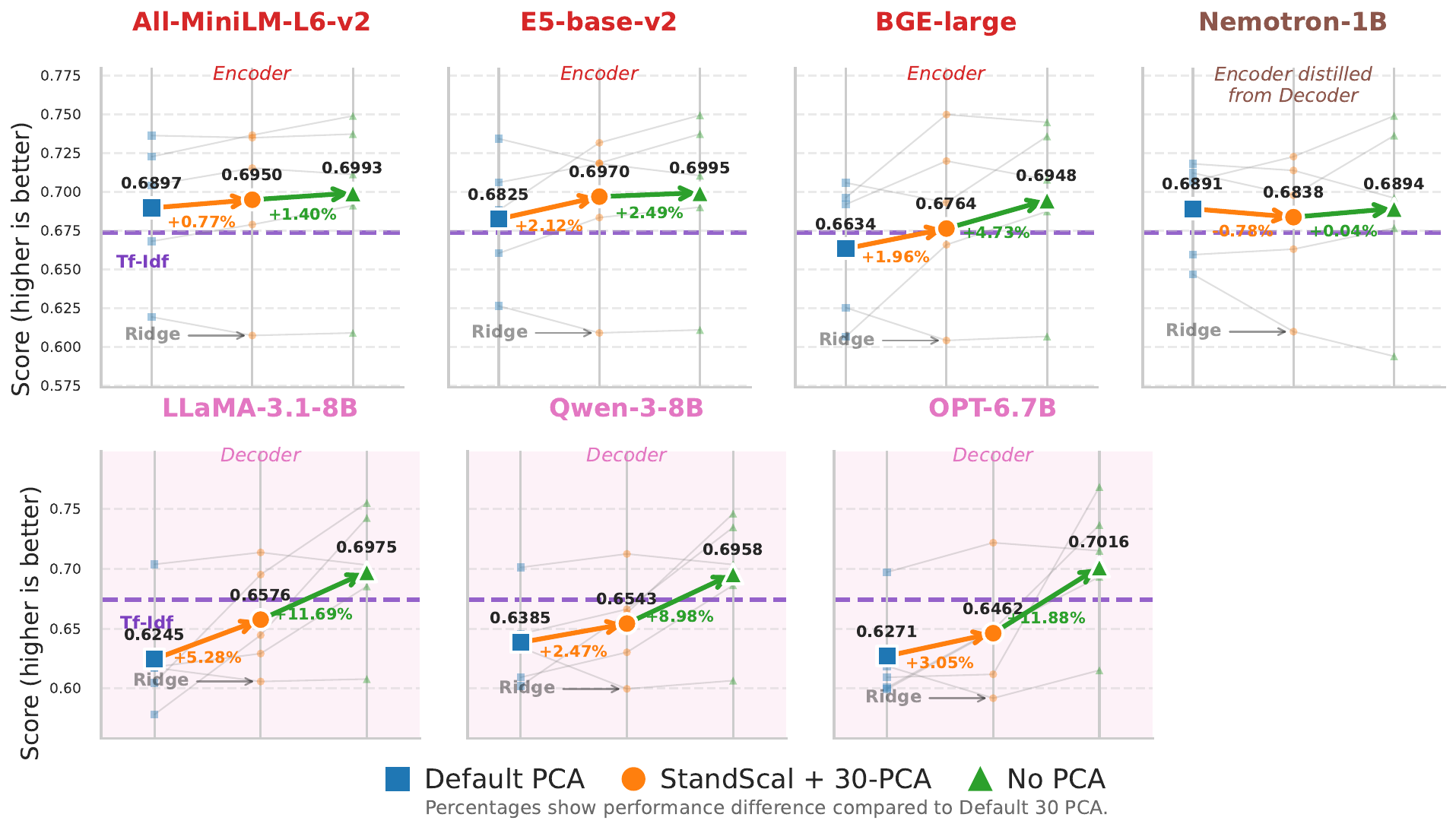}
    \caption{\textbf{Per-learner breakdown of post-processing strategies across LM encoders.} 
Same setup as \autoref{fig:pca_postprocessing_comparison}, but with each individual learner shown as a faint background line (Ridge, XGBoost, ExtraTrees, TabPFN-2.5, TabICLv2) in addition to the across-learner mean (foreground). Top row: encoder-only models and the distilled Nemotron-1B. Bottom row: decoder-only models. The post-processing trend is consistent across learners --- removing PCA helps decoders and is roughly neutral for encoders --- with Ridge (annotated by arrows) consistently the weakest learner across all encoders and post-processing variants. The spread of background lines indicates that the choice of learner has a smaller effect on score than the choice of post-processing for decoder-only models, while the opposite holds for encoder-only models.}
    \label{fig:pca_postprocessing_comparison_appendix}
\end{figure}

\begin{figure}[!ht]
    \centering
    \begin{subfigure}[b]{0.49\linewidth}
        \centering
        \includegraphics[width=\linewidth]{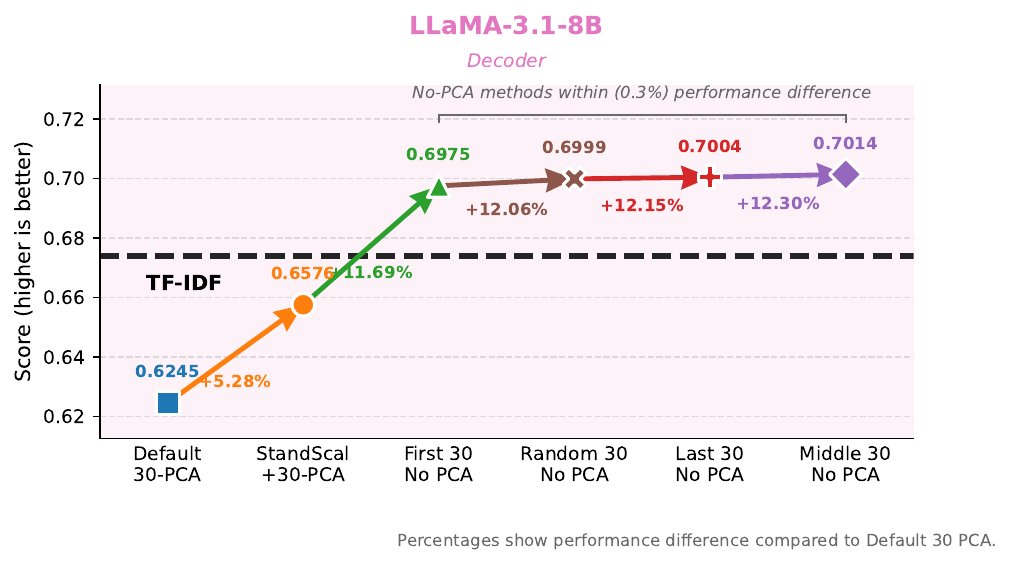}
        \caption{Post-processing methods for LlaMA-3.1-8B}
        \label{fig:pp_llama_6methods}
    \end{subfigure}
    \hfill
    \begin{subfigure}[b]{0.49\linewidth}
        \centering
        \includegraphics[width=\linewidth]{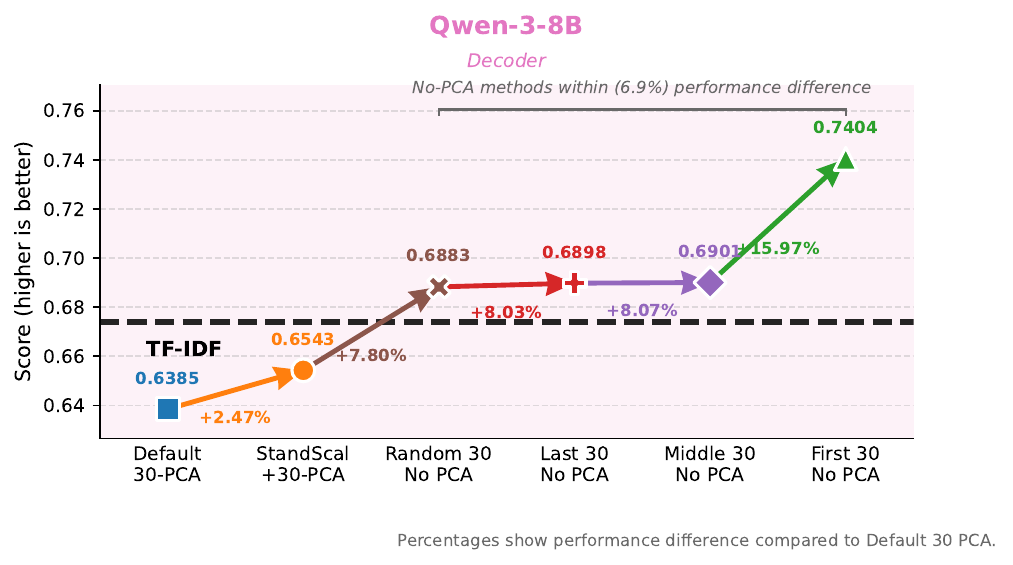}
        \caption{Post-processing methods for Qwen-3-8B}
        \label{fig:pp_qwen_6methods}
    \end{subfigure}

    \caption{\textbf{Effect of embedding post-processing on LLaMA-3.1-8B and Qwen3-Embedding-8B.}
Scores are averages of $R^2$ (regression) and AUC (classification) across all STRABLE datasets,
over five learners (Ridge, XGBoost, ExtraTrees, TabPFN-2.5, TabICL).
The dashed line marks the TF-IDF baseline.
Percentages are relative improvements over Default 30-PCA.
\textbf{Left (LLaMA-3.1-8B, decoder-only):} all four No-PCA slicing strategies
(First/Last/Middle/Random 30) cluster within 0.3\% of each other and consistently
outperform both PCA variants by $\sim$12\%, confirming that
the choice of \emph{which} 30 dimensions to retain is immaterial for this model.
\textbf{Right (Qwen3-Embedding-8B, MRL-trained):} the four No-PCA strategies spread over 6.9\%,
with First-30 yielding a $+$15.97\% gain---reflecting the model's MRL training, which
explicitly front-loads semantic information into the leading dimensions. Qwen3-Embedding-8B's Matryoshka training supports a minimum of 32 dimensions; retaining 32 instead of 30 shifts the median score of Qwen-3-8B + TabPFN-2.5 from $0.786$ to $0.788$, with 104 of 108 tables within $0.01$.
\emph{Note:} these results use unoptimized LLM inference.}
    \label{fig:llama_qwen_pp_6methods}
\end{figure}

\begin{table}[ht!]
\centering
\caption{Per-dimension variance concentration (Gini coefficient \citep{RePEc:tpr:restat:v:61:y:1979:i:1:p:146-49}) of embeddings, computed per dataset and aggregated across 108 datasets. \textbf{Higher Gini indicates that variance is more concentrated in a small subset of dimensions.}}
\label{tab:gini_appendix}
\begin{tabular}{lcc}
\toprule
Model & Median Gini & Mean Gini \\
\midrule
MiniLM-L6-v2 (encoder) & 0.152 & 0.164 \\
E5-base-v2 (encoder)   & 0.113 & 0.126 \\
BGE-large (encoder)    & 0.123 & 0.130 \\
\midrule
Qwen3-8B (decoder, Matryoshka) & \textbf{0.236} & \textbf{0.242} \\
OPT-6.7B (decoder)             & \textbf{0.362} & \textbf{0.357} \\
LLaMA-3.1-8B (decoder)         & \textbf{0.422} & \textbf{0.407} \\
\bottomrule
\end{tabular}
\end{table}

\begin{table}[ht!]
\centering
\caption{\textbf{Pairwise comparisons of Gini coefficients across 108 STRABLE datasets.} Paired Wilcoxon signed-rank test \citep{c4091bd3-d888-3152-8886-c284bf66a93a} on Gini values per dataset, with $p$-values Holm-corrected for 15 pairwise comparisons \citep{Haynes2013}. $H_{0}:$ the paired differences in Gini Coefficients are different from zero. All decoder-vs-encoder comparisons are significant at $p < 10^{-17}$. The only marginal pair is E5-base-v2 vs BGE-large, two encoders with very similar Gini distributions.}
\label{tab:gini_pairwise}
\begin{tabular}{llc}
\toprule
Model A & Model B & Holm-corrected $p$ \\
\midrule
LLaMA-3.1-8B & MiniLM-L6-v2 & $2.8 \times 10^{-18}$ \\
LLaMA-3.1-8B & E5-base-v2   & $2.8 \times 10^{-18}$ \\
LLaMA-3.1-8B & BGE-large    & $2.8 \times 10^{-18}$ \\
Qwen3-8B     & MiniLM-L6-v2 & $2.8 \times 10^{-18}$ \\
Qwen3-8B     & E5-base-v2   & $2.8 \times 10^{-18}$ \\
Qwen3-8B     & BGE-large    & $2.8 \times 10^{-18}$ \\
OPT-6.7B     & MiniLM-L6-v2 & $2.8 \times 10^{-18}$ \\
OPT-6.7B     & E5-base-v2   & $2.8 \times 10^{-18}$ \\
OPT-6.7B     & BGE-large    & $2.8 \times 10^{-18}$ \\
\bottomrule
\end{tabular}
\end{table}

\begin{figure}[ht!]
    \centering
    \includegraphics[width=0.55\linewidth]{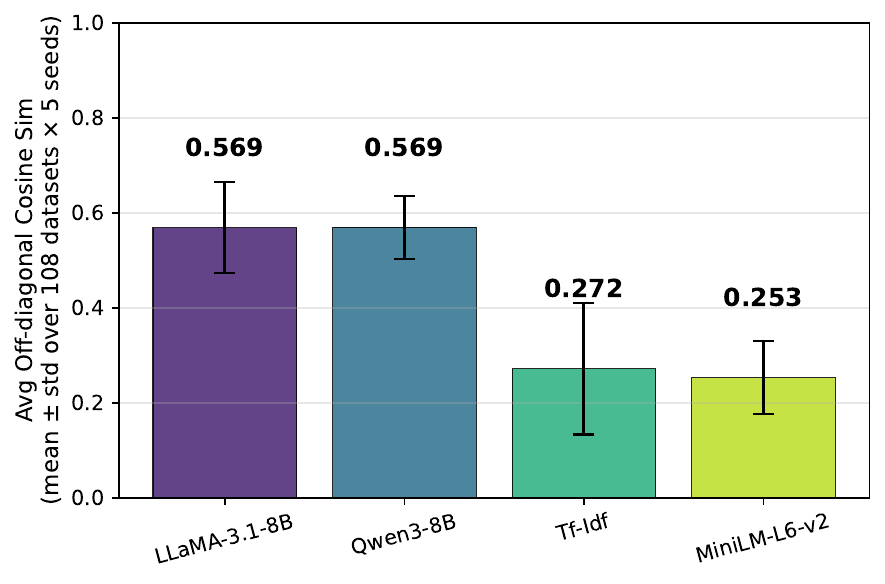}
    \caption{\textbf{Decoder-only models LLaMA-3.1-8B and Qwen-3-8B exhibit notably higher average similarity ($\approx$0.57) than encoder models such as MiniLM-L6-v2 ($\approx$0.25), indicating that unrelated strings tend to receive similar embeddings.} Average off-diagonal cosine similarity of raw string representations across the 108 STRABLE datasets (5 seeds each). Tf-Idf is shown as a baseline.}
    \label{fig:cosine_before}
\end{figure}

\begin{figure}[!ht]
    \centering
    \begin{subfigure}[b]{0.49\linewidth}
        \centering
        \includegraphics[width=\linewidth]{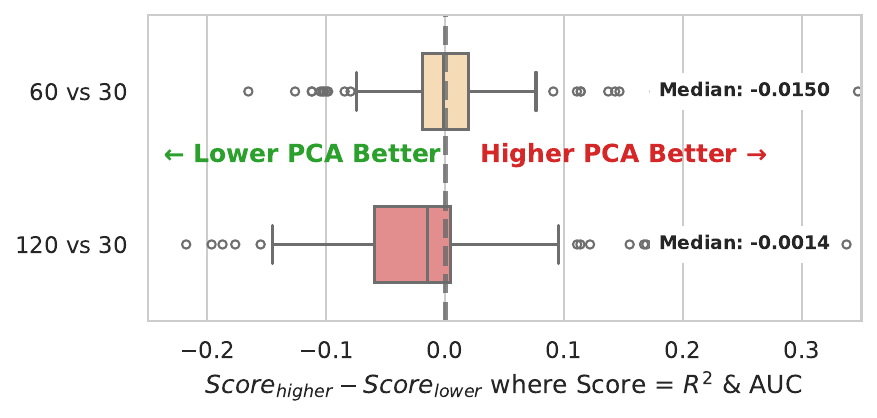}
        \caption{Score deltas}
        \label{fig:pca_score_delta}
    \end{subfigure}
    \hfill
    \begin{subfigure}[b]{0.49\linewidth}
        \centering
        \includegraphics[width=\linewidth]{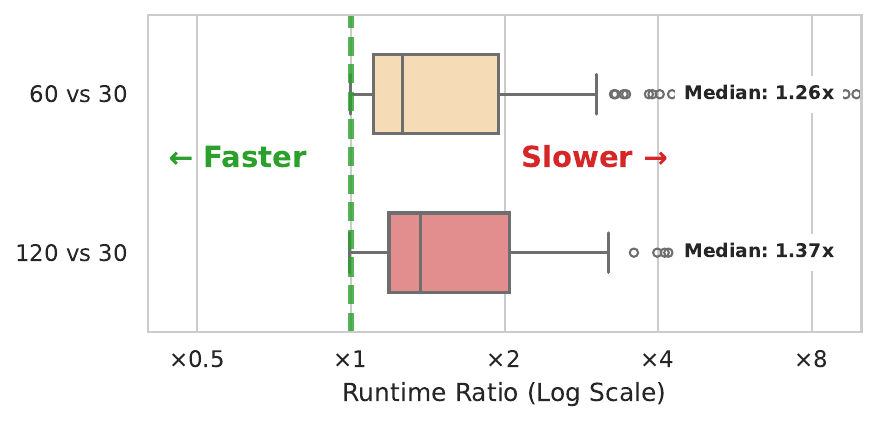}
        \caption{Runtime ratios (log scale)}
        \label{fig:pca_runtime_ratio}
    \end{subfigure}

    \caption{\textbf{Performance difference between LLaMA-3.1-8B + TabPFN-2.5 with PCA at 30, 60, and 120 dimensions.} \textbf{Left:} distribution of per-dataset score deltas (Score\textsubscript{higher} $-$ Score\textsubscript{lower}, where Score = $R^2$ \& AUC). A negative median indicates that the higher-dimensional PCA hurts performance. \textbf{Right:} runtime ratio on a log scale; values $>1$ indicate slower execution. Both comparisons use the 30-component setting as baseline. \textbf{30 components sit at a shallow optimum:} a 98\%-explained-variance criterion retains a median of $\sim$200 components and drops the median score from $0.610$ to $0.243$, while 15 components lower it from $0.698$ to $0.669$ (15-PCA matches or exceeds 30-PCA on only 40/108 tables).}
    \label{fig:llama_tabpfn_pca30_vs_60_vs_120_delta}
\end{figure}

\begin{table}[!ht]
\centering
\caption{Comparison of score variations and runtime. The median score deltas ($\Delta$) indicate a performance drop when shifting from 30 to 60 or 120 components, accompanied by a consistent increase in processing time.}
\label{tab:score_runtime_comparison}
\begin{tabular}{|l|c|c|}
\hline
\textbf{Comparison} & \textbf{Score $\Delta$ (Median)} & \textbf{Runtime} \\
\hline
60 vs 30  & -0.0105 & 1.26x slower  \\
\hline
120 vs 30 & -0.0014 & 1.37x slower \\
\hline
\end{tabular}
\end{table}

\begin{figure}[!ht]
    \centering
    \begin{subfigure}{\linewidth}
        \centering
        \includegraphics[width=0.7\linewidth]{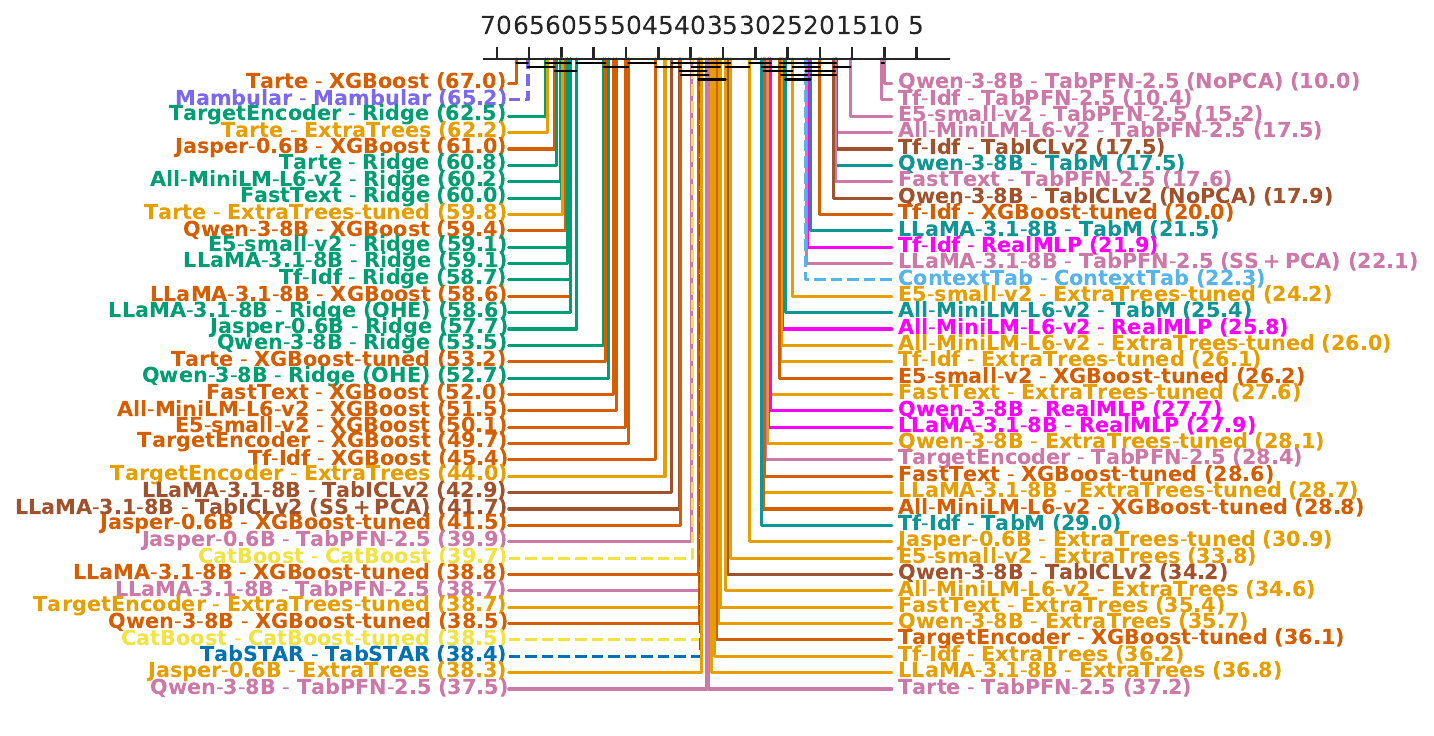}
        \caption{CD Diagram for all tasks}
        \label{fig:cd_diagram_all}
    \end{subfigure}

    \vspace{1.5em} %

    \begin{subfigure}{\linewidth}
        \centering
        \includegraphics[width=0.7\linewidth]{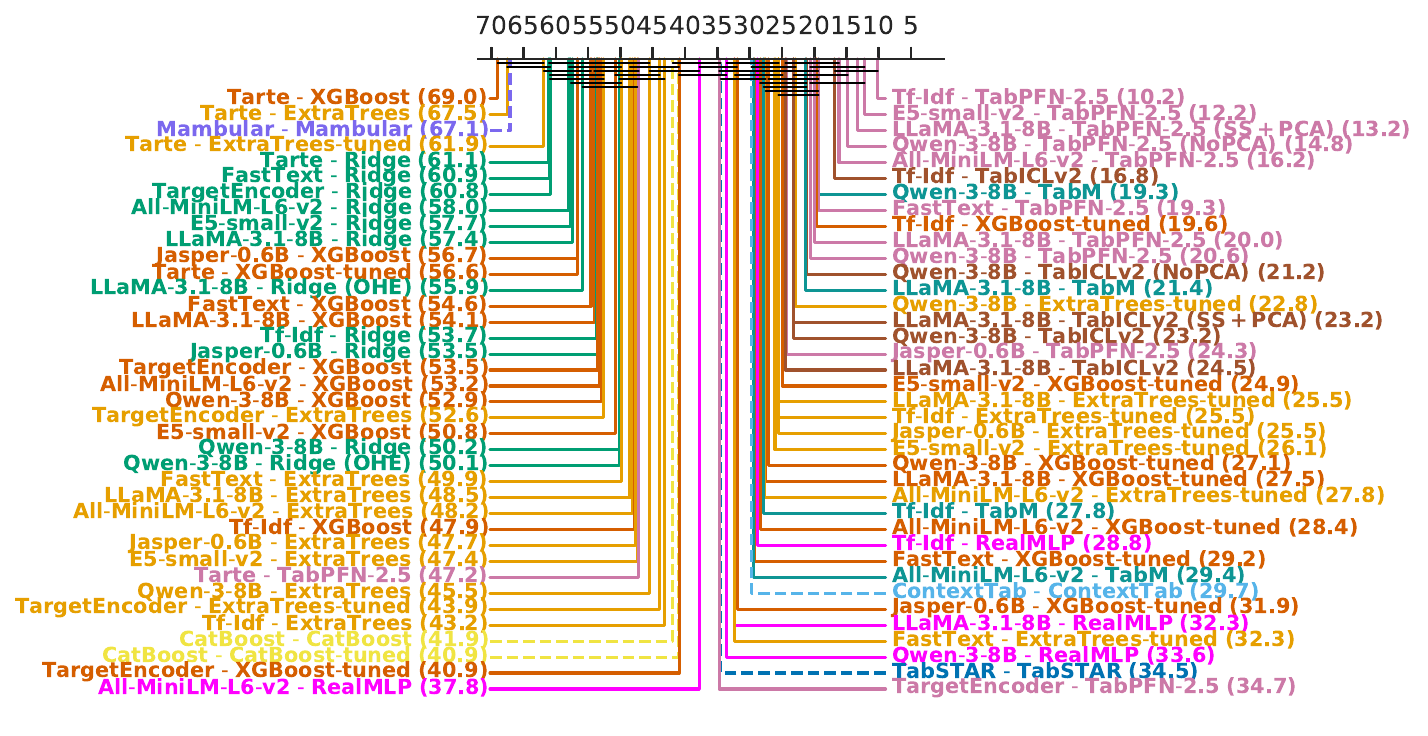}
        \caption{CD Diagram for Classification tasks}
        \label{fig:cd_diagram_cl}
    \end{subfigure}

    \vspace{1.5em}

    \begin{subfigure}{\linewidth}
        \centering
        \includegraphics[width=0.7\linewidth]{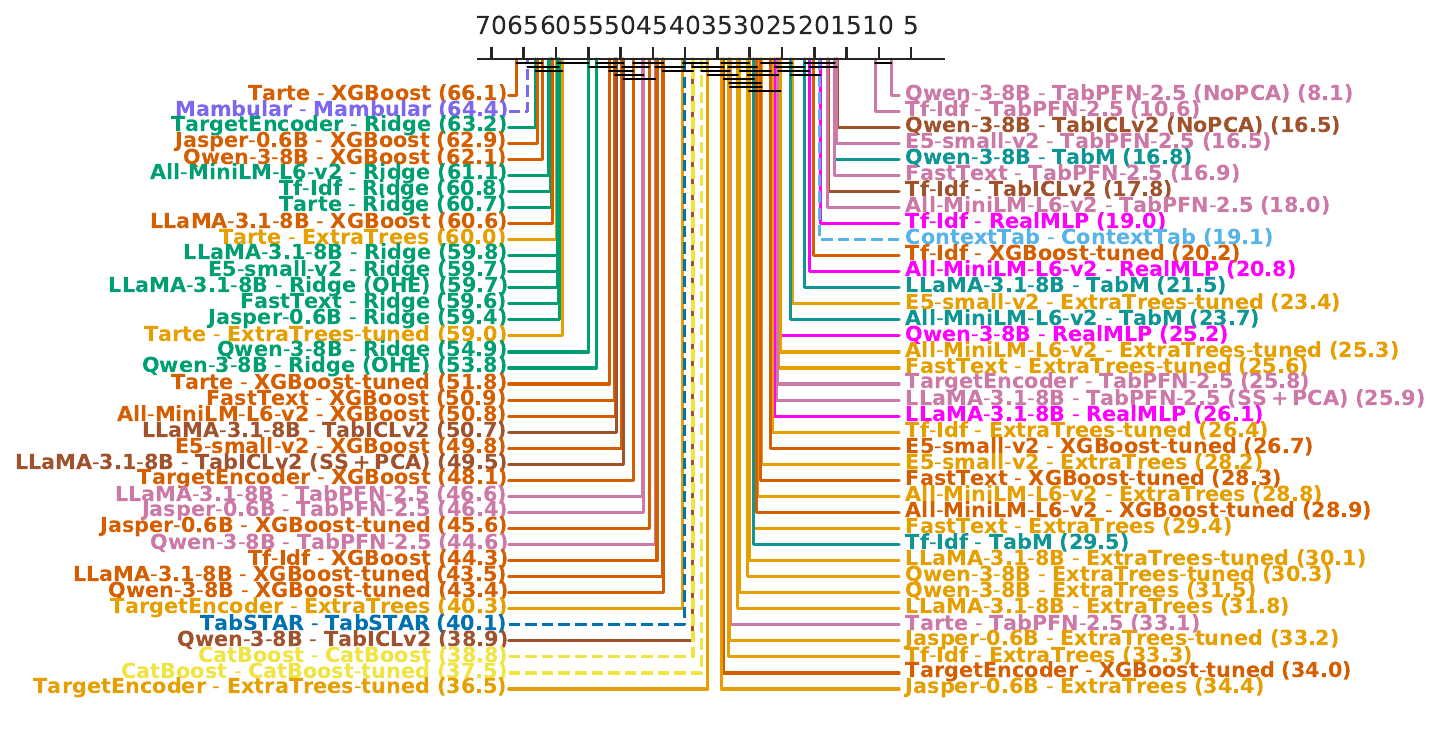}
        \caption{CD Diagram for Regression tasks}
        \label{fig:cd_diagram_re}
    \end{subfigure}

    \caption{\textbf{Critical Difference (CD) Diagrams across different task types.} Comparison of all pipelines using the Friedman test. The diagrams show mean ranks and groups of models that are not significantly different (connected by horizontal bars).}
    \label{fig:combined_cd_diagrams}
\end{figure}

\begin{figure}[!ht]
    \centering
    \includegraphics[width=0.7\linewidth]{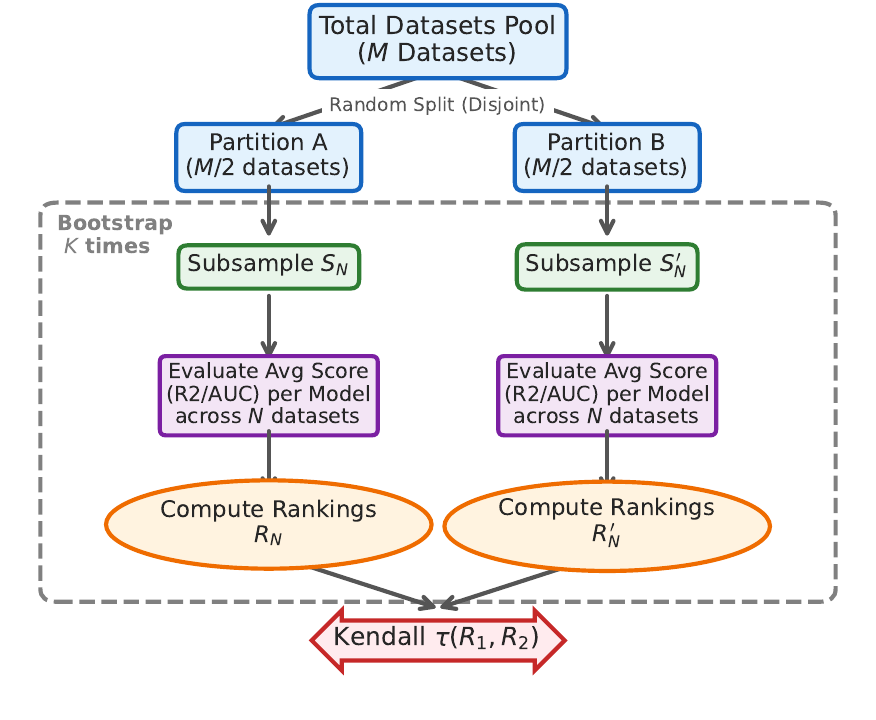}
    \caption{Sampling Diagram to produce the boostrap estimator of the Kendall-$\tau$ correlation between independent benchmarks of size N}
    \label{fig:sampling_diagram_kendalltau}
\end{figure}

\begin{figure*}[!ht]
    \centering
    \makebox[\textwidth]
    {\includegraphics[width=0.6\textwidth]{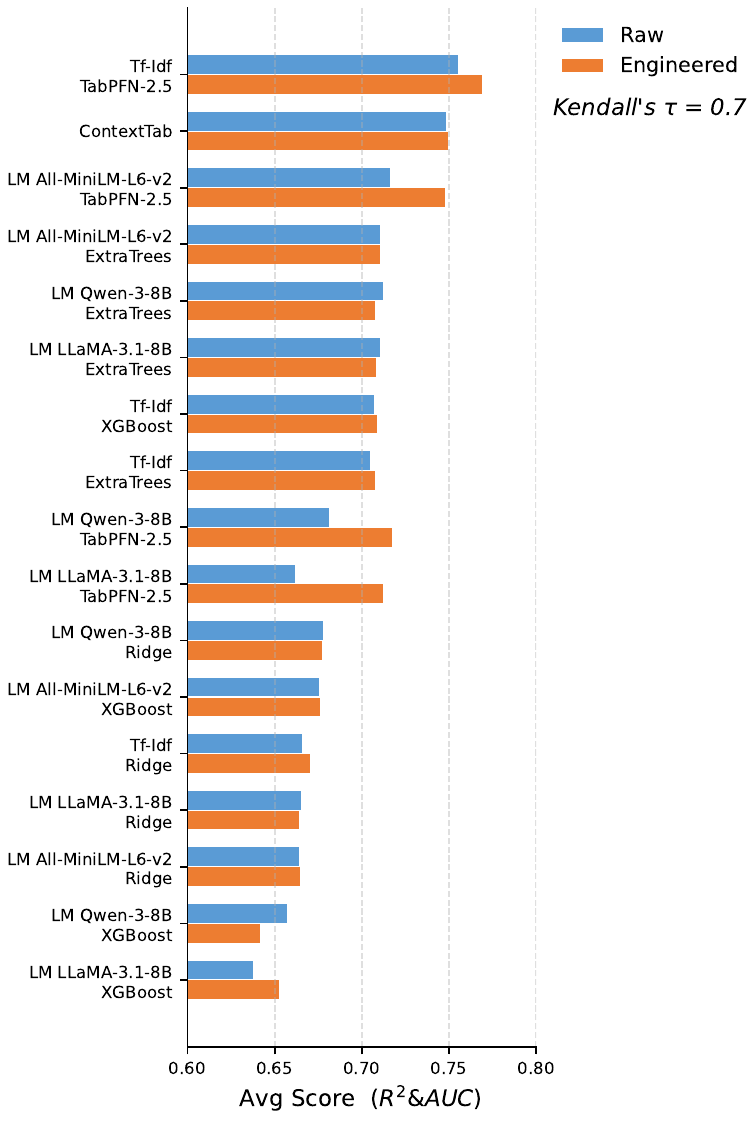}}%
    \caption{\textbf{Raw vs.\ engineered features.} Average score for three representative pipelines with and without manual feature engineering applied to 44/108 tables (date parsing, ordinal encoding, range 
    extraction, coordinate extraction, drug-strength parsing, fiscal year/quarter parsing). The Kendall-$\tau$ between the two rankings is reported.}
    \label{fig:raw_vs_engineered}
\end{figure*}

\begin{figure*}[!ht]
    \centering
    \makebox[\textwidth]
    {\includegraphics[width=0.9\textwidth]{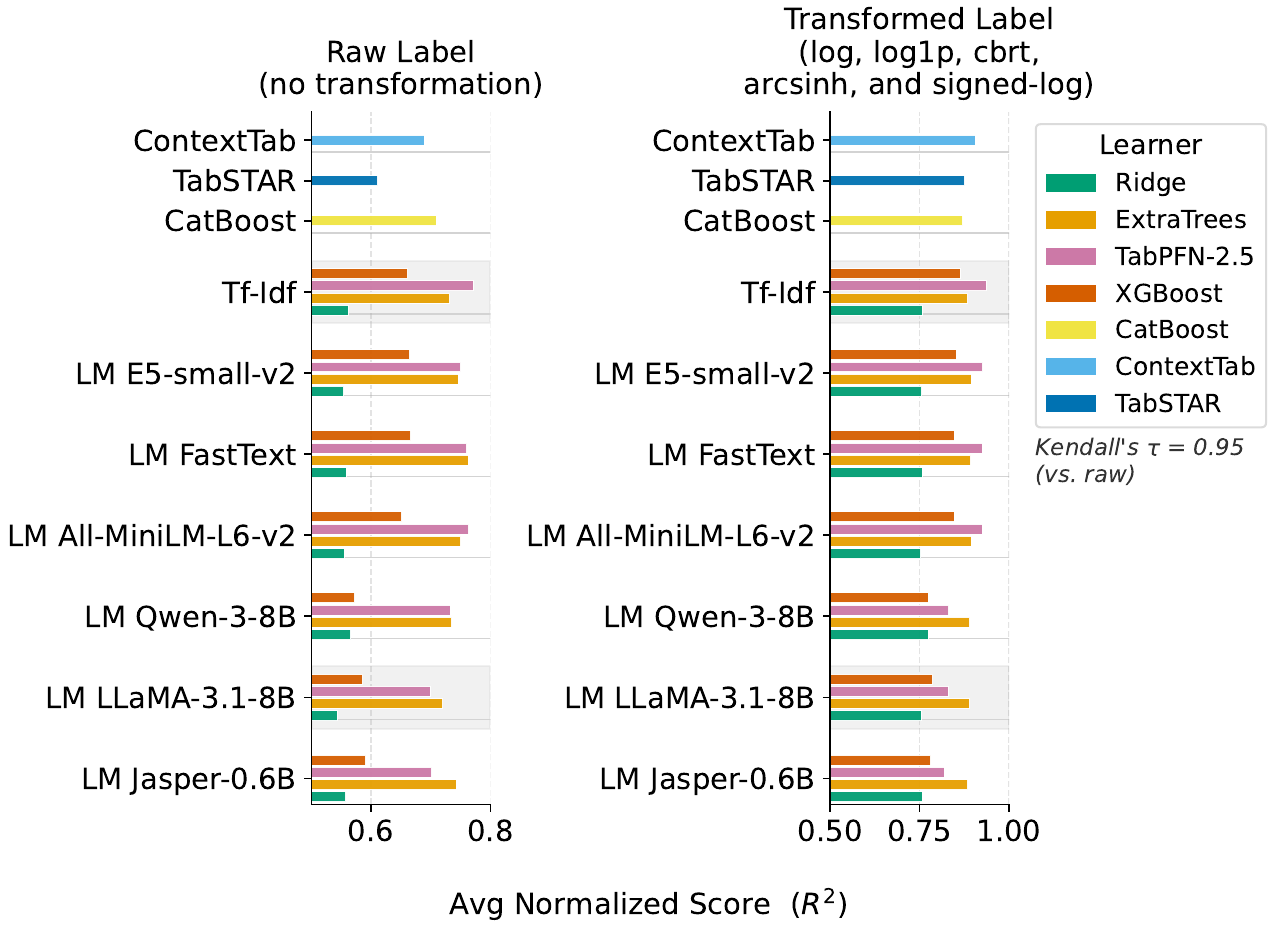}}%
    \caption{\textbf{Impact of Label Transformation on Model Rankings}: We evaluated 31 pipelines (28 modular and 3 end-to-end) across 61 regression tasks. The bar charts compare average normalized $R^2$ scores using raw targets versus transformed targets. While models perform worse on raw targets due to unmitigated skewness and outliers, the relative performance and ranking of the encoders remain consistent across both settings.}
    \label{fig:raw_vs_transformed_label_rankings}
\end{figure*}

\begin{figure*}[!ht]
    \centering
    \makebox[\textwidth]
    {\includegraphics[width=0.5\textwidth]{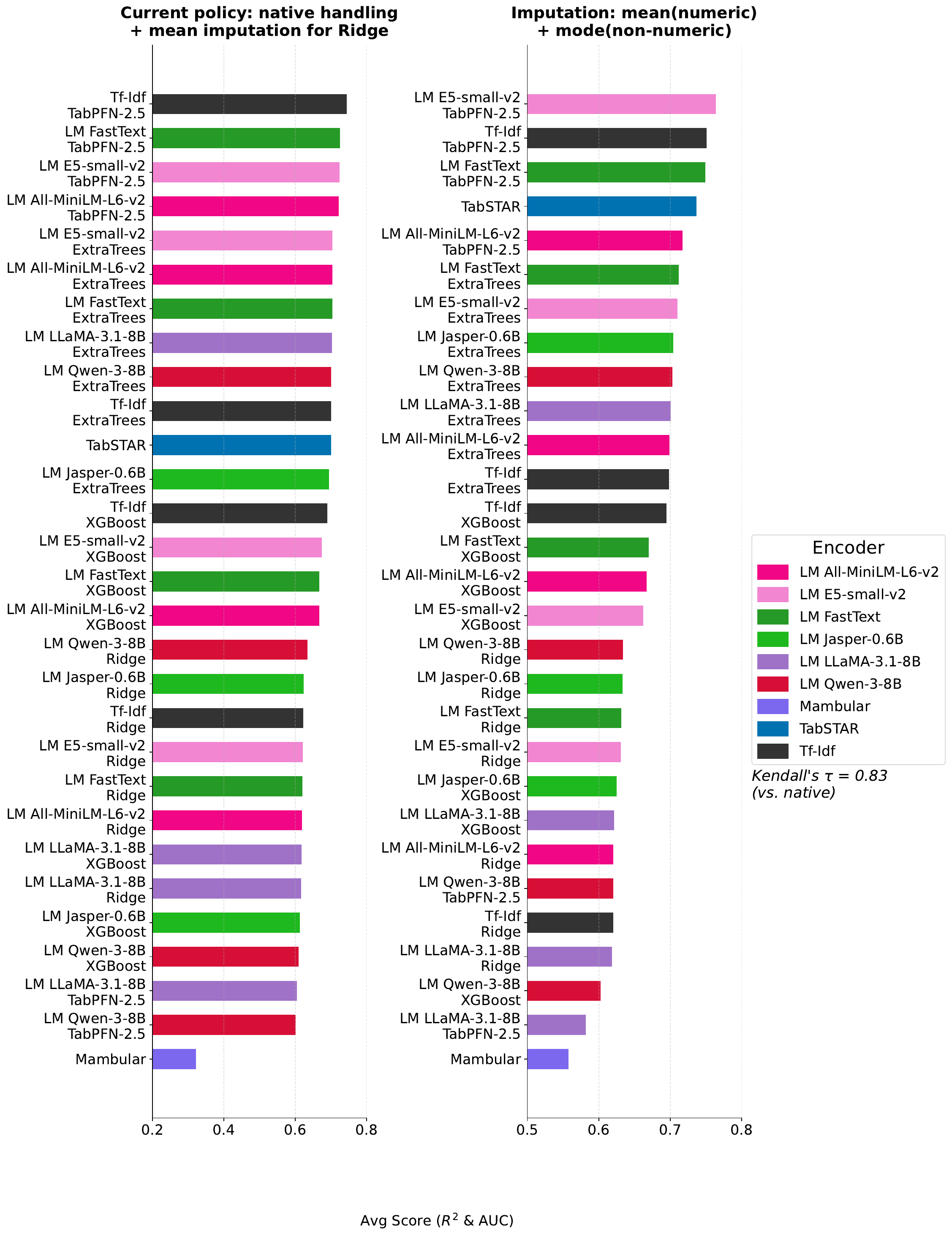}}%
    \caption{\textbf{Impact of Missing Values imputation on Model Rankings.} Comparison of pipeline rankings under two missing-value strategies: native handling per learner with mean imputation for Ridge only (left), and mean/mode imputation applied to all pipelines before encoding (right). Bars show average score ($R^2$ \& AUC) across 108 datasets, color-coded by encoder. Kendall's $\tau = 0.83$ between the two rankings, indicating that pipeline ordering is largely preserved regardless of imputation strategy.}
    \label{fig:missing_values_handling_none}
\end{figure*}

\begin{figure*}[!ht]
    \centering
    \makebox[\textwidth]{\includegraphics[width=0.6\textwidth]{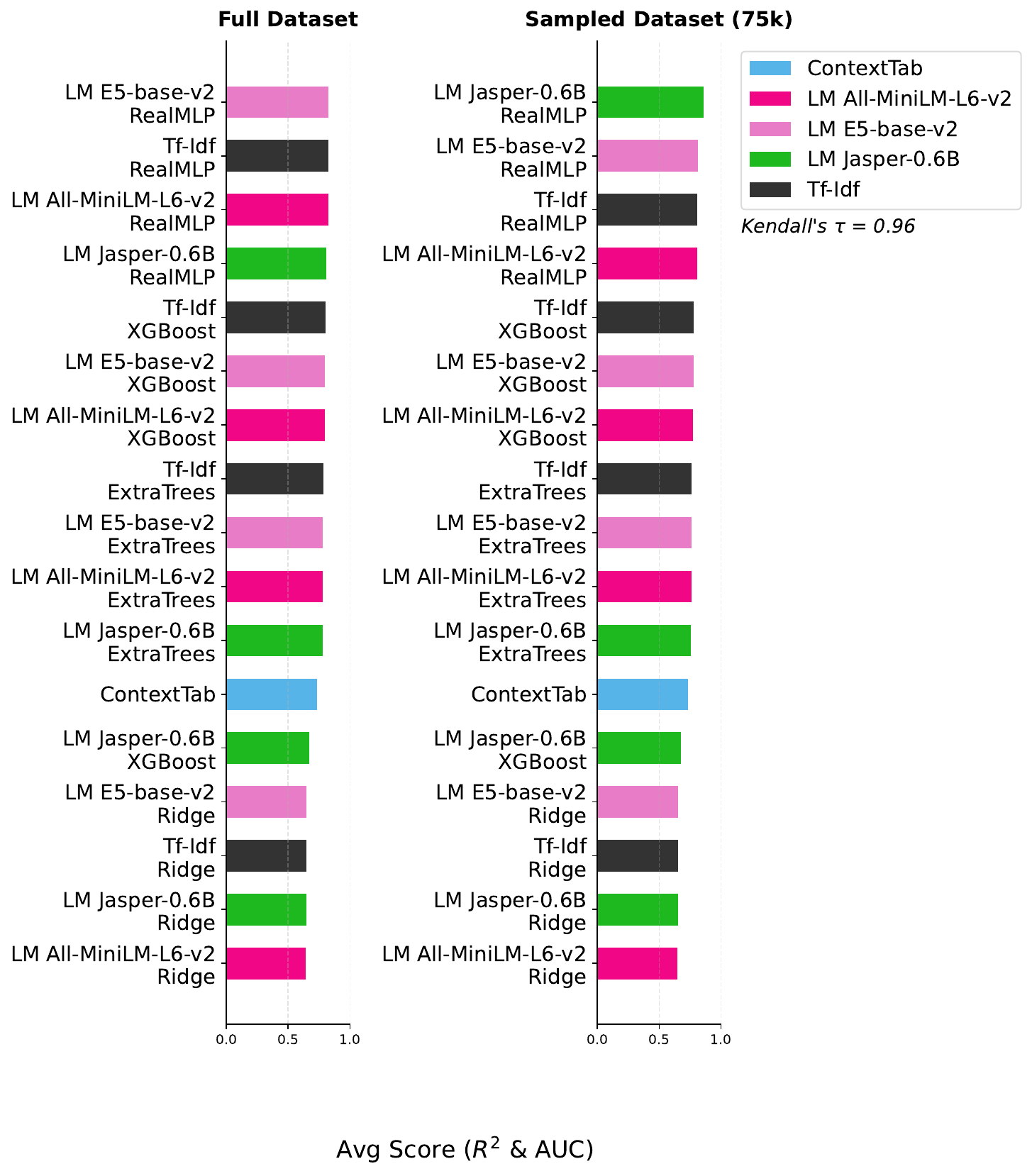}}
    \caption{\textbf{Pipeline rankings on full versus subsampled datasets ($n=75{,}000$).} The cap is applied to all pipelines (\autoref{app:downsampling}). To verify that it does not affect our findings, we re-run a subset of pipelines at full size on eight of the twelve datasets that exceed the cap ($n \in \{77{,}213; 80{,}358; 109{,}766; 117{,}984; 150{,}346; 170{,}730; 183{,}960; 270{,}009\}$): 17 pipelines (\{Tf-Idf, All-MiniLM-L6-v2, E5-base-v2, Jasper-0.6B\} $\times$ \{Ridge, XGBoost, ExtraTrees, RealMLP\}) and one end-to-end pipeline (ConTextTab). The rankings agree with Kendall-$\tau = 0.96$. Across all evaluated learners, the relative ordering of pipelines is preserved, and Tf-Idf consistently outperforms All-MiniLM-L6-v2 — suggesting that Tf-Idf's advantage in our benchmark reflects the nature of the strings in STRABLE rather than an artifact of subsampling.}
    \label{fig:full_vs_sampled_comparison}
\end{figure*}

\begin{table}[t]
  \caption{\textbf{Preprocessing preserves the Tf-Idf-vs-LLM gap.} For each dataset $\times$ learner $\times$ LLM encoder, $\text{gap} = \text{score}(\text{Tf-Idf}) - \text{score}(\text{LLM})$ is computed under the default and the alternative policy; $\Delta\text{gap} = \text{gap}_{\text{alt}} - \text{gap}_{\text{default}}$. We report the median and maximum of $|\Delta\text{gap}|$ across datasets ($R^2$ / AUROC).}
  \label{tab:gap_deltas}
  \centering
  \small
  \begin{tabular}{lcc}
    \toprule
    Alternative policy & median $|\Delta\text{gap}|$ & max $|\Delta\text{gap}|$ \\
    \midrule
    Raw targets (no skew transform)   & 0.005 & 0.030 \\
    Manual feature engineering         & 0.001 & 0.010 \\
    Missing-value imputation           & 0.001 & 0.006 \\
    Subsampling (75K vs.\ full)        & 0.001 & 0.028 \\
    \bottomrule
  \end{tabular}
\end{table}

\begin{table}[h]
  \caption{\textbf{TabuLa-8B underperforms all other pipelines.} \texttt{TabuLa-8B} \citep{gardner2024large} serializes each row and predicts its label autoregressively; at $\approx$2.4\,s per test row, we evaluate it only on five small classification tables (\texttt{orphan-designations}, \texttt{paediatric-investigation-plan}, \texttt{tobacco-problem}, \texttt{drug-shortages}, \texttt{food-enforcement}). Average AUROC is reported. \texttt{TabuLa-8B} is pretrained on the T4 corpus discussed in \autoref{sec:contamination_contexttab}, and its reported generalization has been largely attributed to train--test overlap \citep{gorla2026illusiongeneralizationtabularlanguage}. Prompted or instruction-tuned generative LLMs are not evaluated, as they add configuration choices (serialization, prompting, shot selection) and are costly at benchmark scale.}
  \label{tab:tabula}
  \centering\small
  \begin{tabular}{lc}
    \toprule
    Pipeline & AUROC \\
    \midrule
    Qwen-3-8B + TabPFN-2.5      & 0.871 \\
    Tf-Idf + TabPFN-2.5         & 0.865 \\
    LLaMA-3.1-8B + TabPFN-2.5   & 0.861 \\
    ConTextTab (E2E)            & 0.860 \\
    TabSTAR (E2E)               & 0.832 \\
    CatBoost (E2E)              & 0.813 \\
    Mambular (E2E)              & 0.739 \\
    TabuLa-8B                   & 0.614 \\
    \bottomrule
  \end{tabular}
\end{table}

\begin{figure*}[t!]
    \centering
    \makebox[\textwidth]
    {\includegraphics[width=\textwidth]{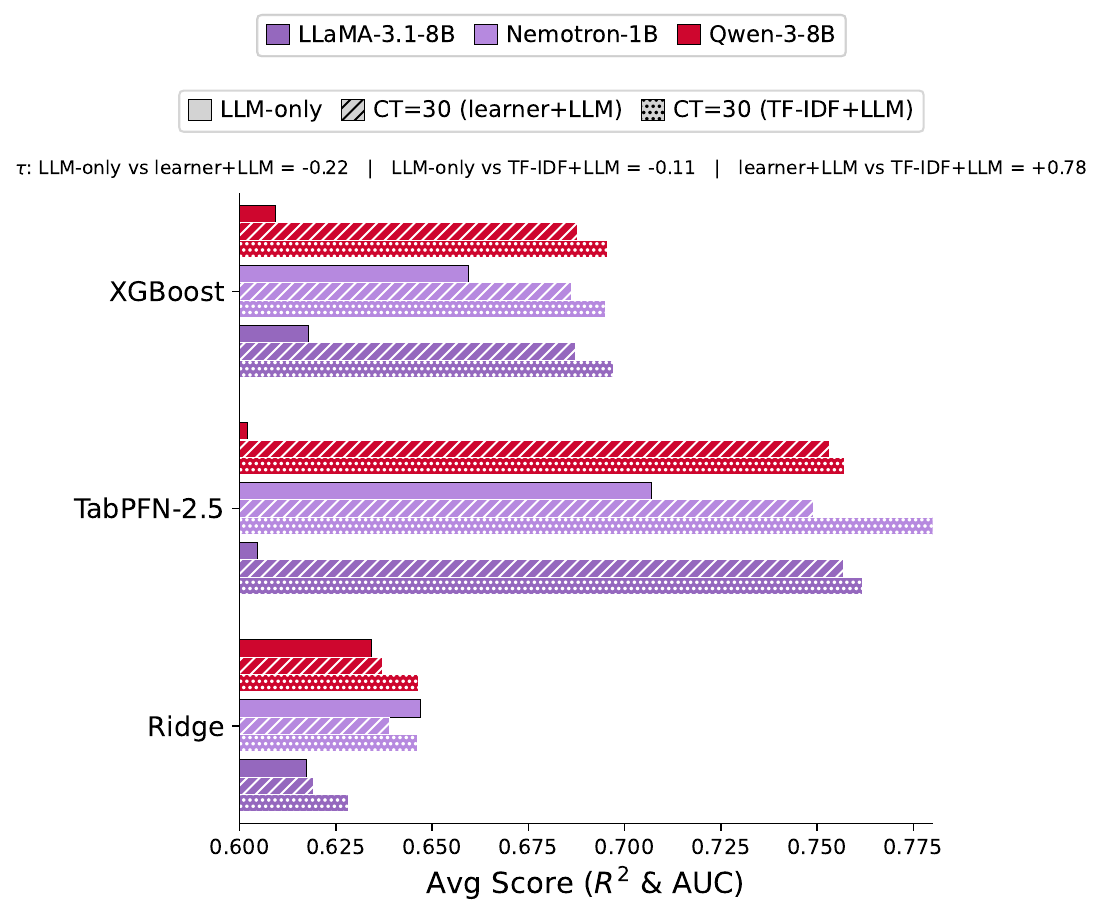}}%
    \caption{\textbf{Effect of low-cardinality routing strategy on LLM encoders.}
    Average score ($R^2$ \& AUC) for three encoders (LLaMA-3.1-8B, Qwen3-8B, Nemotron-1B)
    across three learners under three CT=30 routing strategies:
    \textbf{LLM-only} — all string features encoded by the LLM regardless of cardinality;\textbf{CT=30 (learner+LLM)} — low-cardinality features handled natively by the learner (XGBoost, TabPFN-2.5) or via OHE (Ridge), high-cardinality by the LLM;
    \textbf{CT=30 (TF-IDF+LLM)} — low-cardinality features ($\leq$30) handled by TF-IDF, high-cardinality by the LLM.
    Kendall's $\tau$ between the \emph{TF-IDF+LLM} and \emph{learner+LLM} rankings is $+0.78$,
    indicating that TF-IDF and native learner handling are largely interchangeable for low-cardinality features.
    By contrast, $\tau = -0.22$ between \emph{learner+LLM} and \emph{LLM-only},
    and $\tau = -0.11$ between \emph{TF-IDF+LLM} and \emph{LLM-only},
    showing that routing all features through the LLM disrupts rankings.}
    \label{fig:ridge_xgb_tabpfn_llama_qwen_nemotron_CT30}
\end{figure*}

\begin{figure*}[t!]
    \centering
    \makebox[\textwidth]
    {\includegraphics[width=\textwidth]{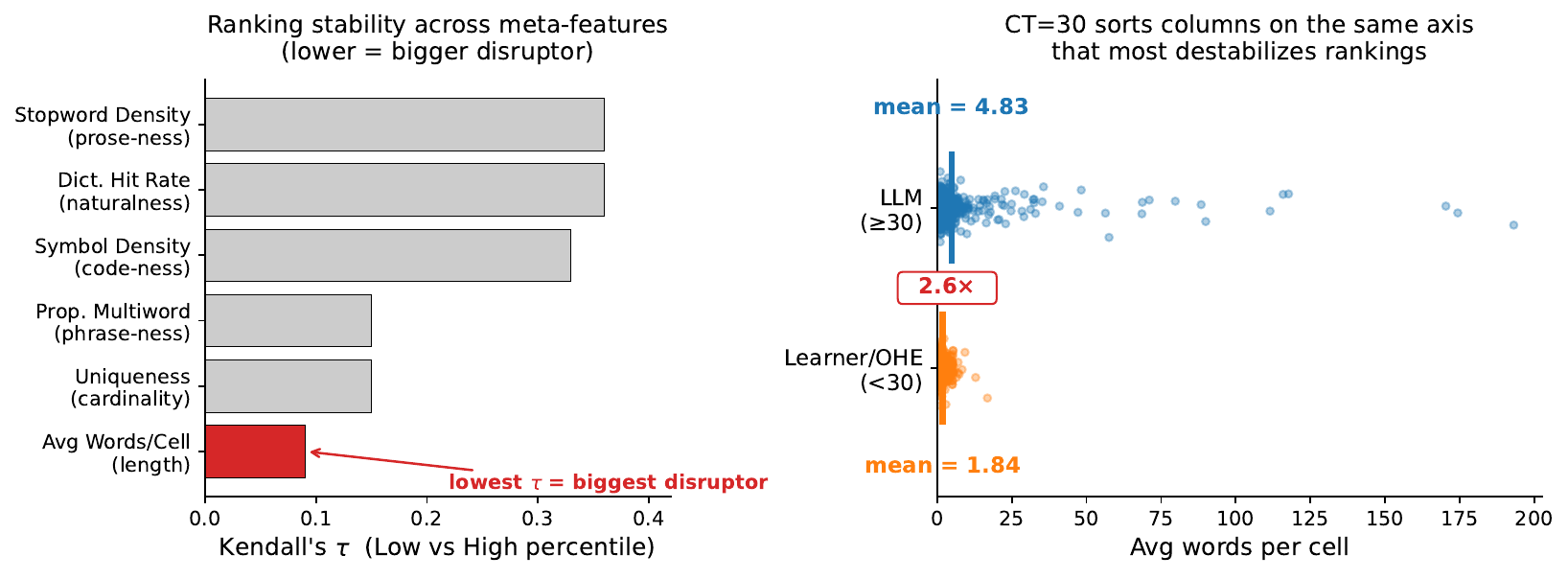}}%
    \caption{\textbf{Longer text is redirected to LLMs encoding}. From figure \ref{fig:ridge_xgb_tabpfn_llama_qwen_nemotron_CT30} we wanted to understand what kind of strings are encoded by the Language Models and what kinds are treated by the 30 Cardinality Threshold. As differentiating factor we use the average words per cell - being the string characteristic that most disrupts rankings. The picture shows that the average words per cell of features encoded by Language Models is 3 times more the length of features treated by OHE or by the native string handling of the learner.}
    \label{fig:CT_30_threshold_vs_string_index}
\end{figure*}

\begin{table}[!ht]
\centering
\caption{Score differences between one-hot and passthrough encoding of
categoricals are negligible for XGBoost and TabPFN-2.5.
$\Delta = \text{score}_{\text{passthrough}} - \text{score}_{\text{OHE}}$ is computed
per aligned $(\text{dataset}, \text{encoder})$ pair.
\textbf{$\overline{|\Delta|}$}: mean absolute score difference across pairs.
\textbf{$\max |\Delta|$}: largest absolute score difference observed.
\textbf{\% within 0.01}: fraction of pairs for which $|\Delta| < 0.01$.}
\small
\begin{tabular}{lcccc}
\toprule
Learner & $n$ $\overline{|\Delta|}$ & $\max |\Delta|$ & \% within 0.01 \\
\midrule
XGBoost     & $0.007$ & $0.083$ & 80\% \\
TabPFN-2.5  & $0.003$ & $0.060$ & 95\% \\
\bottomrule
\end{tabular}
\label{tab:ohe-passthrough}
\end{table}

\begin{table}[!ht]
\centering
\caption{Average score per learner by feature type.}
\label{tab:dtype_results}
\begin{tabular}{lccc}
\toprule
\textbf{Learner} & \textbf{Num-only} & \textbf{Num+Str} & \textbf{Str-only} \\
\midrule
CatBoost          & 0.475 & 0.695 & 0.624 \\
CatBoost-tuned    & 0.462 & 0.694 & 0.617 \\
ConTextTab        & 0.464 & 0.729 & 0.677 \\
ExtraTrees        & 0.430 & 0.691 & 0.641 \\
ExtraTrees-tuned  & 0.484 & 0.705 & 0.650 \\
Ridge             & 0.328 & 0.623 & 0.580 \\
TabPFN-2.5        & 0.485 & 0.686 & 0.585 \\
TabSTAR            & 0.430 & 0.701 & 0.653 \\
XGBoost           & 0.464 & 0.646 & 0.579 \\
XGBoost-tuned     & 0.485 & 0.696 & 0.653 \\
\midrule
\textbf{Average}  & \textbf{0.451} & \textbf{0.687} & \textbf{0.626} \\
\bottomrule
\end{tabular}
\end{table}

\begin{table}[ht]
\centering
\caption{Domain-level string meta-features and ranking stability ($\tau$). $\tau$: Kendall's tau measuring encoder ranking stability (higher = more stable). \textbf{n}: number of datasets in the category. \textbf{Words/Cell}: average number of words per cell. \textbf{Uniqueness}: ratio of unique values to total entries. \textbf{Vocab Div.}: ratio of unique words to total words across the category's string columns (lower values indicate more repetitive vocabulary). }
\label{tab:domain-meta}
\renewcommand{\arraystretch}{1.3}
\begin{tabular}{lcccccc}
\toprule
\textbf{Category} & $\tau$ & \textbf{n} & \textbf{Words/Cell} & \textbf{Uniqueness} & \textbf{Vocab Div.} \\
\midrule
Food            & 0.250 &  6 & 4.708 & 0.300 & 0.113 \\
Education       & 0.475 & 10 & 5.824 & 0.354 & 0.145 \\
Commerce        & 0.616 &  5 & 5.617 & 0.341 & 0.157 \\
Social          & 0.693 &  4 & 2.035 & 0.073 & 0.035 \\
Energy          & 0.702 &  9 & 2.894 & 0.219 & 0.085 \\
Infrastructure  & 0.767 & 18 & 2.935 & 0.241 & 0.105 \\
Economy         & 0.792 & 26 & 3.380 & 0.169 & 0.066 \\
Health          & 0.834 & 30 & 4.913 & 0.256 & 0.083 \\
\bottomrule
\end{tabular}
\label{tab:domain_metafeatures_stability}
\end{table}

\begin{figure}[t]
    \centering
    \includegraphics[height=10cm]{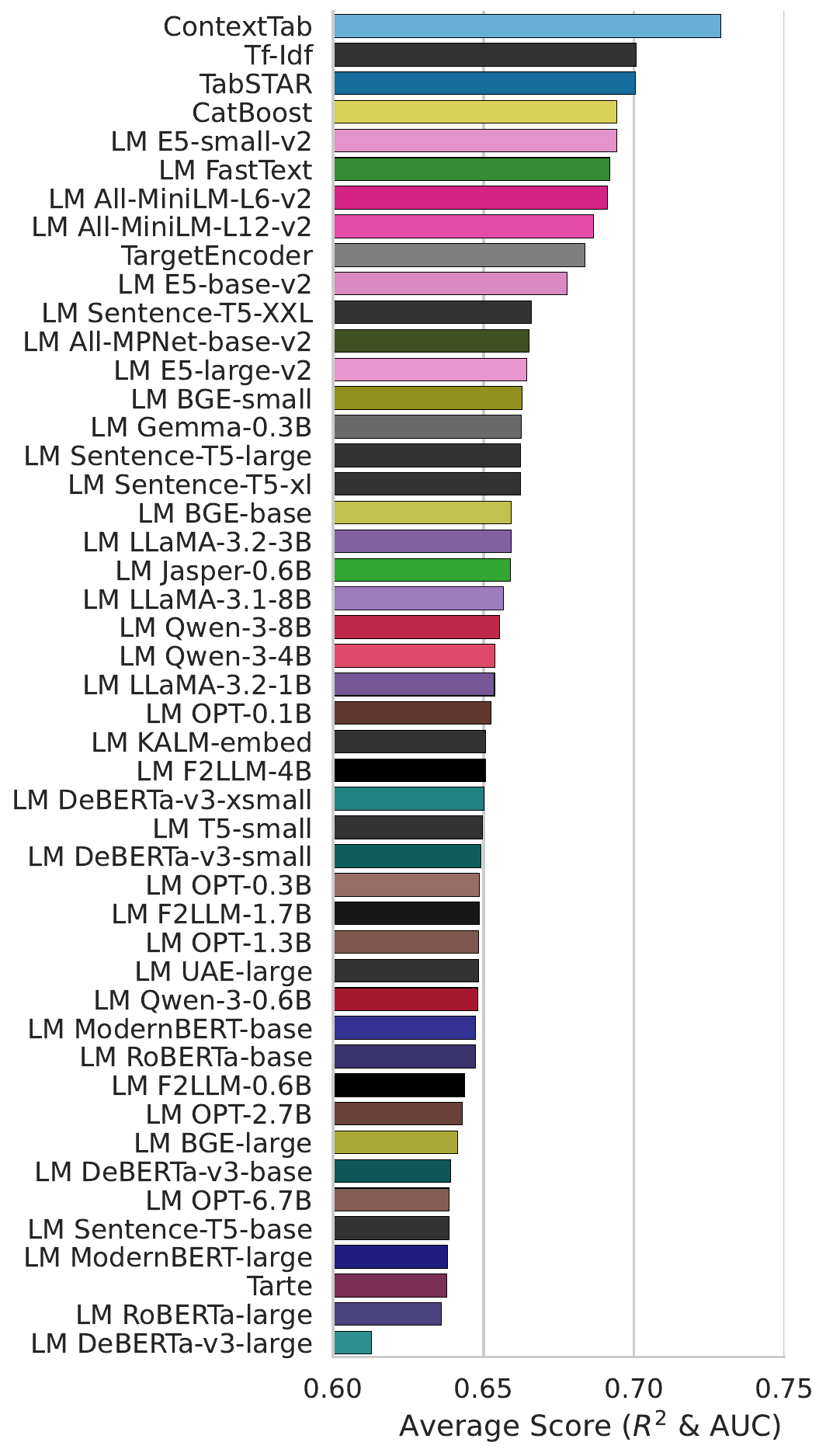}
    \caption{Average performance per encoder. ConTextTab and TabSTAR performances - being end-to-end architectures - are based on their respective learners. Catboost encoder performance is based on the average between the default and tuned version of its respective learner. Seven encoders - All-MiniLM-L12-v2, E5-base-v2, E5-large-v2, LLaMA-3.2-1B, LLaMA-3.2-3B, Qwen-3-0.6B, Qwen-3-4B - were run fully on 4 learners - ExtraTrees, Ridge, XGBoost and XGBoost-tuned - and for 98\% of the cases for TabPFN-2.5. Nine encoders, which are our representative ones - All-MiniLM-L6-v2, E5-small-v2, FastText, Jasper-0.6B, LLaMA-3.1-8B, Qwen-3-8B, TargetEncoder, Tarte, Tf-Idf - were fully run for ExtraTrees, ExtraTrees-tuned, Ridge, TabPFN-2.5, XGBoost, XGBoost-tuned. The remaining 28 encoders were run on ExtraTrees, Ridge and XGBoost.}%
    \label{fig:avg_encoder_performance_all_llm}
\end{figure}

\begin{table}[h]
    \centering
    \caption{Pareto Optimality Table}
    \label{tab:pareto_optimality_score_runtime}
    \begin{tabular}{lllrr}
        \toprule
        \textbf{Encoder} & \textbf{Learner} & \textbf{Score} & \textbf{Runtime (s/1k)} \\
        \midrule
        TargetEncoder & Ridge & 0.6593 & 0.0924 \\
        TargetEncoder & ExtraTrees & 0.7290 & 0.1695 \\
        Tf-Idf & ExtraTrees & 0.7407 & 0.9679 \\
        LM E5-small-v2 & ExtraTrees & 0.7422 & 2.5152 \\
        Tf-Idf & TabICLv2 & 0.7799 & 2.9988 \\
        Tf-Idf & TabPFN-2.5 & 0.7891 & 6.6700 \\
        \bottomrule
    \end{tabular}
\end{table}

\begin{table}[htbp]
    \centering
    \caption{Summary statistics for the number of columns per dataset after applying 30-PCA. \textbf{TabICLv2 was trained on a maximum of 100 features, while TabPFN was trained on up to 2000}. This discrepancy in the training feature budget may explain why TabICLv2, despite being highly competitive, performs slightly worse than TabPFN in this high-dimensional regime.}
    \label{tab:pca_col_estim}
    \begin{tabular}{lr}
        \toprule
        \textbf{Statistic} & \textbf{Value} \\
        \midrule
        Count & 108.00 \\
        Mean  & 416.17 \\
        Std   & 254.61 \\
        Min   & 95.00 \\
        25\%  & 243.50 \\
        50\%  & 344.50 \\
        75\%  & 520.5 \\
        Max   & 1270.00 \\
        \bottomrule
    \end{tabular}
\end{table}

\end{appendices}

\clearpage
\section*{NeurIPS Paper Checklist}

\begin{enumerate}

\item {\bf Claims}
    \item[] Question: Do the main claims made in the abstract and introduction accurately reflect the paper's contributions and scope?
    \item[] Answer: \answerYes{} %
    \item[] Justification: The abstract and introduction (Sections 1 and 2) state four contributions, each of which is explicitly developed in a dedicated section: the benchmarking landscape and the gap that motivates STRABLE (Section 2), the curation methodology that yields 108 tables with raw strings (Section 3), the empirical study of approximately 445 pipelines (Section 4), and the analysis showing that the ranking produced by STRABLE is stable and close to the oracle ranking (Section 5). Limitations of these claims are stated in Section 6.

\item {\bf Limitations}
    \item[] Question: Does the paper discuss the limitations of the work performed by the authors?
    \item[] Answer: \answerYes{} %
    \item[] Justification:  Section 6 contains an explicit ``Limitations'' paragraph noting that STRABLE reflects the string distribution of data-science tables rather than long-form text, so it enables only limited study of sentence-heavy tables, and that it does not address time-series specific validation protocols. 

\item {\bf Theory assumptions and proofs}
    \item[] Question: For each theoretical result, does the paper provide the full set of assumptions and a complete (and correct) proof?
    \item[] Answer: \answerYes{} %
    \item[] Justification: Appendix A states Assumptions A.1--A.5 and provides complete proofs for Lemmas A.1--A.2, Propositions A.1--A.6, and Corollaries A.1--A.6. The basic homoskedastic setting (Sections A.1--A.4) is relaxed in Section A.5 to incorporate inductive biases, and Sections A.6--A.7 derive analogous bounds for top-1 disagreement. All theorems and lemmas are numbered, cross-referenced, and the assumptions used in each proof are stated explicitly. 

    \item {\bf Experimental result reproducibility}
    \item[] Question: Does the paper fully disclose all the information needed to reproduce the main experimental results of the paper to the extent that it affects the main claims and/or conclusions of the paper (regardless of whether the code and data are provided or not)?
    \item[] Answer: \answerYes{} %
    \item[] Justification: The dataset corpus is downloadable from the URL in the abstract footnote, and all 108 datasets are individually documented in Section C.4 with their original sources and URLs. The minimal preprocessing is described in Section 3 and Section C.5. The full evaluation pipeline -- encoders, learners, end-to-end models, cross-validation protocol, hyperparameter search, prediction scheme, evaluation metrics, and statistical significance tests -- is detailed in Appendix D.1, with hyperparameter search spaces in Table D.1 and encoder specifications in Table D.2 and Table C.3. 

\item {\bf Open access to data and code}
    \item[] Question: Does the paper provide open access to the data and code, with sufficient instructions to faithfully reproduce the main experimental results, as described in supplemental material?
    \item[] Answer: \answerYes{}{} %
    \item[] Justification: The full set of curated datasets and the code are openly available at the URLs in the abstract footnotes, and the complete list of 108 sources with their original URLs is given in Section C.4. Section 3, Section C.5 (sub-sampling), Section C.6 (string profiling), and Appendix D.1 (hyperparameter search spaces in Table D.1, cross-validation protocol, prediction scheme) collectively provide the instructions and configurations needed to reproduce the main results.

\item {\bf Experimental setting/details}
    \item[] Question: Does the paper specify all the training and test details (e.g., data splits, hyperparameters, how they were chosen, type of optimizer) necessary to understand the results?
    \item[] Answer: \answerYes{} %
    \item[] Justification: Appendix D.1 specifies the nested cross-validation protocol, the randomized hyperparameter search over 100 iterations including defaults, and the prediction scheme. Hyperparameter search spaces for all tuned learners (Ridge, XGBoost, CatBoost, ExtraTrees) are listed in Table D.1; default-only configurations for TabPFN-2.5, TabSTAR, ContextTab, RealMLP, TabM, TabICLv2, and Mambular are explicitly stated. Preprocessing choices (PCA dimension, missing-value handling, target transformations, sub-sampling threshold) are detailed in Sections 3 and 4.

\item {\bf Experiment statistical significance}
    \item[] Question: Does the paper report error bars suitably and correctly defined or other appropriate information about the statistical significance of the experiments?
    \item[] Answer: \answerYes{} %
    \item[] Justification: Figure 1 (right panel) reports 95\% confidence intervals on average scores; Figure 5 reports median $\pm$ standard error of Kendall-$\tau$ across bootstrap subsamples. Pairwise statistical comparisons in the critical-difference diagrams (Figure 3, Figures E.4--E.6) use the Friedman test followed by the Conover-Iman post-hoc test at $\alpha = 0.05$, with the test statistic given in Equation D.1 and full description in Appendix D.1.

\item {\bf Experiments compute resources}
    \item[] Question: For each experiment, does the paper provide sufficient information on the computer resources (type of compute workers, memory, time of execution) needed to reproduce the experiments?
    \item[] Answer: \answerYes{} %
    \item[] Justification: Appendix D.1 (``Computational resources'') reports the hardware used (NVIDIA V100/A100/A40 GPUs and AMD EPYC / Intel Xeon CPUs with up to 512\,GB RAM) and the total compute budget for the entire STRABLE experiment (842 CPU+GPU days).
    
\item {\bf Code of ethics}
    \item[] Question: Does the research conducted in the paper conform, in every respect, with the NeurIPS Code of Ethics \url{https://neurips.cc/public/EthicsGuidelines}?
    \item[] Answer: \answerYes{} %
    \item[] Justification: The research conforms to the NeurIPS Code of Ethics. All datasets are aggregated from public institutional repositories (e.g., FDA, World Bank, HRSA, FCC, OpenML-style sources) and community-driven platforms, listed individually in Section C.4.

\item {\bf Broader impacts}
    \item[] Question: Does the paper discuss both potential positive societal impacts and negative societal impacts of the work performed?
    \item[] Answer: \answerNA{} %
    \item[] Justification: The paper introduces a benchmarking corpus and an empirical study of existing tabular learners on tables containing strings. The contribution is methodological -- providing a foundation for evaluating tabular learning pipelines -- and does not introduce new generative capabilities, surveillance technologies, or models with a direct path to harmful applications. Positive impacts (better empirical comparison of tabular methods, guidelines for practitioners, more rigorous evaluation protocols) follow naturally from the contribution and do not warrant a dedicated section.
    
\item {\bf Safeguards}
    \item[] Question: Does the paper describe safeguards that have been put in place for responsible release of data or models that have a high risk for misuse (e.g., pre-trained language models, image generators, or scraped datasets)?
    \item[] Answer: \answerNA{} %
    \item[] Justification: The released asset is a curated collection of tabular datasets aggregated from public institutional and community sources (Section C.4). It does not contain pre-trained generative models, image data, free-form personal communications, or scraped content of a kind that would pose a high risk of misuse. We re-distribute the same content already publicly available from each original source.

\item {\bf Licenses for existing assets}
    \item[] Question: Are the creators or original owners of assets (e.g., code, data, models), used in the paper, properly credited and are the license and terms of use explicitly mentioned and properly respected?
    \item[] Answer: \answerYes{} %
    \item[] Justification: Every dataset source is individually cited and linked in Section C.4. All software dependencies used in the pipelines are cited as well as the embedding models listed in Table C.3 with their HuggingFace identifiers. Datasets are used in accordance with their original public terms of service.

\item {\bf New assets}
    \item[] Question: Are new assets introduced in the paper well documented and is the documentation provided alongside the assets?
    \item[] Answer: \answerYes{} %
    \item[] Justification: The new asset introduced is the STRABLE corpus, which is documented in Section 3 (curation methodology), Section C.3 (sources and characteristics by domain), Section C.4 (per-dataset descriptions, sources, URLs, and target tasks for all 108 datasets), Section C.5 (sub-sampling protocol), and Section C.6 (semantic taxonomy and column-level profiling, with validation against manual annotation). The download link in Section C.1 provides anonymous access to the curated corpus.

\item {\bf Crowdsourcing and research with human subjects}
    \item[] Question: For crowdsourcing experiments and research with human subjects, does the paper include the full text of instructions given to participants and screenshots, if applicable, as well as details about compensation (if any)? 
    \item[] Answer: \answerNA{} %
    \item[] Justification: The paper does not involve crowdsourcing or research with human subjects. The validation of the column-profiling heuristic in Section C.6 was performed by the authors themselves with a state-of-the-art LLM as a secondary verifier; no external annotators were recruited.

\item {\bf Institutional review board (IRB) approvals or equivalent for research with human subjects}
    \item[] Question: Does the paper describe potential risks incurred by study participants, whether such risks were disclosed to the subjects, and whether Institutional Review Board (IRB) approvals (or an equivalent approval/review based on the requirements of your country or institution) were obtained?
    \item[] Answer: \answerNA{} %
    \item[] Justification: The paper does not involve research with human subjects. All datasets are aggregated from existing public sources.

\item {\bf Declaration of LLM usage}
    \item[] Question: Does the paper describe the usage of LLMs if it is an important, original, or non-standard component of the core methods in this research? Note that if the LLM is used only for writing, editing, or formatting purposes and does \emph{not} impact the core methodology, scientific rigor, or originality of the research, declaration is not required.
    \item[] Answer: \answerYes{} %
    \item[] Justification: LLMs are a core component of the studied pipelines: a wide range of pre-trained language models (encoder-only, decoder-only, and encoder--decoder; full list in Table C.3 and Table D.2) are used as string encoders within the modular pipelines evaluated on STRABLE. Their use, configuration, and post-processing are described in Section 4. Additionally, an LLM-as-a-judge was used as a secondary verifier in the validation of the string-column profiling heuristic (Section C.6); the heuristic itself is deterministic and rule-based.

\end{enumerate}

\end{document}